\documentclass[letterpaper]{article} 
\usepackage{aaai2027}  
\usepackage[hyphens]{url}  
\usepackage{graphicx} 
\usepackage{natbib}  
\usepackage{caption} 
\usepackage{algorithm}
\usepackage{algorithmic}
\usepackage{amsmath}
\usepackage{amsthm}
\allowdisplaybreaks
\usepackage{dsfont}
\usepackage{graphicx}
\usepackage{amssymb}
\usepackage{subcaption}
\usepackage{mathtools}
\usepackage{upgreek}
\usepackage{booktabs}
\usepackage{float}
\usepackage{multirow}
\usepackage{natbib}
\usepackage[table]{xcolor}
\usepackage{physics}
\newtheorem{lemma}{Lemma}

\usepackage{newfloat}
\usepackage{listings}
\DeclareCaptionStyle{ruled}{labelfont=normalfont,labelsep=colon,strut=off} 
\floatstyle{ruled}
\newfloat{listing}{tb}{lst}{}
\floatname{listing}{Listing}

\usepackage{booktabs}

\nocopyright

\title{Floating Radiance Networks}
\author{
    Krzysztof Byrski\textsuperscript{\rm 1}\equalcontrib\corresponding,
    Rafał Tobiasz \textsuperscript{\rm 1, 3}\equalcontrib,
    Grzegorz Wilczyński \textsuperscript{\rm 1, 3},
    Mikołaj Zieliński \textsuperscript{\rm 2},
    Dawid Baran,\\
    Dominik Belter \textsuperscript{\rm 2},
    Jacek Tabor \textsuperscript{\rm 1},
    Przemysław Spurek \textsuperscript{\rm 1, 3}
}
\affiliations{
    \textsuperscript{\rm 1} Jagiellonian University, Faculty of Mathematics and Computer Science\\
    \textsuperscript{\rm 2} Poznań University of Technology, Institute of Robotics and Machine Intelligence\\
    \textsuperscript{\rm 3} IDEAS Research Institute \\
    krzysiek.byrski@uj.edu.pl
}

\usepackage{amsmath}
\usepackage{amssymb}
\usepackage{booktabs}
\usepackage{multirow}
\usepackage[table]{xcolor}
\usepackage{upgreek}
\usepackage{url}

\renewcommand{\vec}[1]{\mathaccent"017E{#1}}

\usepackage{pifont}

\definecolor{googleGreen}{HTML}{4FAF72}
\definecolor{googleRed}{HTML}{DF746A}
\definecolor{googleGray}{HTML}{939AA1}

\newcommand{\cmark}{\textcolor{googleGreen}{\ding{51}}}
\newcommand{\xmark}{\textcolor{googleRed}{\ding{55}}}
\newcommand{\pmark}{\textcolor{googleGray}{\ding{108}}}

\definecolor{bestcolor}{rgb}{1.0, 0.5, 0.5}   
\definecolor{secondcolor}{rgb}{1.0, 0.8, 0.6} 
\definecolor{thirdcolor}{rgb}{1.0, 1.0, 0.6}  

\def\our{FlaRe}
\def\ourfull{Floating Radiance Networks}

\begin{document}

\maketitle

\begin{abstract}

    Recent advances in neural scene representations enable photorealistic novel-view synthesis, yet most methods remain tightly coupled to a single rendering paradigm, limiting their versatility and integration with conventional graphics workflows. We introduce \ourfull{} (\our{}), a neural scene representation combining explicit ray-traceable geometry with continuous neural radiance functions. A scene is represented by floating planar generalized Gaussian primitives, each carrying a compact latent descriptor of a local radiance field. A lightweight decoder shared across the scene maps this descriptor, local surface coordinates, and viewing direction to color and opacity. This formulation preserves the expressiveness of neural fields while providing an explicitly addressable structure that can be efficiently queried and manipulated. Hardware-accelerated primitive intersections enable interactive rendering and recursive ray-tracing, including reflections, refractions, transparency, and shadows. The same representation further supports primitive-level deformation, mesh extraction, and appearance stylization directly in its learned descriptor space. Experiments across standard reconstruction benchmarks demonstrate competitive rendering quality while using a compact set of primitives. Together, these results establish \our{} as a versatile representation that brings high-fidelity neural rendering, ray-tracing, geometric manipulation, and appearance editing into a unified scene model. Source code is available online. Source code can be found at: \url{https://github.com/KByrski/FlaRe}

\end{abstract}

\section{Introduction}

    \begin{figure}[t]
    \centering
    \includegraphics[width=\columnwidth]{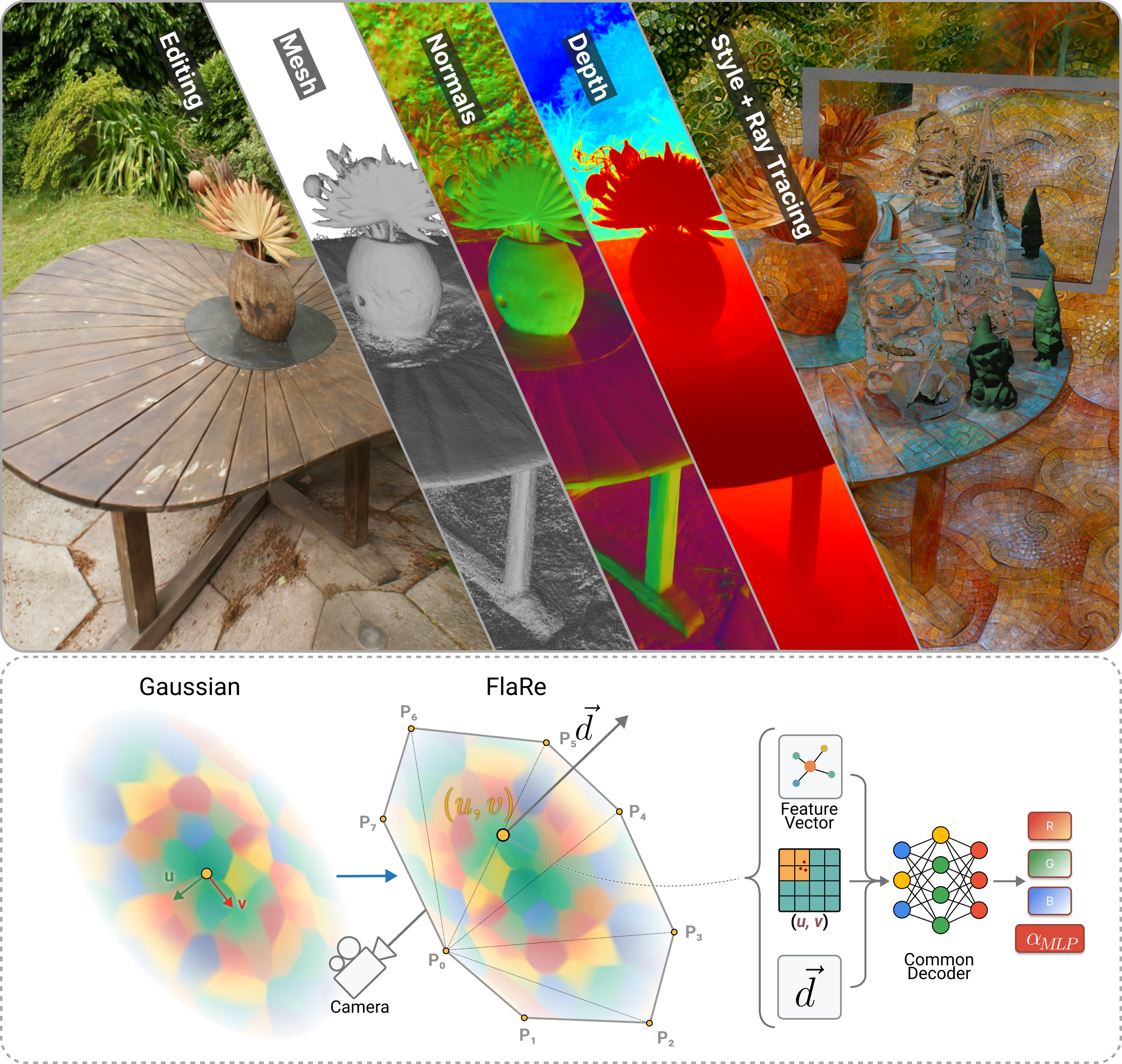}
        \caption{\our{} represents a scene with floating planar primitives, each carrying a compact local radiance descriptor. At a ray--primitive intersection $(u,v)$, a shared decoder maps the descriptor, local coordinates, and viewing direction to color and opacity. Hardware-accelerated intersections replace volumetric sampling, so the capabilities above share one explicit, ray-traceable model.}
    \label{fig:teaser}
    \end{figure}

Neural scene representations have transformed photorealistic 3D reconstruction and novel-view synthesis. Neural Radiance Fields (NeRFs)~\cite{mildenhall2020nerf} achieve high visual fidelity by modeling scenes as continuous volumetric functions, but dense ray marching and repeated network evaluations make rendering costly and complicate integration with conventional graphics pipelines. In contrast, 3D Gaussian Splatting (3DGS)~\cite{kerbl20233d} enables real-time rendering through explicit anisotropic primitives and rasterization. However, its finite per-primitive appearance coefficients do not define a continuous spatial radiance function over each primitive footprint.
 \begin{table}[t]
        \centering
        \scriptsize
        \setlength{\tabcolsep}{2.0pt}
        \renewcommand{\arraystretch}{1.08}
        \resizebox{\columnwidth}{!}{%
            \begin{tabular}{lccccccc}
                \toprule
                Method
                & \shortstack{Features}
                & \shortstack{Local Neural\\Appearance}
                & \shortstack{Ray\\Tracing}
                & \shortstack{Secondary\\Rays}
                & \shortstack{Geometry\\Editing}
                & \shortstack{Mesh\\Extraction} 
                & \shortstack{Primitive\\Reduction} \\
                \midrule
                3DGS
                & \xmark & \xmark & \xmark & \xmark & \xmark & \xmark & \xmark \\
    
                Scaffold-GS
                & \pmark & \xmark & \xmark & \xmark & \xmark & \xmark & \xmark \\
    
                2DGS
                & \xmark & \xmark & \xmark & \xmark & \pmark & \cmark & \xmark \\
    
                Nexels
                & \xmark & \cmark & \xmark & \xmark & \xmark & \xmark & \cmark \\
    
                3DGRT
                & \xmark & \xmark & \cmark & \cmark & \xmark & \xmark & \xmark \\
    
                IRIS
                & \cmark & \pmark & \xmark & \xmark & \cmark & \xmark & \xmark \\
    
                Power Foam
                & \xmark & \cmark & \cmark & \pmark & \xmark & \xmark & \xmark \\
    
                \midrule
                \textbf{\our{} (ours)}
                & \cmark & \cmark & \cmark & \cmark & \cmark & \cmark & \cmark \\
                \bottomrule
            \end{tabular}%
        }
        \caption{Capabilities demonstrated by the original formulations.
        A circle denotes partial or indirect support. Features
        denote learned latent descriptors stored with explicit scene elements.}
        \label{tab:representation_capabilities}
    \end{table}
This design space exposes a persistent trade-off between the expressiveness of continuous neural fields and the efficiency, addressability, and graphics compatibility of explicit geometry. Recent methods combine neural features with explicit structures or support direct ray queries~\cite{radfoam,eks,iris,powerfoam}. Nevertheless, they are often tied to a particular rendering procedure, aggregate features across scene elements, or require task-specific extensions for editing, ray-tracing, or geometry extraction. Combining locally continuous radiance with independently addressable, ray-traceable geometry would provide a more direct bridge between neural rendering and conventional graphics.

We introduce \ourfull{} (\our{}), a neural scene representation built from floating planar generalized Gaussian primitives carrying local neural radiance functions (Figure~\ref{fig:teaser}). Each primitive stores a compact trainable descriptor evaluated by a lightweight decoder shared across the scene. Conditioned on the descriptor, local surface coordinates, and viewing direction, the decoder predicts spatially varying and view-dependent color and opacity. This avoids storing an independent network per primitive while retaining a continuous appearance function over its local support. The explicit primitive structure makes \our{} directly compatible with hardware-accelerated ray-tracing. Generalized Gaussian primitives use shared canonical mesh proxies, enabling efficient collection of ordered ray--primitive intersections and front-to-back compositing. The representation supports arbitrary and recursive ray queries, including reflections, refractions, transparency, shadows, and composition with mesh geometry. Because primitives remain individually addressable, their geometry can be directly deformed or manipulated. With geometry-aware regularization, the planar representation also supports mesh extraction, while its descriptor space provides a suitable parameterization for appearance stylization and modeling under restricted primitive budgets.

Unlike NeRF-Casting, TraM-NeRF, and Mirror-NeRF~\cite{nerfcasting,tramnerf,mirror-nerf}, which introduce secondary rays within specialized volumetric rendering formulations, \our{} makes the representation itself directly ray-queryable. Thus, rendering, recursive ray-tracing, geometry manipulation, mesh extraction, and appearance editing use the same underlying representation rather than separate task-specific scene models (as shown in Table~\ref{tab:representation_capabilities}). Experiments on standard benchmarks demonstrate strong novel-view synthesis quality at interactive rates with a compact set of primitives.

Our contributions are summarized as follows:
\begin{itemize}
    \item We introduce \ourfull{}, a neural scene representation in which explicit ray-traceable planar primitives carry continuous local radiance functions.

    \item We design a compact formulation combining per-primitive latent descriptors, a shared lightweight auto-decoder, collision-free multiresolution LUT-Encoding, generalized Gaussian kernels, and hardware-compatible mesh proxies.

    \item We demonstrate recursive ray-tracing, primitive deformation, mesh extraction, appearance stylization, and reconstruction under restricted primitive budgets while retaining strong quality and interactive performance.
\end{itemize}

\begin{figure}[t]
    \centering
    \includegraphics[width=\columnwidth]{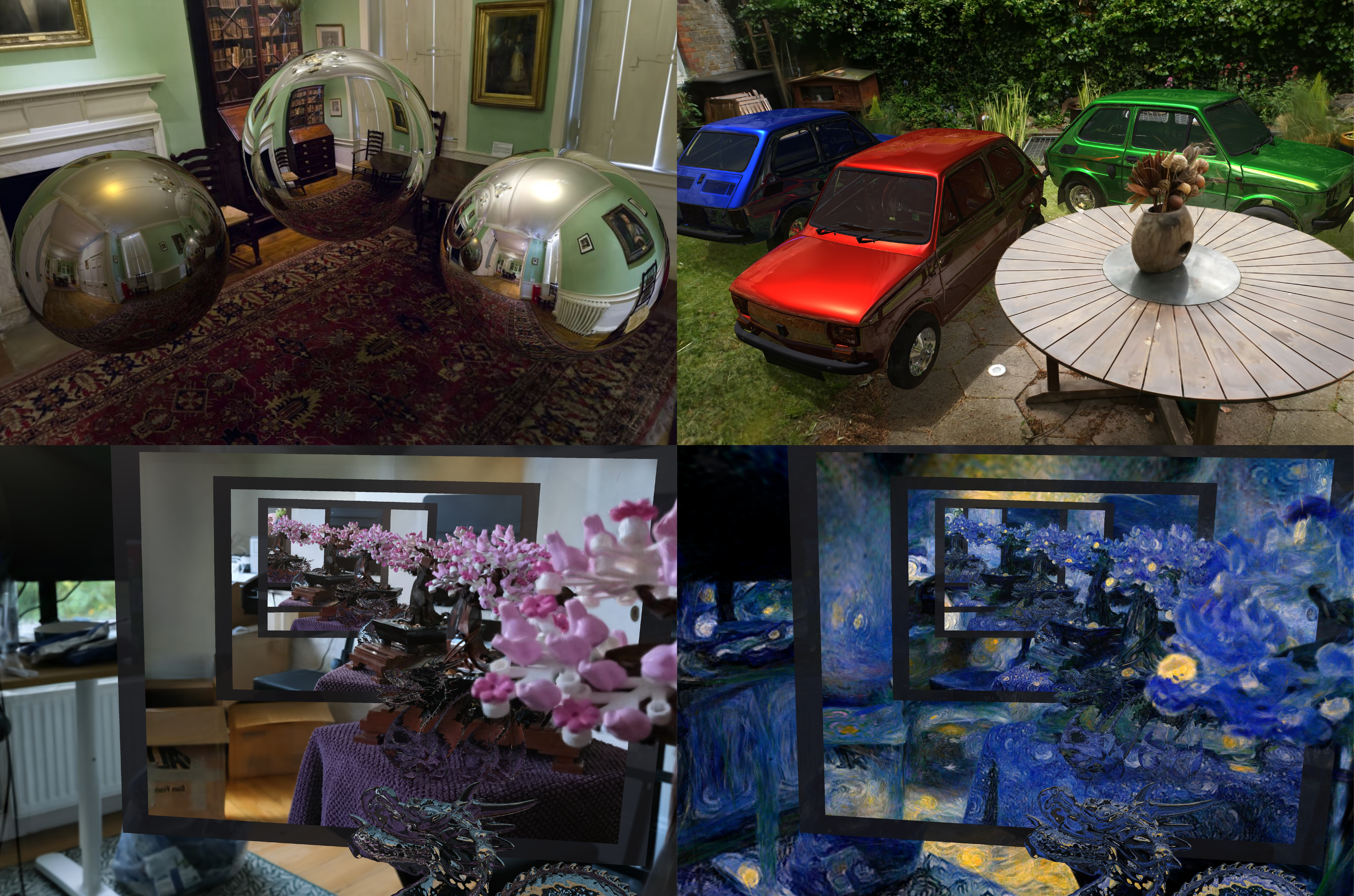}
    \caption{Recursive ray-tracing on \our{} reconstructions. Seamless integration with conventional meshes allows for material-dependent reflections, refractions, and nested bounces. This hybrid rendering natively generalizes to the scene after appearance stylization (bottom-right).}
    \label{fig:ray_tracing}
\end{figure}

\section{Related Work}

\begin{figure*}[t!]
    \centering
    \includegraphics[width=\textwidth]{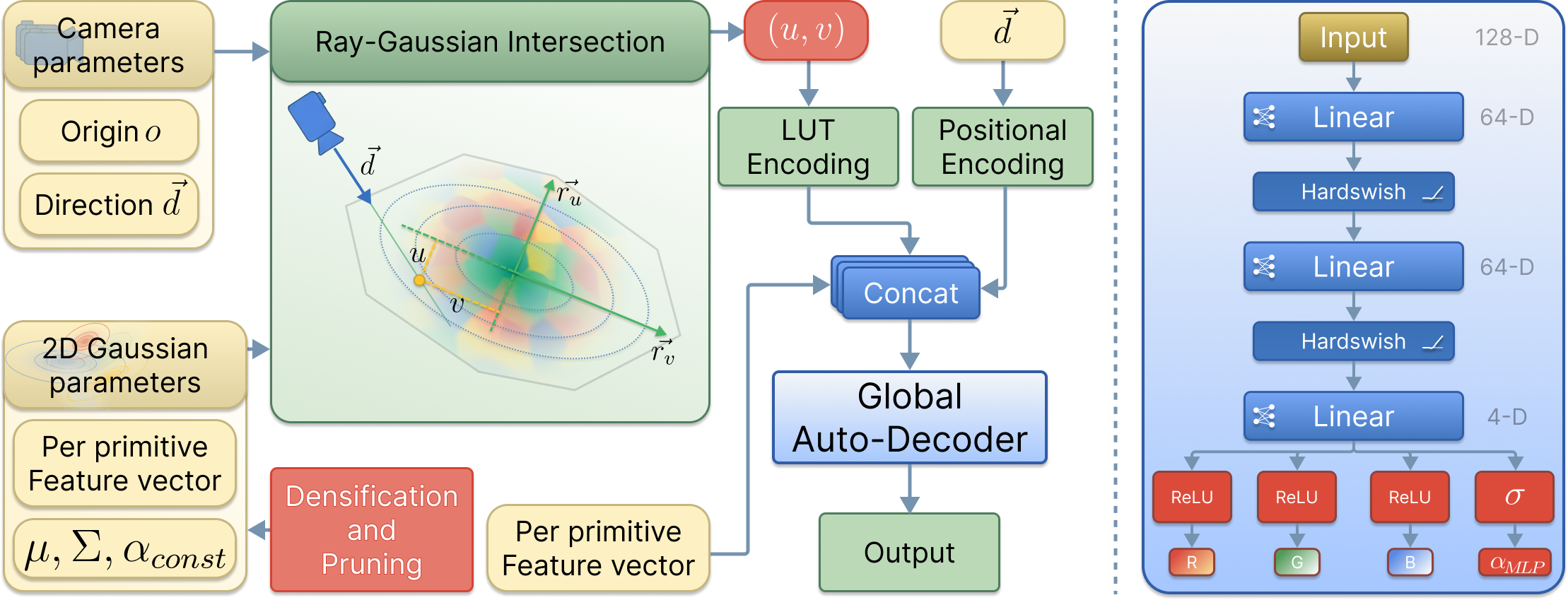}
    \caption{\textbf{Left}: Overview of the rendering pipeline. Every intersection along a camera ray contributes local coordinates $(u,v)$ and a viewing direction $d$, which are encoded and concatenated with the per-primitive feature vector before decoding into color and opacity. Densification and pruning update the primitive set during optimization. \textbf{Right}: The auto-decoder architecture comprises two linear hidden layers each followed by a Hardswish activation, and a linear output layer, which branches into four channels processed by three ReLU and one $\sigma$ activation functions, respectively.}
    \label{fig:architecture_mlp}
\end{figure*}

Neural Radiance Fields represent scenes as continuous functions evaluated
and integrated along camera rays~\cite{mildenhall2020nerf}, while later
methods improve quality and efficiency through scene contraction, sparse
grids, hash encodings, and tensor factorizations
~\cite{barron2022mipnerf360,barron2023zipnerf,plenoxels,instantngp,tensorf}.
In contrast, 3D Gaussian Splatting uses explicit anisotropic primitives
with rasterized per-primitive appearance~\cite{kerbl20233d}, and 2DGS
introduces planar surfels for improved geometric consistency~\cite{2dgs}.
Several approaches enrich explicit structures with learned appearance:
DNMP attaches latent features to mesh elements~\cite{dnmp}, Scaffold-GS
predicts Gaussian attributes from anchor features~\cite{scaffoldgs},
Gaussian Billboards use spatially varying textures~\cite{billboard},
and EKS and IRIS condition neural radiance on Gaussian-carried or
ray-intersection features~\cite{eks,iris}. Nexels similarly augments
planar primitives with spatially varying neural appearance~\cite{nexels}.
Unlike these methods, \our{} associates each primitive with a compact
descriptor of a continuous local radiance function evaluated directly
from local coordinates and viewing direction using one shared decoder.

Recent work also develops scene representations compatible with direct
ray queries. 3DGRT and 3DGUT adapt Gaussian primitives to
hardware-accelerated ray-tracing and secondary rays
~\cite{3dgrt,3dgut}, while RaySplats performs direct ray--Gaussian
intersection~\cite{raysplat}. REdiSplats, MeshSplats, and LinPrim expose
polygonal or polyhedral structures compatible with conventional graphics
pipelines~\cite{redisplat,meshsplat_tobiasz,linprim}. Beyond Gaussian
representations, Radiant Foam, Power Foam, SVRaster and Radiance Meshes employ
structured cells or meshes for efficient differentiable ray-tracing and,
in some cases, rasterization~\cite{radfoam,powerfoam,sun2025sparse,radiance_meshes}.
Explicit structure can additionally facilitate mesh extraction,
deformation, and appearance manipulation, as demonstrated by
2DGS, SuGaR, Gaussian Frosting, GaMeS, and Gaussian stylization methods
~\cite{2dgs,guedon2024sugar,frosting,waczynska2024games,
stylegaussian,clipgaussian}. In contrast to task-specific extensions,
\our{} combines local neural appearance, direct ray-tracing, geometry
editing, mesh extraction, and stylization within the same underlying
primitive representation.

\section{FlaRe: Floating Radiance Networks}

Given posed RGB images, \our{} optimizes a scene representation $\mathcal{S}$ composed of floating planar Gaussian primitives with learned local radiance descriptors. For each camera ray $r$, we intersect the ray with $\mathcal{S}$, decode color and opacity at the encountered primitives, and alpha-composite the resulting samples to obtain the rendered pixel $\mathrm{c}$.

\paragraph{Gaussian Primitive Scene Representation}
FlaRe represents a scene using \(N\) anisotropic planar generalized Gaussian primitives, initialized with equal in-plane scales:
$$
\mathcal{G}
=
\left\{
\left(
\mathcal{N}(\upmu_i,\Sigma_i,\kappa_i), \mathrm{c}_{\text{const},i}, \alpha_{\mathrm{const},i},
\mathrm{z}_i
\right)
\right\}_{i=1}^{N}.
$$
Each primitive is parameterized by a trainable center $\upmu_i\in\mathbb{R}^3$, singular covariance matrix $\Sigma_i = \mathrm{R}_i \mathrm{S}_i \mathrm{S}_i^\top \mathrm{R}_i^\top$, shape parameter $\kappa_i = 1+\operatorname{softplus}(\tilde{\kappa}_i)$, view-independent constant color parameter in the RGB space $\mathrm{c}_{i,\text{const}} \in \mathbb{R}^3$, constant opacity $\alpha_{\mathrm{const},i}=\sigma(\tilde{\alpha}_{\mathrm{const},i})$, and a trainable latent descriptor $\mathrm{z}_i$. Here, $\mathrm{R}_i$ is obtained from the trainable quaternion $\mathrm{q}_i$, while $\mathrm{S}_i=\operatorname{diag}\left(
\begin{bmatrix}
\exp(\tilde{\mathrm{s}}_i), 0
\end{bmatrix}
\right)
$ with $\tilde{\mathrm{s}}_i\in\mathbb{R}^2$ controls the primitive scale in its local tangent plane. The scalar $\tilde{\kappa}_i\in\mathbb{R}$ determines the generalized Gaussian shape, whereas $\tilde{\alpha}_{\mathrm{const},i}\in\mathbb{R}$ controls its base opacity. Finally, $\mathrm{z}_i \in \mathbb{R}^{96}$ encodes the local 2D neural radiance field associated with the primitive.

\paragraph{Isotropic Planar Generalized Gaussian Kernel}

Inspired by Nexels~\cite{nexels}, FlaRe employs opacity modulation via an unnormalized generalized Gaussian Kernel in the underlying 2D representation primitives. This formulation enables more precise modeling of sharp boundaries and high-curvature surfaces, effectively mitigating approximation errors inherent in classical Gaussian kernels characterized by a fixed quadratic fall-off profile. We adopt a definition of the Isotropic Planar Generalized Gaussian Kernel, given by the following formula:
\[
\mathcal{N}(\upmu,\Sigma,\kappa)(\mathbf{x})
=
\exp\!\left(
-\frac{
\left\|\mathrm{P}\mathrm{S}^{-1}\mathrm{R}^\top(\mathrm{x}-\upmu)\right\|_2^{2\kappa}
}{2\kappa}
\right),
\]
where \(\mathrm{P}=[\,\mathrm{I}_2\;\mathbf{0}\,]\in\mathbb{R}^{2\times3}\) selects the two tangent-plane coordinates. Thus, for a point expressed in the primitive's local frame as \((u,v)\), the kernel simplifies to
\[
\mathcal{N}(\upmu,\Sigma,\kappa)(u,v)
=
\exp\!\left(
-\frac{(u^2+v^2)^\kappa}{2\kappa}
\right),
\]
with final opacity
$
\alpha
=
\alpha_{\mathrm{const}}\alpha_{\mathrm{MLP}}\mathcal{N}(\upmu,\Sigma,\kappa)(u,v),
$
where \(\alpha_{\mathrm{MLP}}\) is predicted by the auto-decoder.

This approach allows us to utilize a single canonical polygon for each Gaussian, regardless of the $\kappa$ parameter, since the kernel has circular isocontours in normalized canonical coordinates, independent of the shape parameter \(\kappa\). In contrast, for the density function proposed by~\cite{nexels}, the polygon's geometry does not align with the isocontours, preventing the density values from vanishing at the boundary and necessitating the computation of a unique proxy polygon for every individual value of the shape parameter.

\paragraph{Global auto-decoder}
\label{sec:auto_decoder}
For clarity, we omit the ray index superscript $(r)$ throughout this subsection whenever the dependence on the current ray is unambiguous. In contrast to conventional Gaussian-based representations (e.g., 3DGS~\cite{kerbl20233d} or 3DGRT~\cite{3dgrt}), which rely on static, per-primitive color and opacity attributes, typically optimized as trainable constants modulated by a spatial Gaussian kernel, FlaRe adopts a more flexible approach. Specifically, both the color and opacity of samples within each 2D Gaussian footprint are predicted by a dedicated 2D neural radiance field, uniquely assigned to each individual primitive and operating within its specific local coordinate frame, and encoded individually for each Gaussian primitive as a trainable latent descriptor vector. The color and opacity are then decoded using a single global auto-decoder MLP:
$$
\Phi \left( \left( u_i, v_i \right), \mathrm{\vec{d}}, \mathrm{z}_i \;; \mathcal{G}, \mathcal{F}, \Theta \right) = \left( \mathrm{c}_{\text{MLP},i}, \alpha_{\text{MLP},i} \right)
$$
where $\left( \mathrm{c}_{\text{MLP},i}, \alpha_{\text{MLP},i} \right) \in \mathbb{R}^3 \times \mathbb{R}$ denotes the color and opacity predicted by the MLP at the intersection point $\mathrm{P}_i$ of ray $r$ and the $i$-th primitive encountered by the ray, given the direction vector $\mathrm{\vec{d}}$ and per-primitive local 2D neural radiance field descriptor $\mathrm{z}_i$ with $\mathrm{P}_i$ having the coordinates: $\left( u_i, v_i \right)$ within the primitive's local reference frame. The network is conditioned on the set of all Gaussian primitives $\mathcal{G}$, the LUT-Table features $\mathcal{F}$ and MLP parameters $\Theta$.
Our global auto-decoder architecture is visible in Figure~\ref{fig:architecture_mlp}.

\begin{figure}[t]
\centering
\includegraphics[width=0.43\textwidth]{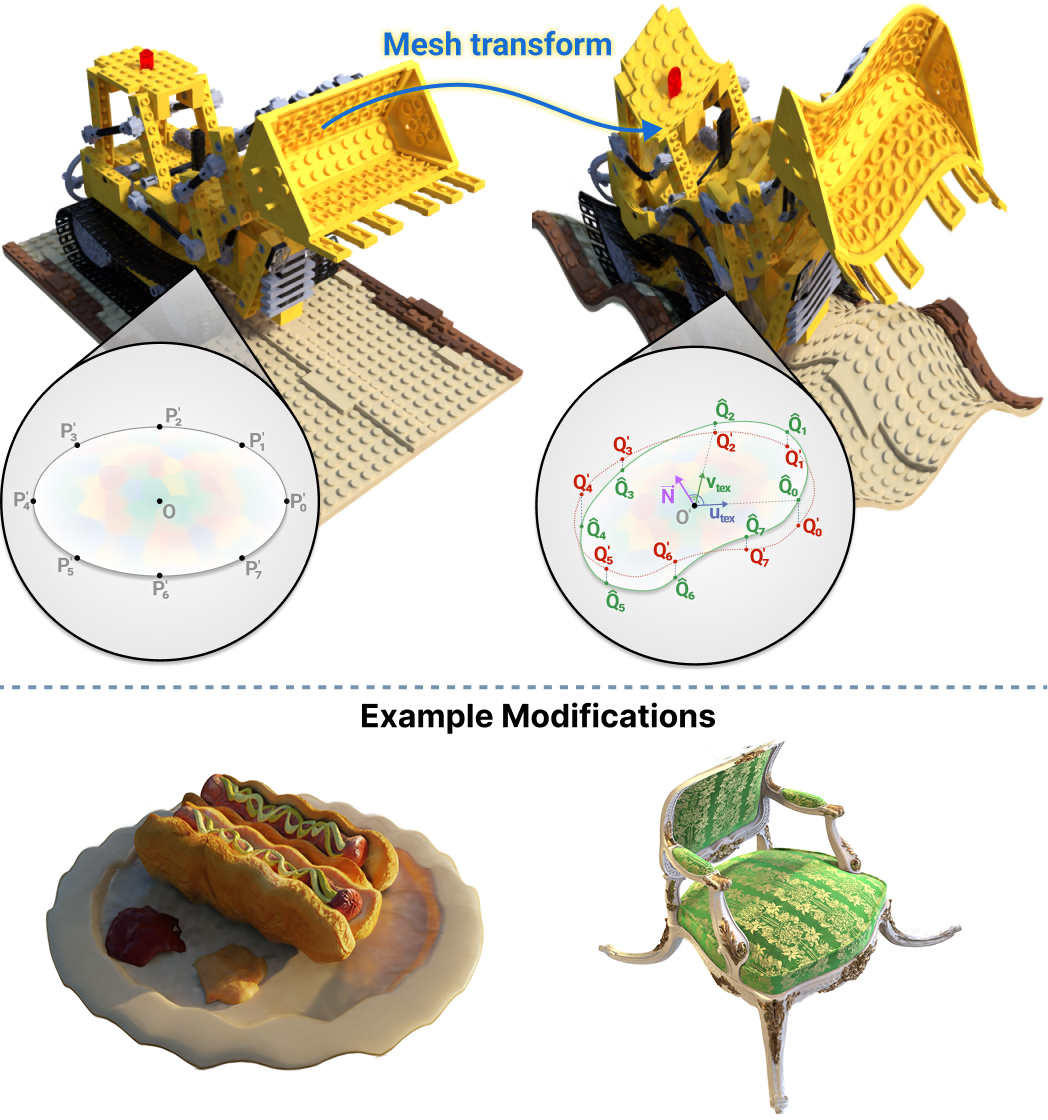}
\caption{Primitive-level mesh modification: deforming the proxy meshes transfers the deformation to the underlying primitives, whose parameters are recovered directly from the deformed geometry without retraining.}
\label{fig:mesh_transform}
\end{figure}

\begin{figure}[t]
    \centering
    \includegraphics[width=\columnwidth]{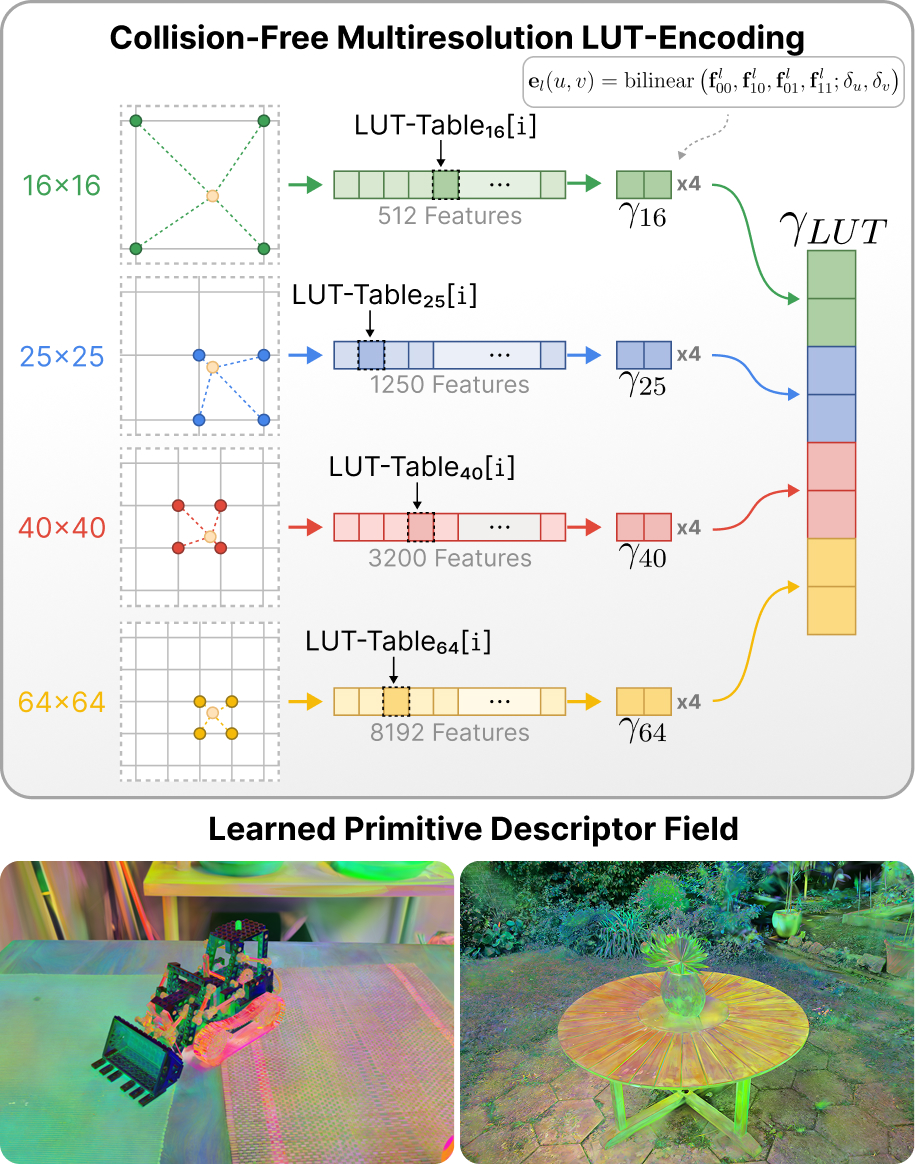}
    \caption{Top: Collision-free multiresolution LUT encoding. Local coordinates $(u,v)$ are bilinearly interpolated on four compact 2D grids and concatenated into $\gamma_{\mathrm{LUT}}$. Bottom: the resulting encoding visualized on reconstructed scenes.}
    \label{fig:lut}
\end{figure}

\paragraph{LUT-Encoding} Prior to entering the network, input data is processed through specific encoding schemes: the spatial coordinates $\left( u_i, v_i \right)$ are mapped using lookup-table encoding (LUT-Encoding), while the view direction vector $\mathbf{d}$ is encoded via positional encoding. These encoded vectors are then concatenated with the individual per-primitive neural radiance field descriptor to form a latent vector, which serves as the input to our MLP (see Figure~\ref{fig:architecture_mlp}).

Our LUT-Encoding mechanism generates latent vectors functionally equivalent to the hash-based encodings proposed in InstantNGP~\cite{instantngp}. However, unlike their approach, which relies on high-dimensional hash-grid lookups, we leverage the fact that our floating primitives already carry a significant portion of the geometric representation. Consequently, we can reduce the encoding overhead by employing noticeably a smaller multiresolution structure. The compact size of the 2D grids enables a direct, collision-free one-to-one mapping between grid vertices and feature vectors, effectively replacing the hashing mechanism with a straightforward lookup table (LUT) (see Figure~\ref{fig:lut}).

\definecolor{rank1}{rgb}{1.0, 0.5, 0.5}   
\definecolor{rank2}{rgb}{1.0, 0.55, 0.52}   
\definecolor{rank3}{rgb}{1.0, 0.60, 0.54}   
\definecolor{rank4}{rgb}{1.0, 0.65, 0.56}   
\definecolor{rank5}{rgb}{1.0, 0.70, 0.58}   
\definecolor{rank6}{rgb}{1.0, 0.75, 0.5}   
\definecolor{rank7}{rgb}{1.0, 0.8, 0.6} 
\definecolor{rank8}{rgb}{1.0, 0.833, 0.6} 
\definecolor{rank9}{rgb}{1.0, 0.866, 0.6} 
\definecolor{rank10}{rgb}{1.0, 0.9, 0.6} 
\definecolor{rank11}{rgb}{1.0, 0.933, 0.6} 
\definecolor{rank12}{rgb}{1.0, 0.966, 0.6} 
\definecolor{rank13}{rgb}{1.0, 1.0, 0.6}  

\newcommand{\rc}[2]{\cellcolor{rank#1}#2}

\begin{table*}[t]
\centering
\small
\setlength{\tabcolsep}{3pt}
\renewcommand{\arraystretch}{1.05}

\begin{tabular}{cl ccc ccc ccc ccc}
\toprule
& &
\multicolumn{6}{c}{Mip-NeRF360} &
\multicolumn{3}{c}{Tanks and Temples} &
\multicolumn{3}{c}{Deep Blending} \\
\cmidrule(lr){3-8}
\cmidrule(lr){9-11}
\cmidrule(lr){12-14}

& Method &
SSIM $\uparrow$ & PSNR $\uparrow$ & LPIPS $\downarrow$ &
Train $\downarrow$ & FPS $\uparrow$ & MEM $\downarrow$ &
SSIM $\uparrow$ & PSNR $\uparrow$ & LPIPS $\downarrow$ &
SSIM $\uparrow$ & PSNR $\uparrow$ & LPIPS $\downarrow$ \\
\midrule

\multirow{3}{*}{\rotatebox[origin=c]{90}{\textbf{RF}}}
& Plenoxels
& \rc{13}{0.670} & \rc{13}{23.63} & \rc{11}{0.440}
& \rc{4}{25min 45s} & \rc{11}{7} & \rc{10}{2.1GB}
& \rc{13}{0.379} & \rc{13}{21.08} & \rc{13}{0.795}
& \rc{13}{0.510} & \rc{13}{23.06} & \rc{13}{0.510} \\

& INGP
& \rc{12}{0.725} & \rc{12}{26.43} & -
& \rc{1}{5m 37s} & \rc{10}{11.7} & \rc{2}{13MB}
& \rc{12}{0.723} & \rc{12}{21.72} & \rc{12}{0.330}
& \rc{12}{0.797} & \rc{12}{23.62} & \rc{12}{0.423} \\

& M-NeRF360
& \rc{7}{0.844} & \rc{1}{29.23} & -
& \rc{12}{48h} & \rc{13}{0.06} & \rc{1}{8.6MB}
& \rc{11}{0.759} & \rc{11}{22.22} & \rc{11}{0.257}
& \rc{6}{0.901} & \rc{9}{29.40} & \rc{4}{0.245} \\

\midrule

\multirow{5}{*}{\rotatebox[origin=c]{90}{\textbf{GS}}}
& Scaffold-GS
& \rc{6}{0.848} & \rc{5}{28.84} & \rc{7}{0.220}
& \rc{3}{24m 22s} & \rc{6}{102} & \rc{3}{156MB}
& \rc{3}{0.853} & \rc{2}{23.96} & \rc{4}{0.177}
& \rc{2}{0.906} & \rc{2}{30.21} & \rc{5}{0.254} \\

& 3DGS
& \rc{1}{0.870} & \rc{7}{28.69} & \rc{7}{0.220}
& \rc{6}{41m 33s} & \rc{4}{134} & \rc{8}{734MB}
& \rc{6}{0.841} & \rc{6}{23.14} & \rc{6}{0.183}
& \rc{4}{0.903} & \rc{8}{29.41} & \rc{3}{0.243} \\

& 3DGRT
& \rc{4}{0.854} & \rc{6}{28.71} & \rc{10}{0.250}
& \rc{8}{47m 49s} & \rc{2}{178} & \rc{6}{383MB}
& \rc{9}{0.830} & \rc{5}{23.20} & \rc{9}{0.222}
& \rc{7}{0.900} & \rc{10}{29.23} & \rc{9}{0.315} \\

& 2DGS
& \rc{4}{0.854} & \rc{9}{28.28} & \rc{5}{0.217}
& \rc{5}{33m 53s} & \rc{7}{62} & \rc{7}{428MB}
& \rc{1}{0.880} & \rc{1}{24.07} & \rc{1}{0.124}
& \rc{4}{0.903} & \rc{7}{29.53} & \rc{6}{0.257} \\

& Nexels
& \rc{3}{0.860} & \rc{4}{28.89} & \rc{2}{0.174}
& \rc{9}{1h 6m} & \rc{8}{50} & \rc{4}{160MB}
& \rc{7}{0.839} & \rc{3}{23.51} & \rc{3}{0.157}
& \rc{1}{0.911} & \rc{1}{30.47} & \rc{1}{0.204} \\

\midrule

\multirow{5}{*}{\rotatebox[origin=c]{90}{\textbf{ENF}}}

& PowerFoam
& \rc{8}{0.840} & \rc{3}{28.91} & \rc{4}{0.210}
& - & \rc{3}{174} & -
& \rc{8}{0.833} & \rc{10}{22.53} & \rc{4}{0.177}
& \rc{11}{0.884} & \rc{11}{28.59} & \rc{8}{0.285} \\

& SVRaster
& \rc{10}{0.822} & \rc{11}{27.34} & \rc{3}{0.186}
& \rc{2}{15m 24s} & \rc{5}{117} & \rc{9}{1.8GB}
& \rc{4}{0.851} & \rc{7}{23.04} & \rc{2}{0.144}
& \rc{8}{0.896} & \rc{4}{29.84} & \rc{2}{0.228} \\

& RadianceMeshes
& \rc{2}{0.867} & \rc{10}{27.99} & \rc{9}{0.230}
& \rc{10}{4h} & \rc{1}{294} & \rc{11}{2.4GB}
& \rc{2}{0.854} & \rc{4}{23.21} & \rc{8}{0.207}
& \rc{3}{0.904} & \rc{6}{29.58} & \rc{11}{0.353} \\

& IRIS
& \rc{11}{0.820} & \rc{8}{28.54} & \rc{1}{0.163}
& \rc{7}{43m 43s} & \rc{12}{0.32} & \rc{12}{9.6GB}
& \rc{5}{0.843} & \rc{9}{22.74} & \rc{10}{0.226}
& \rc{8}{0.896} & \rc{5}{29.83} & \rc{10}{0.323} \\

& \textbf{\our{} (our)}
& \rc{9}{0.836} & \rc{2}{28.93} & \rc{6}{0.219}
& \rc{11}{11h 34m} & \rc{9}{16.53} & \rc{5}{207MB}
& \rc{10}{0.804} & \rc{8}{22.84} & \rc{7}{0.221}
& \rc{8}{0.896} & \rc{3}{29.90} & \rc{7}{0.264} \\

\bottomrule
\end{tabular}

\caption{
Quantitative comparison on Mip-NeRF\,360~\cite{barron2022mipnerf360}, Tanks and Temples~\cite{tandt}, and
Deep Blending~\cite{deepblending}. Methods are grouped by scene
representation: RF denotes radiance fields, GS denotes Gaussian-based
methods, ENF denotes explicit neural fields.
}
\label{tab:quantitative_results}
\end{table*}

\paragraph{Mesh representation of a 2D Gaussian primitive}
\label{mesh_representation}
Following~\cite{3dgrt}, we represent each 2D Gaussian by a triangulated stretched regular polygon that serves as a mesh proxy for hardware-accelerated ray--triangle intersection. To reduce memory usage, all primitives share a single canonical polygon and are instantiated through per-primitive affine transformations. For a canonical vertex
\[
\mathbf{P}_j =
\begin{bmatrix}
\cos\!\left(\frac{2\pi j}{k}\right) ,
\sin\!\left(\frac{2\pi j}{k}\right) ,
0
\end{bmatrix}^T,
\]
the corresponding world-space proxy is obtained using the Gaussian center, rotation, and planar scale. We scale the proxy such that the Gaussian opacity reaches a predefined threshold $\alpha_{\min}$ at its boundary, yielding
\[
\Delta s =
\left(
-2\kappa
\ln\frac{\alpha_{\min}}{\alpha_{\mathrm{const}}}
\right)^{\frac{1}{2\kappa}},
\qquad
\alpha_{\min}\leq\alpha_{\mathrm{const}}.
\]
The resulting proxy therefore tightly bounds the effective support of the generalized Gaussian.

\paragraph{Mesh modification}
Given a deformed proxy polygon with world-space vertices $\mathbf{Q}_i$, we recover the corresponding 2D Gaussian parameters directly from the transformed geometry, without retraining (see Figure~\ref{fig:mesh_transform}). Our procedure is designed to remain stable under non-rigid deformations that may break local coplanarity or alter the primitive orientation. We center the transformed proxy vertices, compute their covariance matrix, and use its SVD to recover the principal tangent directions of the deformed primitive. Since the proxy is planar, we obtain the surface normal as the cross product of these directions and construct the updated local frame from the resulting orthonormal basis. The Gaussian center is set to the transformed polygon centroid, while its planar scales are recovered from the extent of the projected vertices in this local frame, accounting for the clamping factor $\Delta s$.

For purely rigid or axis-aligned deformations, this reconstruction is sufficient. However, under shear, the original orthogonal texture coordinates $(u,v)$ become non-orthogonal, causing a mismatch between the deformed geometry and the primitive's local neural radiance field. We therefore additionally estimate a linear transformation $\mathbf{A}$ that maps the local coordinates into the deformed texture space before they are passed to the MLP.

\paragraph{Color aggregation along the ray}
Following~\cite{kopanas21,kopanas22}, we render each pixel by collecting the ordered intersections of its camera ray with the Gaussian proxy meshes and compositing the corresponding samples front-to-back. For ray $r$, the final color is
$$
\mathbf{c}
=
\sum_{i=1}^{|\mathcal{G}^{(r)}|}
\mathbf{c}_{\mathrm{const},i}
\mathbf{c}_{\mathrm{MLP},i}
\alpha_i
\prod_{j<i}\left(1-\alpha_j\right),
$$
where the effective opacity of the $i$-th intersection is
$$
\alpha_i
=
\alpha_{\mathrm{const},i}
\alpha_{\mathrm{MLP},i}
\mathcal{N}\!\left(u_i,v_i\right).
$$
Thus, the learned color and opacity are modulated by the primitive's generalized Gaussian kernel value before standard alpha compositing.

\paragraph{Loss function}
To stabilize geometry optimization before the local neural radiance fields become expressive, we introduce a warmup stage that smoothly transitions from a simplified base renderer to the full FlaRe model. The base renderer uses only view-independent per-primitive RGB color and opacity, while the full model additionally employs the learned neural radiance descriptors alongside MLP-predicted color and opacity. We optimize
\[
\mathcal{L}
=
(1-\lambda)\mathcal{L}_{\mathrm{base}}
+
\lambda\mathcal{L}_{\mathrm{FlaRe}}
+\lambda_s \mathcal{L}_s,
\]
where
\[
\mathcal{L}_{m}
=
\lambda_{\mathrm{RGB},m}\mathcal{L}_{\mathrm{RGB},m},
\space
m\in\{\mathrm{base},\mathrm{FlaRe}\},
\]
and $\lambda\in[0,1]$ is linearly increased during the warmup.

$\mathcal{L}_{\mathrm{RGB}}$ is the standard $\mathcal{L}_2$ photometric loss. 
Finally,
\[
\mathcal{L}_s
=
\frac{1}{N}\sum_{i=1}^{N}\|\mathbf{s}_i\|_2,
\]
regularizes primitive scales and discourages excessively large Gaussians.

Additional implementation details, derivations, and an extended description of the FlaRe pipeline are provided in the Supplementary Material.

\section{Experiments}

We evaluate \our{} in terms of novel-view synthesis, graphics capabilities, and representation compactness. In addition to standard reconstruction benchmarks, we study recursive ray-tracing, mesh extraction, appearance stylization, primitive reduction, and LUT-Encoding. Experimental details are provided in the Supplementary Material.

\subsection{Novel-View Synthesis}

\paragraph{Quantitative comparison}
Across radiance fields (RF), Gaussian-based methods (GS), and explicit neural fields (ENF), \our{} achieves a favorable balance of novel-view synthesis quality, training cost, rendering speed, and memory, as summarized in Table~\ref{tab:quantitative_results}. On Mip-NeRF\,360, \our{} achieves the second-highest PSNR overall, outperforming 3DGS, 3DGRT, and the remaining ENF baselines. On Deep Blending, \our{} leads the ENF group in PSNR. Compared with other ray-queryable neural representations, it offers a stronger quality--efficiency trade-off: \our{} surpasses IRIS while using substantially less memory and rendering interactively, and improves over Radiance Meshes on both benchmarks at roughly $12\times$ lower memory. Rasterized Gaussian methods remain faster, but \our{} retains competitive reconstruction quality within a compact, explicitly ray-traceable representation.

\begin{figure}[t]
    \centering
    \includegraphics[width=\columnwidth]{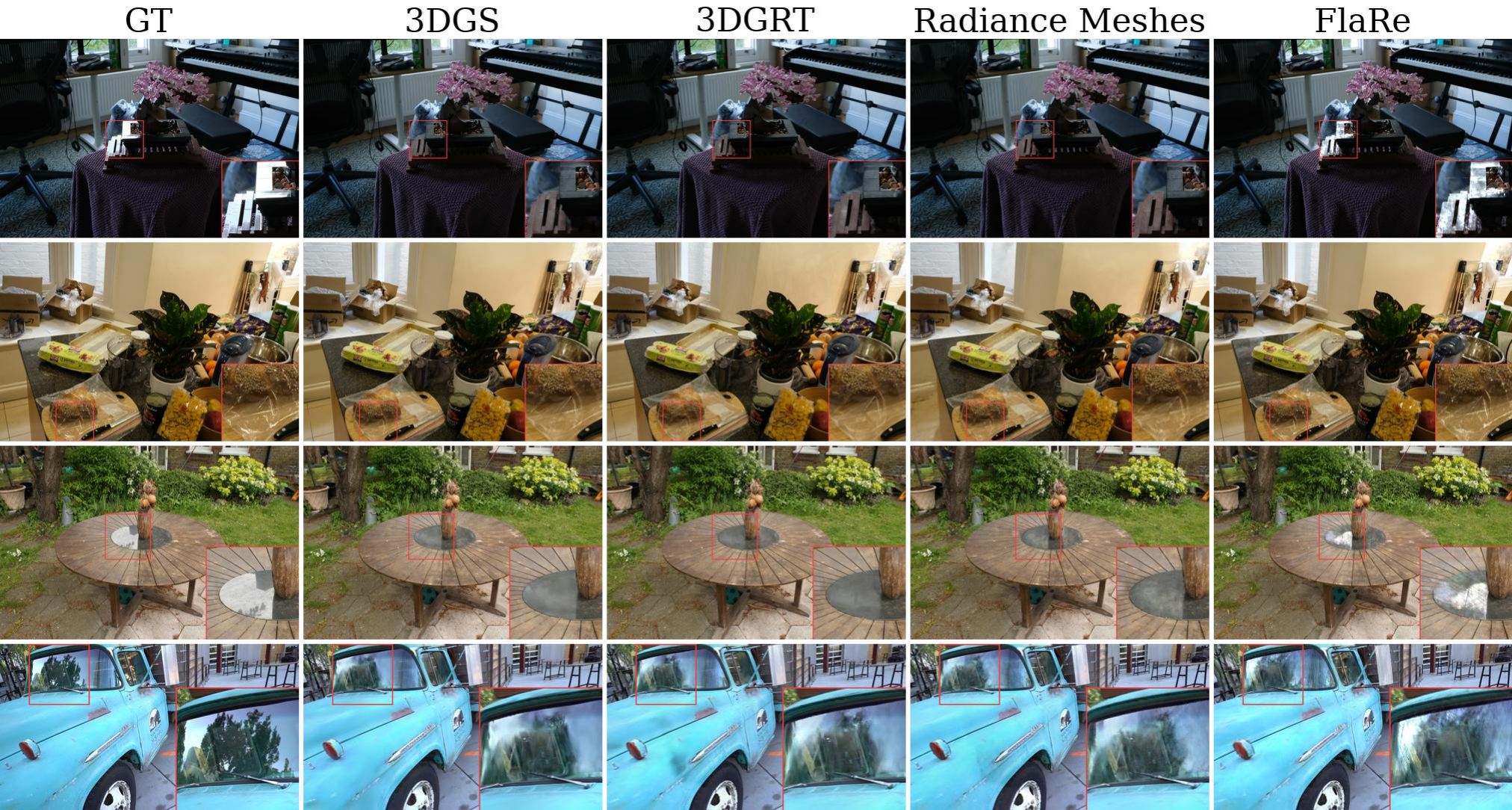}
    \caption{Qualitative comparison on Mip-NeRF\,360. Enlarged crops highlight view-dependent reflections and specularities, which \our{} reconstructs more faithfully than the competing methods.}
    \label{fig:qualitative_reflections}
\end{figure}

\paragraph{Qualitative comparison} 

Compared with 3DGS, 3DGRT, and Radiance Meshes, \our{} more faithfully reconstructs reflections and specular highlights, as shown in Figure~\ref{fig:qualitative_reflections}. Rather than assigning each primitive a fixed or low-order directional color, \our{} predicts radiance at every ray--primitive intersection from the primitive descriptor, local coordinates, and viewing direction. This allows a single primitive to model continuous spatial and angular appearance variation, producing sharper and more accurately positioned highlights.

The structure of the learned appearance space is further illustrated in Figure~\ref{fig:lut}. We project the per-primitive descriptors onto three principal components and map them to RGB. Primitives with similar appearance form coherent regions, whereas distinct materials and patterns separate more clearly. This indicates that the descriptors capture meaningful local appearance information rather than serving merely as primitive identifiers.

\begin{figure}[t]
    \centering
    \includegraphics[width=\columnwidth]{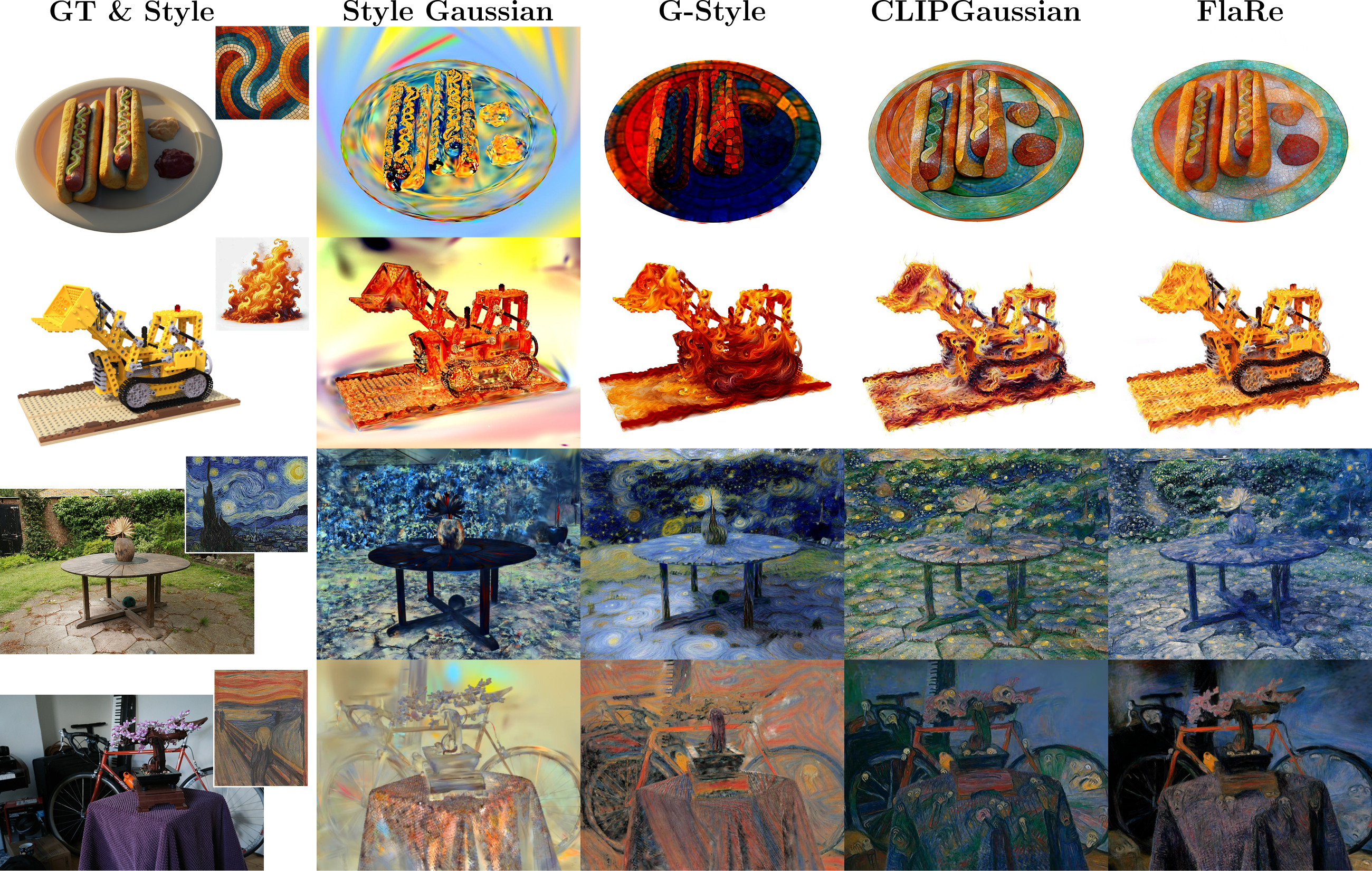}
    \caption{Image-guided 3D style transfer comparison. \our{} reproduces the target style while better preserving scene structure and shading, showing that its learned primitive descriptors provide a directly optimizable appearance representation.}
    \label{fig:style_transfer}
\end{figure}

\subsection{Graphics Capabilities}

\paragraph{Ray-Tracing}

One of our contributions is direct compatibility with recursive ray-tracing. Primary and secondary rays query the same learned primitives through hardware-accelerated intersections, without converting the scene to another representation or retraining it for a specific effect. This enables reflections, refractions, transparency, shadows, and hybrid composition with conventional mesh geometry within a single rendering pipeline. As shown in Figure~\ref{fig:ray_tracing}, mesh objects with different BRDFs interact directly with the neural scene: metallic cars reflect the outdoor Garden environment, while glass, mirror, and metal spheres in Dr.~Johnson exhibit material-dependent reflection and refraction. The Bonsai example further demonstrates multiple recursive interactions, where a refractive dragon and opposing mirrors produce nested reflections containing both mesh objects and reconstructed scene content. We additionally apply the same setup after appearance stylization (see Figure~\ref{fig:teaser}), showing that secondary rays continue to query the modified primitive descriptors without ray-tracing-specific adaptation. These results demonstrate that \our{} supports general ray-based graphics effects within a single editable neural scene representation.

\paragraph{Style Transfer}

    We evaluate whether \our{}'s per-primitive descriptors form a controllable appearance space by adapting the image-guided objective of CLIPGaussian~\cite{clipgaussian} and comparing against StyleGaussian~\cite{liu2024stylegaussian} and G-Style~\cite{kovacs2024G}. Rather than directly optimizing Gaussian appearance parameters, \our{} performs stylization in the descriptor space while fine-tuning the shared decoder. As shown in Figure~\ref{fig:style_transfer}, \our{} transfers the reference palette and visual character while better preserving object structure, shading, and scene-specific color relationships. This suggests that the descriptor space provides a coherent appearance prior across primitives.

\begin{figure}[t]
    \centering
    \includegraphics[width=\columnwidth]{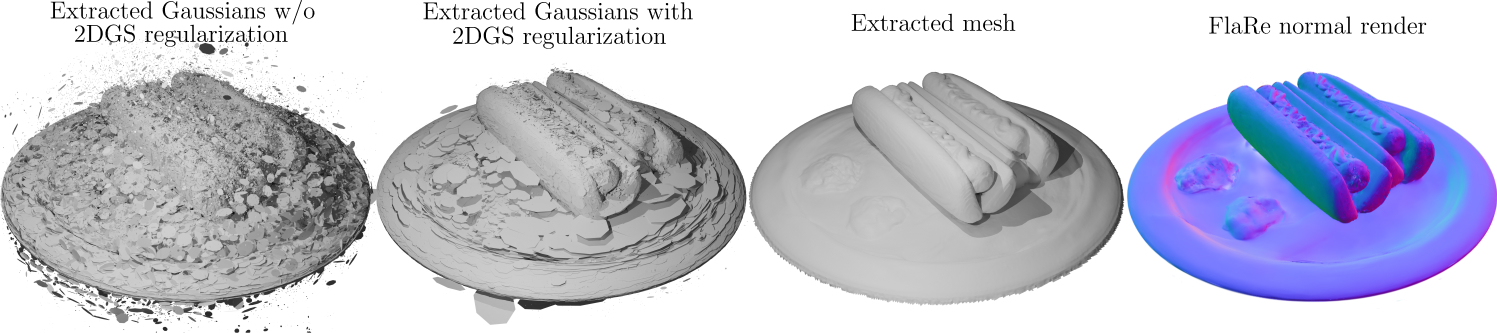}
    \caption{Mesh extraction with \our{}. From left to right: proxy meshes of the learned primitives without and with the 2DGS regularizers, the TSDF-extracted triangle mesh, and the corresponding normal render produced by \our{}.}
    \label{fig:mesh_extraction}
\end{figure}

\paragraph{Mesh Extraction} The explicit planar structure of \our{} allows us to directly adopt the mesh-extraction pipeline of 2DGS~\cite{2dgs}. For mesh-oriented training, we additionally use its depth-distortion and normal-consistency regularizers, activated after 3k and 7k iterations, respectively. As shown in Figure~\ref{fig:mesh_extraction}, these losses substantially improve the geometric alignment of the primitives and enable TSDF-based surface reconstruction. We render depth maps from the training views, fuse them into a TSDF volume, and extract its zero level set as a triangle mesh. During depth rendering, we set $\alpha_{\mathrm{MLP}}=1$ so that the recovered geometry depends only on the explicit primitives rather than appearance-dependent opacity.

\subsection{Representation Analysis}

\paragraph{Primitive reduction.}
We evaluate how well \our{} preserves reconstruction quality when the number of primitives is restricted, comparing against PUP-3DGS~\cite{HansonTuPUP3DGS} and LP-3DGS~\cite{zhang2024lp3dgs} on \emph{Lego} and \emph{Bonsai}. As shown in Figure~\ref{fig:compression}, \our{} matches or surpasses both baselines across comparable primitive budgets, with the largest advantage under aggressive reduction on \emph{Bonsai}. This indicates that its local neural radiance descriptors retain strong appearance capacity even with substantially fewer primitives, while maintaining stable reconstruction quality. Full protocols and results are provided in the Supplementary Material.

\begin{figure}[t]
    \centering
    \includegraphics[width=\columnwidth]{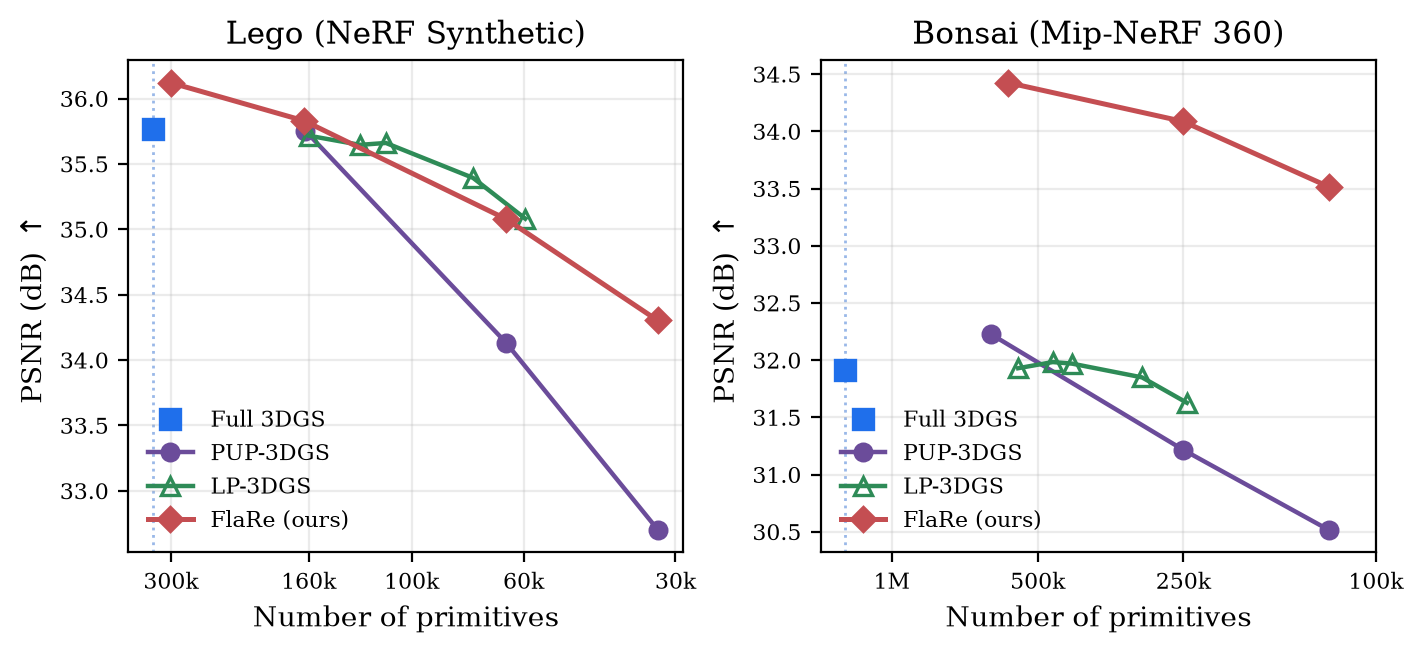}
    \caption{Compression study: test PSNR versus number of primitives on Lego and Bonsai.
    PUP-3DGS~\cite{HansonTuPUP3DGS} prunes a densified 3DGS model to fixed remaining ratios;
    LP-3DGS~\cite{zhang2024lp3dgs} varies the mask sparsity weight $\lambda_{\mathrm{mask}}$;
    \our{} is trained from scratch under matched primitive budgets.
    Vertical dotted lines mark the full densified 3DGS counts.}
    \label{fig:compression}
\end{figure}

\paragraph{LUT-Encoding.}
We isolate the encoding component by training a small coordinate-based network to reconstruct five continuous RGB-valued 2D signals from their spatial coordinates $(u,v)$. The signals include both smoothly varying patterns and localized high-frequency details, allowing us to compare reconstruction fidelity and lookup efficiency. We evaluate our collision-free LUT against hash-grid encodings with matched resolutions or parameter budgets. As shown in Table~\ref{tab:encoding_2d}, LUT-Encoding achieves the highest mean PSNR and the fastest forward evaluation. Its advantage over increasingly constrained hash tables also grows with the collision rate, indicating that direct lookup is better suited to the compact local coordinate domains used by \our{}.

\begin{table}[t]
\centering
\small
\setlength{\tabcolsep}{3.2pt}
\resizebox{\columnwidth}{!}{%
\begin{tabular}{lrr c cc}
\toprule
Encoding & Params & Collisions & Mean PSNR &
\multicolumn{2}{c}{Time (ms) $\downarrow$} \\
\cmidrule(lr){5-6}
& (K) & (\%) & (dB) $\uparrow$ & Forward & Train step \\
\midrule
\textbf{LUT (ours)} & 13.15 & 0.0 & \textbf{42.93} {\scriptsize$\pm$0.61}
& \textbf{1.123} & \textbf{3.917} \\
Hash ($T{=}2048$) & 9.06 & 31.1 & 41.06 {\scriptsize$\pm$0.34}
& \underline{1.208} & \underline{4.018} \\
Hash ($T{=}1024$) & 5.86 & 56.6 & 38.92 {\scriptsize$\pm$0.31}
& 1.295 & 4.130 \\
Hash ($T{=}512$) & 3.58 & 74.3 & 33.98 {\scriptsize$\pm$0.26}
& 1.379 & 4.250 \\
Hash (equal budget) & 13.15 & 91.6 & \underline{41.48} {\scriptsize$\pm$0.27}
& 1.293 & 4.105 \\
\bottomrule
\end{tabular}
}
\caption{Controlled coordinate-based reconstruction of five continuous RGB-valued 2D signals. Mean PSNR is averaged over the signals and random seeds. All encodings use the same training samples, bilinear interpolation, feature dimensionality, decoder, and optimization protocol. ``Train step'' includes both forward and backward passes.}
\label{tab:encoding_2d}
\end{table}

\section{Conclusions}

We introduced \ourfull{} (\our{}), a neural scene representation that combines explicit ray-traceable planar primitives with continuous local neural radiance functions. By coupling compact per-primitive descriptors with a shared auto-decoder, collision-free LUT-Encoding, and hardware-compatible proxy geometry, \our{} retains the expressiveness of neural fields while exposing a directly queryable and manipulable scene structure. Our experiments demonstrate competitive reconstruction quality together with interactive rendering, recursive ray-tracing, descriptor-space style transfer, mesh extraction, and favorable quality under aggressive primitive reduction. These results show that neural scene representations need not be specialized for a single rendering paradigm, and position \our{} as a step toward unified representations that can support both high-quality neural rendering and conventional graphics operations within the same scene model.

\textbf{Limitations.}
The finite set of planar primitives can limit far-field geometric accuracy in large, unbounded outdoor scenes, while training and interactive rendering require a modern GPU with hardware-accelerated ray-tracing.

\textbf{Societal Impact}
As with other photorealistic 3D reconstruction and editing methods, \our{} can support beneficial applications in visualization and content creation but may also facilitate the creation or manipulation of realistic synthetic visual content.

\section{Acknowledgments}
M. Zieliński and D. Belter were supported by the National Science Centre, Poland, under research project no UMO-2023/51/B/ST6/01646. The scientific work was carried out using the infrastructure of the Poznan Supercomputing and Networking Center. The work of P. Spurek was supported by the National Centre of Science (Poland) Grant No. 2023/50/E/ST6/00068.

\bibliography{bibliography}

@String(CVPR= {IEEE Conf. Comput. Vis. Pattern Recog.})

@String(ICCV= {Int. Conf. Comput. Vis.})

@String(ECCV= {Eur. Conf. Comput. Vis.})

@String(TOG= {ACM Trans. Graph.})

@String(CVPR  = {CVPR})

@String(ICCV  = {ICCV})

@String(ECCV  = {ECCV})

@String(TOG   = {ACM TOG})

@inproceedings{mildenhall2020nerf,
 title={{NeRF: Representing Scenes as Neural Radiance Fields for View Synthesis}},
 author={Ben Mildenhall and Pratul P. Srinivasan and Matthew Tancik and Jonathan T. Barron and Ravi Ramamoorthi and Ren Ng},
 year={2020},
 booktitle={ECCV},
}

@article{kerbl20233d,
  title={3D Gaussian Splatting for Real-Time Radiance Field Rendering},
  author={Kerbl, Bernhard and Kopanas, Georgios and Leimk{\"u}hler, Thomas and Drettakis, George},
  journal={ACM Transactions on Graphics},
  volume={42},
  number={4},
  year={2023}
}

@article{instantngp,
    author = {Thomas M\"uller and Alex Evans and Christoph Schied and Alexander Keller},
    title = {Instant Neural Graphics Primitives with a Multiresolution Hash Encoding},
    journal = {ACM Trans. Graph.},
    issue_date = {July 2022},
    volume = {41},
    number = {4},
    month = jul,
    year = {2022},
    pages = {102:1--102:15},
    articleno = {102},
    numpages = {15},
    url = {https://doi.org/10.1145/3528223.3530127},
    doi = {10.1145/3528223.3530127},
    publisher = {ACM},
    address = {New York, NY, USA},
}

@article{barron2021mipnerf,
    title={Mip-NeRF: A Multiscale Representation 
           for Anti-Aliasing Neural Radiance Fields},
    author={Jonathan T. Barron and Ben Mildenhall and 
            Matthew Tancik and Peter Hedman and 
            Ricardo Martin-Brualla and Pratul P. Srinivasan},
    journal={ICCV},
    year={2021}
}

@article{waczynska2024games,
  title={Games: Mesh-based adapting and modification of gaussian splatting},
  author={Waczy{\'n}ska, Joanna and Borycki, Piotr and Tadeja, S{\l}awomir and Tabor, Jacek and Spurek, Przemys{\l}aw},
  journal={arXiv preprint arXiv:2402.01459},
  year={2024}
}

@inproceedings{guedon2024sugar,
  title={Sugar: Surface-aligned gaussian splatting for efficient 3d mesh reconstruction and high-quality mesh rendering},
  author={Gu{\'e}don, Antoine and Lepetit, Vincent},
  booktitle={Proceedings of the IEEE/CVF Conference on Computer Vision and Pattern Recognition},
  pages={5354--5363},
  year={2024}
}

@InProceedings{plenoxels,
    author    = {Fridovich-Keil, Sara and Yu, Alex and Tancik, Matthew and Chen, Qinhong and Recht, Benjamin and Kanazawa, Angjoo},
    title     = {Plenoxels: Radiance Fields Without Neural Networks},
    booktitle = {Proceedings of the IEEE/CVF Conference on Computer Vision and Pattern Recognition (CVPR)},
    month     = {June},
    year      = {2022},
    pages     = {5501-5510}
}

@article{barron2022mipnerf360,
    title={Mip-NeRF 360: Unbounded Anti-Aliased Neural Radiance Fields},
    author={Jonathan T. Barron and Ben Mildenhall and 
            Dor Verbin and Pratul P. Srinivasan and Peter Hedman},
    journal={CVPR},
    year={2022}
}

@article{deepblending,
  title={Deep blending for free-viewpoint image-based rendering},
  author={Hedman, Peter and Philip, Julien and Price, True and Frahm, Jan-Michael and Drettakis, George and Brostow, Gabriel},
  journal={ACM Transactions on Graphics (ToG)},
  volume={37},
  number={6},
  pages={1--15},
  year={2018},
  publisher={ACM New York, NY, USA}
}

@article{tandt,
  title={Tanks and temples: Benchmarking large-scale scene reconstruction},
  author={Knapitsch, Arno and Park, Jaesik and Zhou, Qian-Yi and Koltun, Vladlen},
  journal={ACM Transactions on Graphics (ToG)},
  volume={36},
  number={4},
  pages={1--13},
  year={2017},
  publisher={ACM New York, NY, USA}
}

@misc{radiance_meshes,
      title={Radiance Meshes for Volumetric Reconstruction}, 
      author={Alexander Mai and Trevor Hedstrom and George Kopanas and Janne Kontkanen and Falko Kuester and Jonathan T. Barron},
      year={2025},
      eprint={2512.04076},
      archivePrefix={arXiv},
      primaryClass={cs.GR},
      url={https://arxiv.org/abs/2512.04076}, 
}

@article{barron2023zipnerf,
    title={Zip-NeRF: Anti-Aliased Grid-Based Neural Radiance Fields},
    author={Jonathan T. Barron and Ben Mildenhall and 
            Dor Verbin and Pratul P. Srinivasan and Peter Hedman},
    journal={ICCV},
    year={2023}
}

@inproceedings{nerfcasting,
author = {Verbin, Dor and Srinivasan, Pratul P. and Hedman, Peter and Mildenhall, Ben and Attal, Benjamin and Szeliski, Richard and Barron, Jonathan T.},
title = {NeRF-Casting: Improved View-Dependent Appearance with Consistent Reflections},
year = {2024},
isbn = {9798400711312},
publisher = {Association for Computing Machinery},
address = {New York, NY, USA},
url = {https://doi.org/10.1145/3680528.3687585},
doi = {10.1145/3680528.3687585},
booktitle = {SIGGRAPH Asia 2024 Conference Papers},
articleno = {15},
numpages = {10},
location = {Tokyo, Japan},
series = {SA '24}
}

@article{tramnerf,
	title = {TraM-{NeRF}: Tracing {Mirror} and {Near}-{Perfect} {Specular} {Reflections} {Through} {Neural} {Radiance} {Fields}},
	author = {Holland, Leif Van and Bliersbach, Ruben and M{\" u}ller, Jan Uwe and Stotko, Patrick and Klein, Reinhard},
	journal = {Computer Graphics Forum},
	year = {2024},
	pages = {e15163},
}

@inproceedings{mirror-nerf,
    title={Mirror-NeRF: Learning Neural Radiance Fields for Mirrors with Whitted-Style Ray Tracing},
    author={Zeng, Junyi and Bao, Chong and Chen, Rui and Dong, Zilong and Zhang, Guofeng and Bao, Hujun and Cui, Zhaopeng},
    booktitle={Proceedings of the 31st ACM International Conference on Multimedia},
    pages={4606--4615},
    year={2023}
}

@inproceedings{tensorf,
  author = {Anpei Chen and Zexiang Xu and Andreas Geiger and Jingyi Yu and Hao Su},
  title = {TensoRF: Tensorial Radiance Fields},
  booktitle = {European Conference on Computer Vision (ECCV)},
  year = {2022}
}

@inproceedings{scaffoldgs,
  title={Scaffold-gs: Structured 3d gaussians for view-adaptive rendering},
  author={Lu, Tao and Yu, Mulin and Xu, Linning and Xiangli, Yuanbo and Wang, Limin and Lin, Dahua and Dai, Bo},
  booktitle={Proceedings of the IEEE/CVF Conference on Computer Vision and Pattern Recognition},
  pages={20654--20664},
  year={2024}
}

@inproceedings{2dgs,
    title={2D Gaussian Splatting for Geometrically Accurate Radiance Fields},
    author={Huang, Binbin and Yu, Zehao and Chen, Anpei and Geiger, Andreas and Gao, Shenghua},
    publisher = {Association for Computing Machinery},
    booktitle = {SIGGRAPH 2024 Conference Papers},
    year      = {2024},
    doi       = {10.1145/3641519.3657428}
}

@article{deferred,
author = {Thies, Justus and Zollh\"{o}fer, Michael and Nie\ss{}ner, Matthias},
title = {Deferred neural rendering: image synthesis using neural textures},
year = {2019},
issue_date = {August 2019},
publisher = {Association for Computing Machinery},
address = {New York, NY, USA},
volume = {38},
number = {4},
issn = {0730-0301},
url = {https://doi.org/10.1145/3306346.3323035},
doi = {10.1145/3306346.3323035},
journal = {ACM Trans. Graph.},
month = jul,
articleno = {66},
numpages = {12}
}

@InProceedings{radfoam,
    author    = {Govindarajan, Shrisudhan and Rebain, Daniel and Yi, Kwang Moo and Tagliasacchi, Andrea},
    title     = {Radiant Foam: Real-Time Differentiable Ray Tracing},
    booktitle = {Proceedings of the IEEE/CVF International Conference on Computer Vision (ICCV)},
    month     = {October},
    year      = {2025},
    pages     = {4135-4145}
}

@misc{eks,
      title={Affine-Equivariant Kernel Space Encoding for NeRF Editing}, 
      author={Mikołaj Zieliński and Krzysztof Byrski and Tomasz Szczepanik and Dominik Belter and Przemysław Spurek},
      year={2026},
      eprint={2508.02831},
      archivePrefix={arXiv},
      primaryClass={cs.CV},
      url={https://arxiv.org/abs/2508.02831}, 
}

@article{iris,
  title={IRIS: Intersection-aware Ray-based Implicit Editable Scenes},
  author={Grzegorz Wilczyński and Mikołaj Zieliński and Krzysztof Byrski and Joanna Waczyńska and Dominik Belter and Przemysław Spurek},
  year={2026},
  eprint={2603.15368},
  archivePrefix={arXiv},
  primaryClass={cs.CV},
  url={https://arxiv.org/abs/2603.15368},
}

@article{powerfoam,
  title   = {Power Foam: Unifying Real-Time Differentiable Ray Tracing and Rasterization},
  author  = {Govindarajan, Shrisudhan and Rebain, Daniel and Verbin, Dor and
             Yi, Kwang Moo and Prabhu, Anish and Tagliasacchi, Andrea},
  journal = {arXiv},
  year    = {2026},
}

@article{dnmp,
  author    = {Lu, Fan and Xu, Yan and Chen, Guang and Li, Hongsheng and Lin, Kwan-Yee and Jiang, Changjun},
  title     = {Urban Radiance Field Representation with Deformable Neural Mesh Primitives},
  journal   = {ICCV},
  year      = {2023},
}

@article{billboard,
  title={Gaussian billboards: Expressive 2d gaussian splatting with textures},
  author={Weiss, Sebastian and Bradley, Derek},
  journal={arXiv preprint arXiv:2412.12734},
  year={2024}
}

@article{3dgrt,
    author = {Nicolas Moenne-Loccoz and Ashkan Mirzaei and Or Perel and Riccardo de Lutio and Janick Martinez Esturo and Gavriel State and Sanja Fidler and Nicholas Sharp and Zan Gojcic},
    title = {3D Gaussian Ray Tracing: Fast Tracing of Particle Scenes},
    journal = {ACM Transactions on Graphics and SIGGRAPH Asia},
    year = {2024},
}

@article{3dgut,
    title={3DGUT: Enabling Distorted Cameras and Secondary Rays in Gaussian Splatting},
    author={Wu, Qi and Martinez Esturo, Janick and Mirzaei, Ashkan and Moenne-Loccoz, Nicolas and Gojcic, Zan},
    journal = {Conference on Computer Vision and Pattern Recognition (CVPR)},
    year={2025}
}

@article{raysplat,
  title={Raysplats: Ray tracing based gaussian splatting},
  author={Byrski, Krzysztof and Mazur, Marcin and Tabor, Jacek and Dziarmaga, Tadeusz and K{\k{a}}dzio{\l}ka, Marcin and Baran, Dawid and Spurek, Przemys{\l}aw},
  journal={arXiv preprint arXiv:2501.19196},
  year={2025}
}

@article{redisplat,
      title={REdiSplats: Ray Tracing for Editable Gaussian Splatting}, 
      author={Krzysztof Byrski and Grzegorz Wilczyński and Weronika Smolak-Dyżewska and Piotr Borycki and Dawid Baran and Sławomir Tadeja and Przemysław Spurek},
      year={2025},
      eprint={2503.12284},
      archivePrefix={arXiv},
      primaryClass={cs.CV},
      url={https://arxiv.org/abs/2503.12284}, 
}

@InProceedings{meshsplat_tobiasz,
author="Tobiasz, Rafa{\l}
and Wilczy{\'{n}}ski, Grzegorz
and Mazur, Marcin
and Tadeja, S{\l}awomir
and Smolak-Dy{\.{z}}ewska, Weronika
and Spurek, Przemys{\l}aw",
editor="Neumann, Philipp
and Puma, Michael J.
and Lees, Michael H.
and Groen, Derek
and Dongarra, Jack J.
and Sloot, Peter M. A.",
title="MeshSplats: Mesh-Based Rendering with Gaussian Splatting Initialization",
booktitle="Computational Science -- ICCS 2026",
year="2026",
publisher="Springer Nature Switzerland",
address="Cham",
pages="386--394",
isbn="978-3-032-29924-6"
}

@inproceedings{linprim,
  title={LinPrim: Linear Primitives for Differentiable Volumetric Rendering},
  author={von L{\"u}tzow, Nicolas and Nie{\ss}ner, Matthias},
  booktitle={The Thirty-ninth Annual Conference on Neural Information Processing Systems},
  year={2025}
}

@article{frosting,
title={Gaussian Frosting: Editable Complex Radiance Fields with Real-Time Rendering},
author={Gu{\'e}don, Antoine and Lepetit, Vincent},
journal={ECCV},
year={2024}
}

@InProceedings{gaussian_editor,
    author    = {Chen, Yiwen and Chen, Zilong and Zhang, Chi and Wang, Feng and Yang, Xiaofeng and Wang, Yikai and Cai, Zhongang and Yang, Lei and Liu, Huaping and Lin, Guosheng},
    title     = {GaussianEditor: Swift and Controllable 3D Editing with Gaussian Splatting},
    booktitle = {Proceedings of the IEEE/CVF Conference on Computer Vision and Pattern Recognition (CVPR)},
    month     = {June},
    year      = {2024},
    pages     = {21476-21485}
}

@inproceedings{gaussian_grouping,
    title={Gaussian Grouping: Segment and Edit Anything in 3D Scenes},
    author={Ye, Mingqiao and Danelljan, Martin and Yu, Fisher and Ke, Lei},
    booktitle={ECCV},
    year={2024}
}

@inproceedings{
  clipgaussian,
  title={CLIPGaussian: Universal and Multimodal Style Transfer Based on Gaussian Splatting},
  author={Kornel Howil and Joanna Waczynska and Piotr Borycki and Tadeusz Dziarmaga and Marcin Mazur and Przemys{\l}aw Spurek},
  booktitle={The Thirty-ninth Annual Conference on Neural Information Processing Systems},
  year={2025},
  url={https://openreview.net/forum?id=kjWB8iaO3l},
}

@article{stylegaussian,
  title={StyleGaussian: Instant 3D Style Transfer with Gaussian Splatting},
  author={Liu, Kunhao and Zhan, Fangneng and Xu, Muyu and Theobalt, Christian and Shao, Ling and Lu, Shijian},
  journal={arXiv preprint arXiv:2403.07807},
  year={2024},
}

@article{nexels,
    title={Nexels: Neurally-textured surfels for real-time novel view synthesis with sparse geometries},
    author={Rong, Victor and Held, Jan and Chu, Victor and Rebain, Daniel and
        Van Droogenbroeck, Marc and Kutulakos, Kiriakos N and Tagliasacchi, Andrea and Lindell, David B},
    journal={arXiv preprint arXiv:2512.13796},
    year={2025}
}

@article{kopanas21,
      author       = "Kopanas, Georgios and Philip, Julien and Leimkühler, Thomas and Drettakis, George",
      title        = "Point-Based Neural Rendering with Per-View Optimization",
      journal      = "Computer Graphics Forum (Proceedings of the Eurographics Symposium on Rendering)",
      number       = "4",
      volume       = "40",
      month        = "June",
      year         = "2021",
      url          = "http://www-sop.inria.fr/reves/Basilic/2021/KPLD21"}

@article{kopanas22,
      title={Neural Point Catacaustics for Novel-View Synthesis of Reflections},
      author={Kopanas, Georgios and Leimk{\"u}hler, Thomas and Rainer, Gilles and Jambon, Cl{\'e}ment and Drettakis, George},
      journal={ACM Transactions on Graphics},
      volume={41},
      number={6},
      pages={Article--201},
      year={2022}
    }

@inproceedings{HansonTuPUP3DGS,
  title={Pup 3d-gs: Principled uncertainty pruning for 3d gaussian splatting},
  author={Hanson, Alex and Tu, Allen and Singla, Vasu and Jayawardhana, Mayuka and Zwicker, Matthias and Goldstein, Tom},
  booktitle={Proceedings of the Computer Vision and Pattern Recognition Conference},
  pages={5949--5958},
  year={2025}
}

@article{zhang2024lp3dgs,
  title={Lp-3dgs: Learning to prune 3d gaussian splatting},
  author={Zhang, Zhaoliang and Song, Tianchen and Lee, Yongjae and Yang, Li and Peng, Cheng and Chellappa, Rama and Fan, Deliang},
  journal={Advances in Neural Information Processing Systems},
  volume={37},
  pages={122434--122457},
  year={2024}
}

@incollection{liu2024stylegaussian,
  title={Stylegaussian: Instant 3d style transfer with gaussian splatting},
  author={Liu, Kunhao and Zhan, Fangneng and Xu, Muyu and Theobalt, Christian and Shao, Ling and Lu, Shijian},
  booktitle={SIGGRAPH Asia 2024 Technical Communications},
  pages={1--4},
  year={2024}
}

@inproceedings{kovacs2024G,
  title={G-Style: Stylized Gaussian Splatting},
  author={Kov{\'a}cs, {\'A}ron Samuel and Hermosilla, Pedro and Raidou, Renata G},
  booktitle={Computer Graphics Forum},
  volume={43},
  number={7},
  pages={e15259},
  year={2024},
  organization={Wiley Online Library}
}

@inproceedings{sun2025sparse,
  title={Sparse voxels rasterization: Real-time high-fidelity radiance field rendering},
  author={Sun, Cheng and Choe, Jaesung and Loop, Charles and Ma, Wei-Chiu and Wang, Yu-Chiang Frank},
  booktitle={Proceedings of the IEEE/CVF conference on computer vision and pattern recognition},
  pages={16187--16196},
  year={2025}
}

\newpage
\onecolumn
\begin{center}
    {\LARGE\bfseries Supplementary Materials\par}
    {\LARGE\bfseries \ourfull{}\par}
    \vspace{1cm}
\end{center}

This supplementary material provides additional technical details, experimental results, and implementation information for Floating Radiance Networks (FlaRe). We first extend the discussion of related work and report supplementary evaluations of primitive reduction, qualitative reconstruction quality, and LUT-Encoding. We then present a more detailed description of the FlaRe representation, including its planar generalized Gaussian primitives, shared auto-decoder, collision-free multiresolution LUT-Encoding, polygonal mesh proxies, rendering formulation, optimization procedure, and primitive-level deformation. Finally, we provide reproducibility details and additional analyses that complement the results presented in the main paper.

\section{Extended Related Works}

\paragraph{Radiance fields and learned appearance on explicit primitives.}
Neural Radiance Fields (NeRFs) represent scenes as continuous functions of position and viewing direction and synthesize novel views through differentiable volume rendering~\cite{mildenhall2020nerf}. Subsequent work improved anti-aliasing, unbounded-scene reconstruction, and rendering efficiency through integrated positional encodings, scene contraction, sparse grids, multiresolution hash encodings, and tensor factorizations~\cite{barron2021mipnerf,barron2022mipnerf360,barron2023zipnerf,plenoxels,instantngp,tensorf}. Despite these advances, NeRF-style methods still reconstruct images by evaluating and integrating samples along rays. In contrast, 3D Gaussian Splatting (3DGS)~\cite{kerbl20233d} achieves real-time rendering using explicit anisotropic primitives and visibility-aware rasterization, while 2DGS~\cite{2dgs} replaces volumetric Gaussians with planar primitives to improve geometric consistency. Several methods further combine explicit geometry with learned appearance. Deferred Neural Rendering~\cite{deferred} associates conventional geometry with neural textures, DNMP~\cite{dnmp} attaches latent features to local mesh elements, and Scaffold-GS~\cite{scaffoldgs} predicts view-adaptive Gaussian attributes from anchor features. Gaussian Billboards~\cite{billboard} introduce spatially varying textures over planar Gaussians, whereas EKS~\cite{eks} uses transformation-aware Gaussian feature carriers to condition an editable neural radiance field. IRIS~\cite{iris} instead aggregates features at ordered ray--primitive intersections. These approaches demonstrate the benefits of learned features on explicit structures, but are not designed around continuous local radiance functions evaluated directly on independently ray-traceable primitives. \our{} assigns each floating planar primitive a compact latent descriptor and decodes spatially varying, view-dependent color and opacity at the corresponding ray--primitive intersection using one shared network.

\paragraph{Ray-queryable neural scene representations.}
A growing body of work adapts Gaussian and explicit radiance representations to arbitrary and secondary rays. 3D Gaussian Ray Tracing (3DGRT)~\cite{3dgrt} encloses Gaussian particles in proxy geometry and uses hardware-accelerated ray intersections, while 3DGUT~\cite{3dgut} extends Gaussian rendering to distorted cameras and secondary-ray applications. RaySplats~\cite{raysplat} directly intersects rays with Gaussian confidence ellipses, and REdiSplats~\cite{redisplat} combines polygonal Gaussian proxies with ray tracing and editable scene structure. MeshSplats~\cite{meshsplat_tobiasz} converts a pretrained Gaussian representation into mesh faces, whereas LinPrim~\cite{linprim} directly optimizes polyhedral primitives compatible with mesh-based rendering tools. Beyond Gaussian-derived methods, Radiant Foam~\cite{radfoam}, Power Foam~\cite{powerfoam}, and Radiance Meshes~\cite{radiance_meshes} introduce structured volumetric or mesh-based representations for efficient differentiable ray tracing and, in some cases, rasterization. These methods primarily focus on efficient traversal, graphics-compatible geometry, or conversion from existing representations. In contrast, \our{} makes a shared decoder of continuous local radiance a fundamental component of the ray-traceable representation: every primitive independently carries a compact descriptor that is evaluated using its local coordinates and the ray direction.

\paragraph{Editable and graphics-compatible neural representations.}
Explicit scene structure is also valuable for downstream graphics operations. SuGaR~\cite{guedon2024sugar}, Gaussian Frosting~\cite{frosting}, and 2DGS~\cite{2dgs} introduce surface-aware Gaussian formulations that facilitate mesh extraction. GaMeS~\cite{waczynska2024games} binds Gaussians to mesh faces to support deformation and animation, while REdiSplats~\cite{redisplat} exposes editable ray-traceable proxy geometry. GaussianEditor~\cite{gaussian_editor} and Gaussian Grouping~\cite{gaussian_grouping} enable local manipulation through generative guidance or semantic grouping. Appearance-oriented methods such as StyleGaussian~\cite{stylegaussian} and CLIPGaussian~\cite{clipgaussian} perform image- or text-guided stylization by optimizing Gaussian appearance and geometry. These capabilities are commonly introduced through specialized representations or task-specific extensions. In \our{}, explicit geometry, learned local appearance, and direct ray queries are properties of the same scene model, allowing primitive-level editing, deformation, mesh extraction, recursive ray tracing, and descriptor-space stylization without replacing the underlying representation.

\section{Extended Results}

\subsection{Primitive Reduction}
\label{sec:compression_supp}

We expand the primitive-reduction study from the main paper by comparing \our{} against two recent primitive-reduction baselines for 3DGS on \emph{Lego} (NeRF Synthetic) and \emph{Bonsai} (Mip-NeRF\,360, indoor, half-resolution):
PUP-3DGS~\cite{HansonTuPUP3DGS}, which post-hoc prunes a pretrained densified model using a Fisher sensitivity score followed by short fine-tuning, and
LP-3DGS~\cite{zhang2024lp3dgs}, which learns a differentiable pruning mask jointly with densification.

\paragraph{Protocol.}
For PUP-3DGS we first train a full 3DGS model for $30$k iterations and then prune to approximately $50\%$, $20\%$, and $10\%$ of the remaining Gaussians (PUP $50$, $80$, and cascaded $80{\rightarrow}50$), following the authors' prune--refine recipe.
For LP-3DGS we sweep the mask sparsity weight $\lambda_{\mathrm{mask}}\in\{10^{-2},5{\cdot}10^{-3},10^{-3},5{\cdot}10^{-4},10^{-4}\}$, which induces a continuum of operating points rather than fixed target counts.
For \our{} we train from scratch for $30$k iterations under a hard upper bound on the primitive count, matched to the PUP remaining counts on each scene (and a high reference cap of $300$k on Lego / $2$M on Bonsai).
Mid-training evaluation is disabled so that reported optimization time reflects training only; metrics are measured once from the final checkpoint.

\paragraph{Results.}
The quantitative comparison in Table~\ref{tab:compression} shows that \our{} consistently provides the best PSNR--primitive-count trade-off on Bonsai, with the largest advantage at aggressive budgets. Even with a loose upper bound, its primitive count saturates well below the imposed cap while its reconstruction quality remains above that of the full 3DGS model. On Lego, the gains are smaller but remain consistent at low primitive counts; at moderate budgets, \our{} matches LP-3DGS and outperforms PUP-3DGS. Near the full model size, all three methods approach the quality of densified 3DGS, with \our{} slightly exceeding the full-model baseline.
LP-3DGS preserves quality better than PUP-3DGS at moderate reduction levels, but its learned mask does not reach the most aggressive counts attained by PUP-3DGS and by our budgeted runs.
Overall, the shared neural decoder of \our{} provides a more compact appearance model per primitive, so a given reconstruction fidelity requires substantially fewer primitives than rasterized 3DGS under either post-hoc or learned pruning.

\begin{table}[H]
\centering
\small
\setlength{\tabcolsep}{3.2pt}
\begin{tabular}{ll r ccc}
\toprule
Scene & Method & \#Prim. & PSNR $\uparrow$ & SSIM $\uparrow$ & LPIPS $\downarrow$ \\
\midrule
\multirow{10}{*}{Lego}
 & Full 3DGS & 324{,}826 & 35.77 & 0.983 & 0.016 \\
\cline{2-6}
 & \multirow{3}{*}{PUP-3DGS} & 32{,}482 & 32.70 & 0.969 & 0.039 \\
 &  & 64{,}965 & 34.13 & 0.977 & 0.024 \\
 &  & 162{,}413 & 35.75 & 0.983 & 0.016 \\
\cline{2-6}
 & \multirow{2}{*}{LP-3DGS} & 59{,}531 & 35.08 & 0.981 & 0.019 \\
 &  & 159{,}430 & 35.72 & 0.983 & 0.016 \\
\cline{2-6}
 & \multirow{4}{*}{\our{}} & 32{,}548 & 34.30 & 0.973 & 0.025 \\
 &  & 64{,}967 & 35.08 & 0.977 & 0.020 \\
 &  & 163{,}179 & 35.83 & 0.980 & 0.017 \\
 &  & 300{,}328 & \textbf{36.12} & \textbf{0.981} & \textbf{0.015} \\
\midrule
\multirow{9}{*}{Bonsai}
 & Full 3DGS & 1{,}250{,}561 & 31.91 & 0.940 & 0.206 \\
\cline{2-6}
 & \multirow{3}{*}{PUP-3DGS} & 125{,}056 & 30.52 & 0.919 & 0.251 \\
 &  & 250{,}112 & 31.22 & 0.931 & 0.226 \\
 &  & 625{,}280 & 32.23 & 0.941 & 0.205 \\
\cline{2-6}
 & \multirow{2}{*}{LP-3DGS} & 245{,}755 & 31.63 & 0.935 & 0.219 \\
 &  & 548{,}477 & 31.93 & 0.939 & 0.210 \\
\cline{2-6}
 & \multirow{3}{*}{\our{}} & 125{,}057 & 33.51 & 0.929 & 0.237 \\
 &  & 250{,}128 & 34.09 & 0.937 & 0.223 \\
 &  & 575{,}455 & \textbf{34.43} & \textbf{0.942} & 0.212 \\
\bottomrule
\end{tabular}
\caption{Quantitative primitive-reduction comparison on Lego and Bonsai.
PUP-3DGS rows correspond to $\approx$10\% / 20\% / 50\% remaining Gaussians after prune--refine.
LP-3DGS rows are representative sparsity weights ($\lambda_{\mathrm{mask}}{=}10^{-2}$ and $10^{-4}$); the complete PSNR-versus-count sweep is shown in the primitive-reduction figure in the main paper.
\our{} rows report runs with hard primitive-count caps selected to match the PUP operating points, together with high-cap reference runs ($300$k on Lego and $2$M on Bonsai).}
\label{tab:compression}
\end{table}

\subsection{Qualitative Results}

Figure~\ref{fig:qualitative_surfaces} shows cleaner walls and ceilings with fewer streaks and splat-shaped artifacts. Here, continuous prediction over each primitive footprint reduces piecewise color discontinuities, while the shared decoder gives all primitives a common appearance model. Together with depth-distortion and scale regularization, which discourage inconsistent depth contributions and oversized primitives, this yields a more coherent reconstruction of broad planar surfaces.

\begin{figure}[H]
    \centering
    \includegraphics[width=\columnwidth]{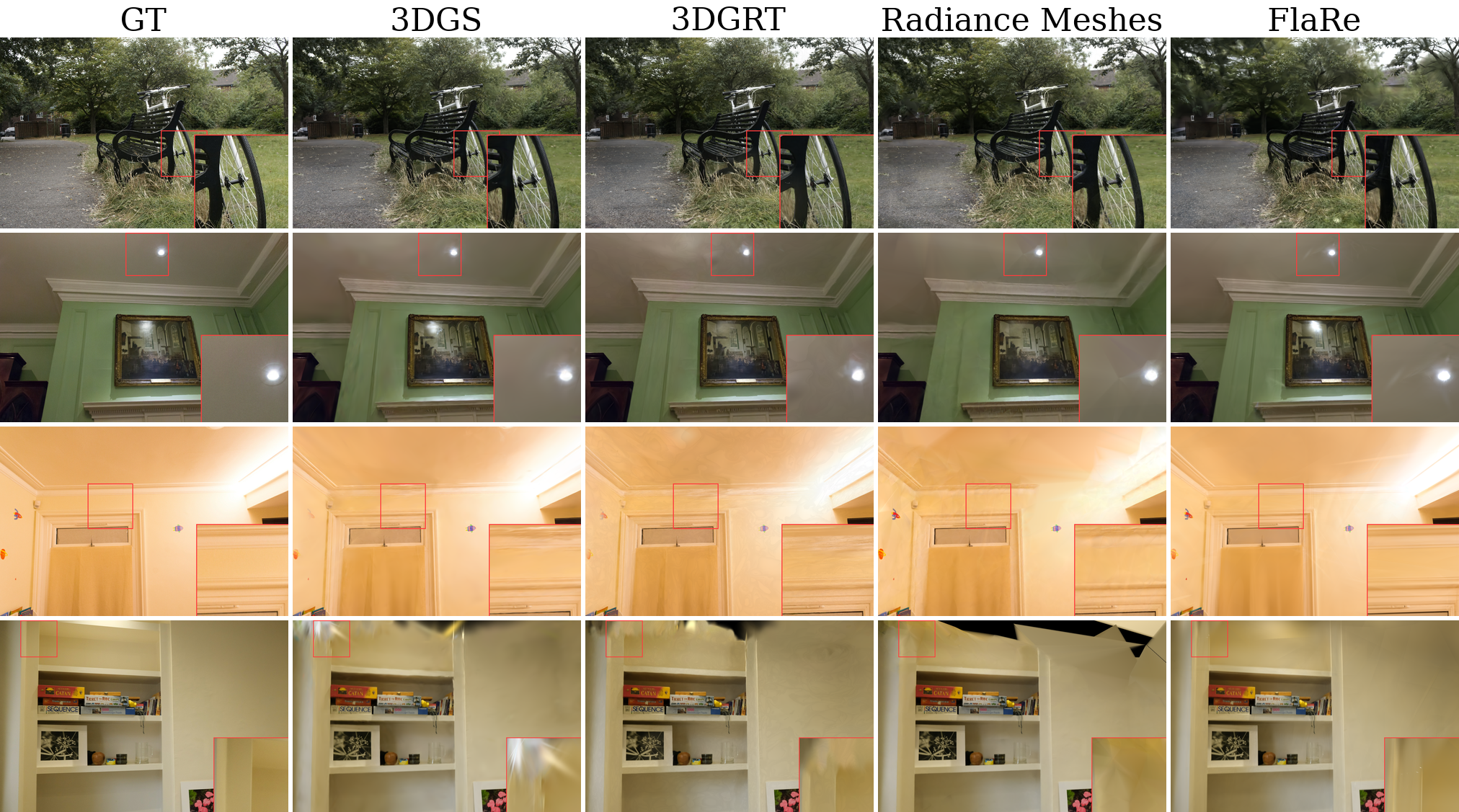}
    \caption{Qualitative comparison on scenes from Mip-NeRF\,360 and Deep Blending. \our{} yields cleaner reconstructions of broad, weakly textured surfaces such as walls and ceilings, with fewer visible rendering artifacts, while preserving fine-scale structures and high-frequency appearance details, as highlighted by the bicycle wheel and its spokes.}
    \label{fig:qualitative_surfaces}
\end{figure}

\subsection{LUT-Encoding}

\paragraph{Controlled 2D encoding experiment.}
We isolate the spatial encoding from rendering and geometry by fitting five continuous RGB functions on $[0,1]^2$: smooth, multiscale Fourier, spatially varying chirp, localized high-frequency, and discontinuous checkerboard signals. At every iteration, we sample 16,384 continuous coordinates uniformly and optimize an identical $8$--$64$--$64$--$3$ ReLU MLP for 5,000 iterations using Adam ($\mathrm{lr}=10^{-3}$, $\beta=(0.9,0.99)$). All encodings use four bilinearly interpolated levels with two features per level, identical initialization and training samples, and are evaluated on a fixed $512^2$ grid over five seeds. We report mean PSNR and encoding-only forward and backward time for batches of 262,144 coordinates. Exact target definitions and configurations are provided in our released evaluation script.

\begin{table}[h!]
\centering
\small
\setlength{\tabcolsep}{3.2pt}
\resizebox{\columnwidth}{!}{%
\begin{tabular}{lrr c c c c c c rr}
\toprule
Encoding & Enc.\ params & Coll. & \multicolumn{5}{c}{PSNR (dB) $\uparrow$} & Mean & \multicolumn{2}{c}{Time (ms) $\downarrow$} \\
\cmidrule(lr){4-8}\cmidrule(lr){10-11}
 & (K) & (\%) & Smooth & Multi. & Chirp & Local & Checker & $\uparrow$ & Fwd. & Fwd.+Bwd. \\
\midrule
\textbf{LUT (ours)} & 13.15 & 0.0 & \textbf{63.36} {\scriptsize$\pm$3.07} & \textbf{53.24} {\scriptsize$\pm$0.35} & \textbf{27.67} {\scriptsize$\pm$0.14} & \underline{43.05} {\scriptsize$\pm$0.38} & \textbf{27.33} {\scriptsize$\pm$0.13} & \textbf{42.93} {\scriptsize$\pm$0.61} & \textbf{1.123} & \textbf{3.917} \\
Hash ($T{=}2048$) & 9.06 & 31.1 & \underline{61.00} {\scriptsize$\pm$1.92} & \underline{51.33} {\scriptsize$\pm$0.13} & \underline{23.70} {\scriptsize$\pm$0.12} & 42.67 {\scriptsize$\pm$1.18} & 26.62 {\scriptsize$\pm$0.16} & 41.06 {\scriptsize$\pm$0.34} & \underline{1.208} & \underline{4.018} \\
Hash ($T{=}1024$) & 5.86 & 56.6 & 59.25 {\scriptsize$\pm$1.21} & 48.72 {\scriptsize$\pm$0.27} & 19.11 {\scriptsize$\pm$0.30} & 42.82 {\scriptsize$\pm$0.14} & 24.68 {\scriptsize$\pm$0.15} & 38.92 {\scriptsize$\pm$0.31} & 1.295 & 4.130 \\
Hash ($T{=}512$) & 3.58 & 74.3 & 48.78 {\scriptsize$\pm$0.46} & 45.14 {\scriptsize$\pm$0.35} & 16.12 {\scriptsize$\pm$0.22} & 39.51 {\scriptsize$\pm$0.65} & 20.36 {\scriptsize$\pm$0.38} & 33.98 {\scriptsize$\pm$0.26} & 1.379 & 4.250 \\
Hash (equal budget) & 13.15 & 91.6 & 59.12 {\scriptsize$\pm$0.62} & 49.47 {\scriptsize$\pm$0.48} & 20.05 {\scriptsize$\pm$0.08} & \textbf{51.72} {\scriptsize$\pm$1.28} & \underline{27.06} {\scriptsize$\pm$0.21} & \underline{41.48} {\scriptsize$\pm$0.27} & 1.293 & 4.105 \\
\bottomrule
\end{tabular}
}
\caption{Controlled 2D encoding comparison. All methods use identical bilinear interpolation, feature dimensionality, decoder, training samples, and optimization settings. Mean $\pm$ standard deviation over 5 seeds. Collision rate is measured over all virtual grid vertices. Best and second-best values are shown in \textbf{bold} and \underline{underline}, respectively.}
\label{tab:encoding_2d}
\end{table}

\paragraph{Baselines.}
Our LUT uses collision-free grids of resolutions $[16,25,40,64]$, containing 13,154 trainable scalars. We compare it with hash grids using the same virtual resolutions and per-level capacities $T\in\{2048,1024,512\}$, resulting in 9,058, 5,858, and 3,584 parameters, respectively. Dense levels are addressed directly, whereas levels exceeding $T$ use Instant-NGP-style hashing. We additionally construct an equal-budget hash grid with resolutions $[16,40,101,256]$ and $T=2360$. It contains 13,152 parameters, approximately matching the LUT, but exchanges collision-free storage for higher virtual resolution.

\paragraph{Results.}
LUT-Encoding achieves the highest average PSNR (42.93\,dB) and the lowest forward time (1.123\,ms). For the fixed-resolution hash grids, increasing collision rates from 31.1\% to 74.3\% reduces average PSNR from 41.06 to 33.98\,dB, while hashing adds 7.5--22.8\% forward overhead. The equal-budget hash grid reaches 41.48\,dB, remaining 1.45\,dB below LUT on average. LUT performs best on smooth, multiscale, chirp, and discontinuous signals because their domains are densely sampled: one-to-one addressing preserves independent vertex features, whereas collisions couple unrelated locations and their gradients. The only exception is localized detail, where the equal-budget hash grid obtains 51.72\,dB compared with 43.05\,dB for LUT. This target contains a localized 34-cycle component, exceeding the approximately 31.5-cycle Nyquist limit of the LUT's $64$-vertex level; the hash grid's virtual resolution of 256 resolves this sparse detail, while the collision cost remains spatially limited.

\begin{figure}[H]
    \centering
    \includegraphics[width=\textwidth]{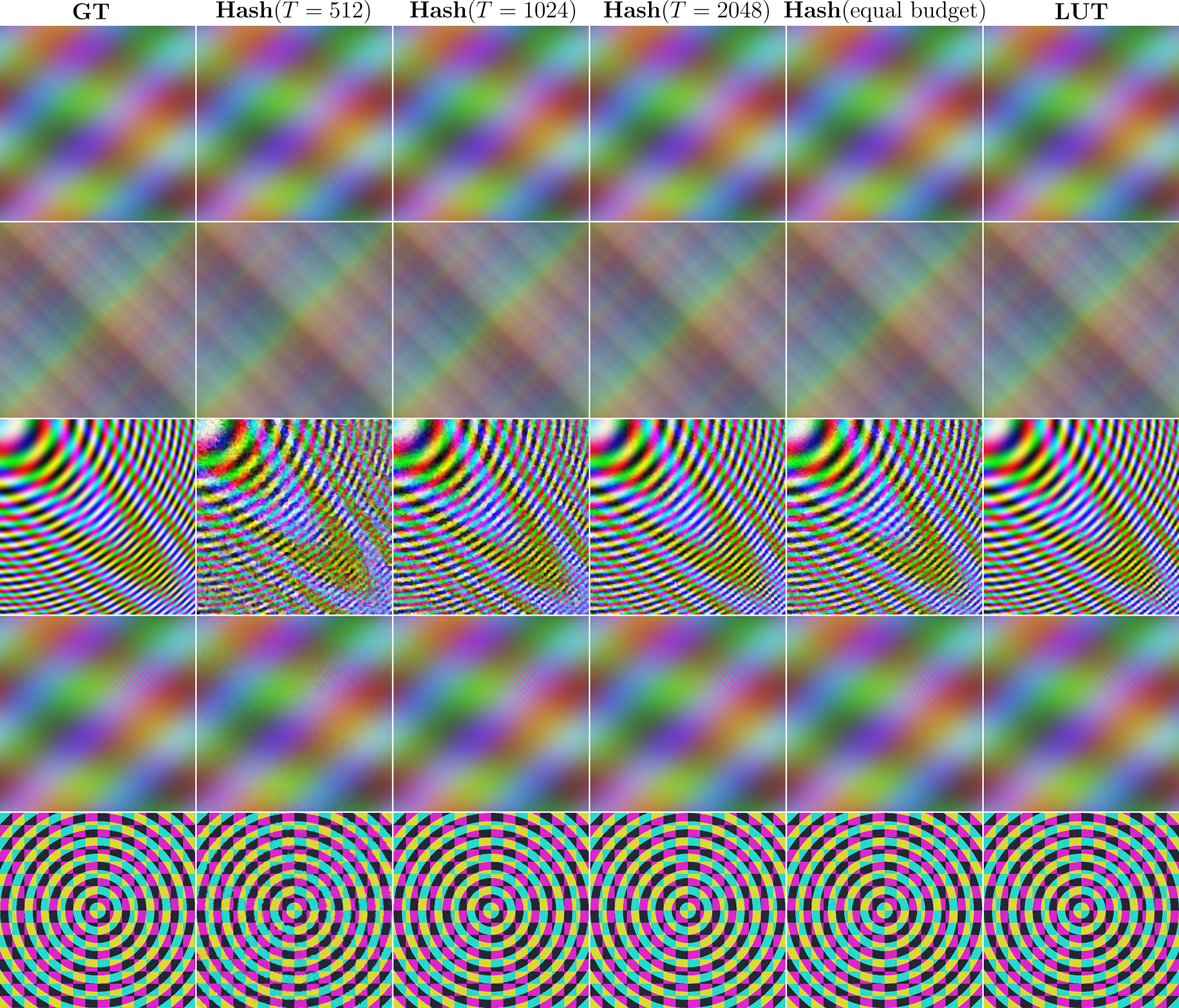}
    \caption{Qualitative comparison on the controlled 2D approximation tasks. Columns show the ground truth followed by hash-grid variants of increasing capacity, the equal-budget hash grid, and our collision-free LUT-Encoding. Rows correspond to the smooth, multiscale, chirp, localized-detail, and checkerboard signals.}
    \label{fig:lut_ablation}
\end{figure}

\section{Reproducibility}

\paragraph{Datasets.}
We evaluate on Mip-NeRF\,360~\cite{barron2022mipnerf360}, Tanks and Temples, Deep Blending, and NeRF Synthetic~\cite{mildenhall2020nerf}. These datasets are industry-standard benchmarks for novel-view synthesis and enable direct comparison with prior radiance-field and Gaussian-based methods under commonly used train/test splits.

\paragraph{Metrics.}
We report PSNR, SSIM, and LPIPS, the standard metrics for this task. They quantify pixel fidelity, structural similarity, and perceptual quality, respectively, and are the common reporting suite in the literature we compare against. Hyperparameter selection during tuning is guided by validation/test PSNR.

\paragraph{Data preprocessing.}
No dataset-specific image preprocessing is required (e.g., no custom rescaling or color normalization beyond the standard loaders used for these benchmarks). Posed RGB images and camera parameters are consumed in their usual released form. Operations such as nearest-neighbor distance estimation for initializing primitive scales are part of model initialization, not an external data-preprocessing stage.

\paragraph{Randomness.}
Training uses fixed configuration defaults. A subset of primitive parameters (e.g., latent descriptors and network weights) is randomly initialized at the start of optimization; afterward, densification/pruning and optimization are deterministic given that initialization and the data order. At the beginning of every training run seeds are set. Residual nondeterminism may arise from GPU/CUDA kernels and hardware-accelerated ray tracing.

\paragraph{Hyperparameter search.}
We tune hyperparameters by coordinate descent: for a selected macro-parameter we evaluate a small set of candidates in the range $[\mathrm{selected}/100,\,100\cdot\mathrm{selected}]$, where $\mathrm{selected}$ is the current best value (or a functionally analogous value from a previous model), keep the candidate with the best PSNR, and move to the next parameter. Final reported numbers use the configuration available in the code supplement.

\paragraph{Compute setup.}
All main experiments were run on a single NVIDIA RTX\,3090 GPU with CUDA/OptiX hardware-accelerated ray--triangle intersection. 

\section{Extended Description of Floating Radiance Networks}
\begin{figure}[H]
    \centering
    \includegraphics[width=\textwidth]{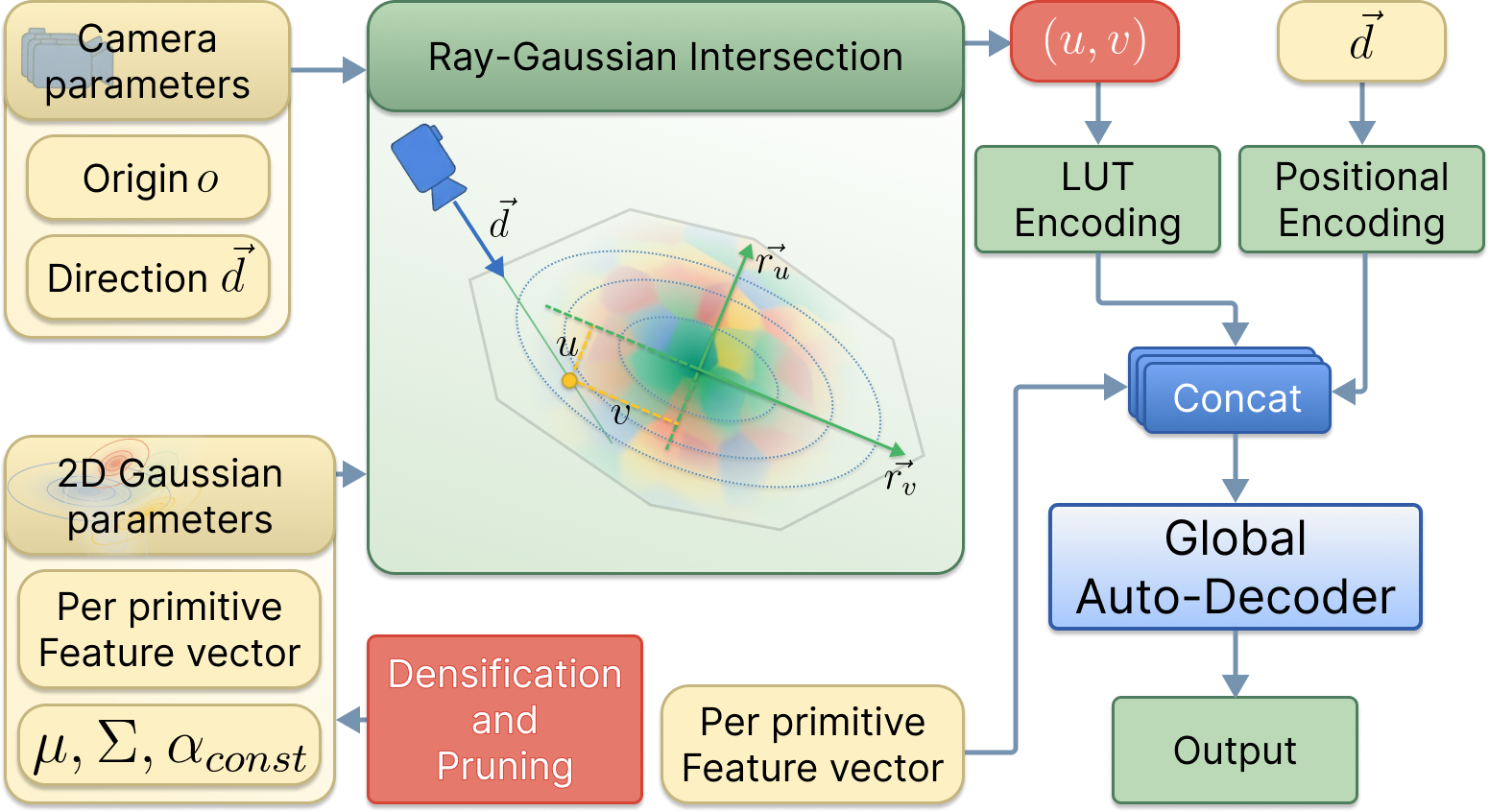}
    \caption{The Floating Radiance Networks (FlaRe) pipeline}
    \label{fig:architecture}
\end{figure}

\subsection{Gaussian Primitive Scene Representation}
FlaRe represents the scene as a collection of anisotropic planar generalized Gaussian primitives with an isotropic kernel defined in normalized local coordinates.:
$$
\mathcal{G}
=
\left\{
\left(
\mathcal{N}(\upmu_i,\Sigma_i,\kappa_i), \mathrm{c}_{\text{const},i}, \alpha_{\mathrm{const},i},
\mathrm{z}_i
\right)
\right\}_{i=1}^{N}.
$$
where:
\begin{itemize}
\item $\mathcal{N} \left( \upmu_i, \Sigma_i, \kappa_i \right)$ denotes the Isotropic Planar Generalized Gaussian Kernel parametrized by:
\begin{itemize}
\item $\upmu_i \in \mathbb{R}^3$: the trainable mean of the Gaussian;
\item $\Sigma_i \in \mathbb{R}^{3 \times 3}$: the singular covariance matrix defined via the rotation matrix $\mathrm{R}_i$ (derived from the trainable quaternion $\mathrm{q}_i$) and the scaling matrix
$
\mathrm{S}_i =
\operatorname{diag}\left(
\begin{bmatrix}
\mathrm{s}_i, 0
\end{bmatrix}
\right) = 
\operatorname{diag}\left(
\begin{bmatrix}
\exp \left( \tilde{\mathrm{s}}_i \right), 0
\end{bmatrix}
\right)
$
for the trainable parameter $\tilde{\mathrm{s}}_i \in \mathbb{R}^2$ controlling the primitive scale in its local tangent plane, admitting the following factorization:$$\Sigma_i = \mathrm{R}_i \cdot \mathrm{S}_i \cdot \mathrm{S}_i^{\mathrm{T}} \cdot \mathrm{R}_i^{\mathrm{T}}$$
\item $\kappa_i = 1 + \operatorname{softplus} \left( \tilde{\kappa}_i \right)$: the shape parameter of the Gaussian for the trainable parameter $\tilde{\kappa}_i \in \mathbb{R}$;
\end{itemize}
\item $\mathrm{c}_{\text{const},i} \in \mathbb{R}^3$ represents the trainable view-independent constant color parameter of the $i$-th Gaussian in the RGB space;
\item $\alpha_{\text{const},i} = \sigma \left( \tilde{\alpha}_{\text{const},i} \right)$ defines the constant opacity of the $i$-th Gaussian for the trainable parameter $\tilde{\alpha}_{\text{const},i} \in \mathbb{R}$;
\item $\mathrm{z}_i \in \mathbb{R}^{96}$ serves as the trainable per-primitive local 2D neural radiance field descriptor;
\item $N$ is the total number of the primitives;
\end{itemize}

\subsection{Isotropic Planar Generalized Gaussian Kernel}
Following the approach presented in Nexels~\cite{nexels}, FlaRe employs opacity modulation via an unnormalized generalized Gaussian Kernel in the underlying 2D representation primitives. This formulation enables more precise modeling of sharp boundaries and high-curvature surfaces, effectively mitigating approximation errors inherent in classical Gaussian kernels characterized by a fixed quadratic fall-off profile.
\begin{figure}[H]
    \centering
    \begin{subfigure}{0.33\columnwidth}
        \centering
        \includegraphics[width=\linewidth]{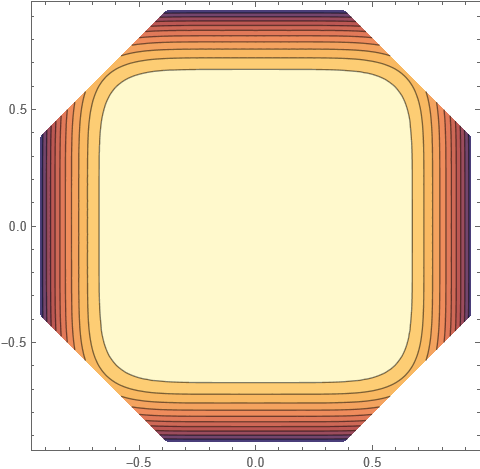}
        \caption{Rong et al. ($\gamma_1=3, \gamma_2=3$)}
        \label{fig:pdf1}
    \end{subfigure}
    \quad 
    \begin{subfigure}{0.33\columnwidth}
        \centering
        \includegraphics[width=\linewidth]{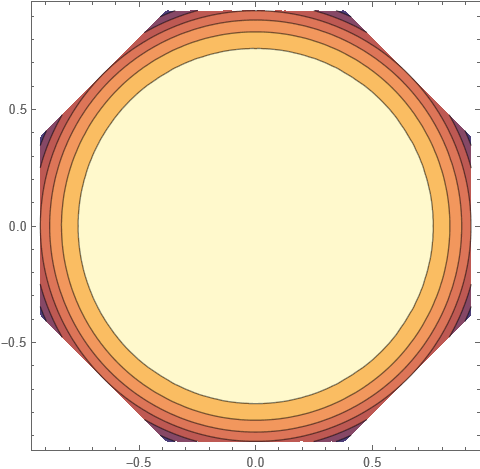}
        \caption{Ours ($\kappa=3$)}
        \label{fig:pdf2}
    \end{subfigure}
    \caption{Comparison of kernel value plots for two distinct Generalized Gaussian Kernel formulations: Rong et al. (left) vs. our Isotropic Planar Generalized Gaussian Kernel (right).}
    \label{fig:distribution-comparison}
\end{figure}
\noindent
Unlike the approach by~\cite{nexels}, where the value of the final opacity $\alpha$ is computed using the following formula:
$$
\alpha = o \exp \left( {-\frac{ { \left \lvert u \right \rvert }^{2{\gamma}_1} }{2}} \right) \exp \left( {-\frac{ { \left \lvert v \right \rvert }^{2{\gamma}_2} }{2}} \right)
$$
we adopt a modified definition of the generalized Gaussian Kernel resulting in our Isotropic Planar Generalized Gaussian Kernel, given by the following formula:
$$
\mathcal{N} \left( \upmu, \Sigma, \kappa \right) \left( \mathrm{x} \right) = \exp \left( -\frac{ { \left \lVert \mathrm{P} \mathrm{S}^{-1} \mathrm{R}^{\mathrm{T}} \left( \mathrm{x} - \upmu \right) \right \rVert }^{2{\kappa}}_2}{2{\kappa}} \right)
$$
where $\mathrm{P}$ is the projection matrix, defined as:
$$
\mathrm{P} = \begin{bmatrix}
1 & 0 & 0 \\
0 & 1 & 0
\end{bmatrix}
$$
Note that for $\left(u, v \right)$ expressed in the Gaussian local coordinate frame, the above-mentioned formula simplifies to:
$$
\mathcal{N} \left( \upmu, \Sigma, \kappa \right) \left( u, v \right) = \exp \left( -\frac{ { \left( u^2 + v^2 \right) }^{\kappa}}{2{\kappa}} \right)
$$
yielding the following value of the final modulated opacity $\alpha$:
$$
\alpha = \alpha_{\text{const}} \cdot \alpha_{\text{MLP}} \cdot \exp \left( -\frac{ { \left( u^2+v^2 \right) }^{\kappa} }{2\kappa} \right)
$$
where $\alpha_{\text{MLP}}$ denotes the opacity predicted by our auto-decoder MLP as detailed in the Subsection \textbf{Global auto-decoder}.

This approach allows us to utilize a single canonical polygon for each Gaussian, regardless of the $\kappa$ parameter, since the density function's isocontours remain isotropic. In contrast, for the density function proposed by~\cite{nexels}, the polygon's geometry does not align with the isocontours, preventing the density values from vanishing at the boundary and necessitating the computation of a unique proxy polygon for every individual value of the shape parameter.

\subsection{Global auto-decoder architecture}
\label{auto_decoder}
In contrast to conventional Gaussian-based representations (e.g., 3DGS~\cite{kerbl20233d} or 3DGRT~\cite{3dgrt}), which rely on static, per-primitive color and opacity attributes — typically optimized as trainable constants modulated by a spatial Gaussian kernel — FlaRe adopts a more flexible approach. Specifically, both the color and opacity of samples within each 2D Gaussian footprint are predicted by a dedicated 2D neural radiance field, uniquely assigned to each individual primitive and operating within its specific local coordinate frame, and encoded individually for each Gaussian primitive as a 96-element FP16 trainable latent descriptor vector. The color and opacity are then decoded using a single global auto-decoder MLP:
$$
\Phi \left( \left( u^{\left( r \right)}_i, v^{\left( r \right)}_i \right), \mathrm{\vec{d}}^{\left( r \right)}, \mathrm{z}^{\left( r \right)}_i \;; \mathcal{G}, \mathcal{F}, \Theta \right) = \left( \mathrm{c}^{\left( r \right)}_{\text{MLP},i}, \alpha^{\left( r \right)}_{\text{MLP},i} \right)
$$
where $\left( \mathrm{c}^{\left( r \right)}_{\text{MLP},i}, \alpha^{\left( r \right)}_{\text{MLP},i} \right) \in \mathbb{R}^3 \times \mathbb{R}$ denotes the color and opacity predicted by the MLP at the intersection point $\mathrm{P}^{\left( r \right)}_i$ of ray $r$ and the $i$-th primitive encountered by the ray, given the direction vector $\mathrm{\vec{d}}^{\left( r \right)}$ and per-primitive local 2D neural radiance field descriptor $\mathrm{z}^{\left( r \right)}_i$ with $\mathrm{P}^{\left( r \right)}_i$ having the coordinates: $\left( u^{\left( r \right)}_i, v^{\left( r \right)}_i \right)$ within the primitive's local reference frame. The network is conditioned on the following entities:
\begin{itemize}
\item $\mathcal{G}$ representing the set of all Gaussian primitives;
\item $\mathcal{F}$ defining the set of the LUT-Table features;
\item $\Theta$ denoting the MLP parameters.
\end{itemize}
Our global auto-decoder MLP adopts a lightweight architecture, as detailed in Figure~\ref{fig:mlp}:
\begin{figure}[H]
    \centering    \includegraphics[width=0.4\columnwidth]{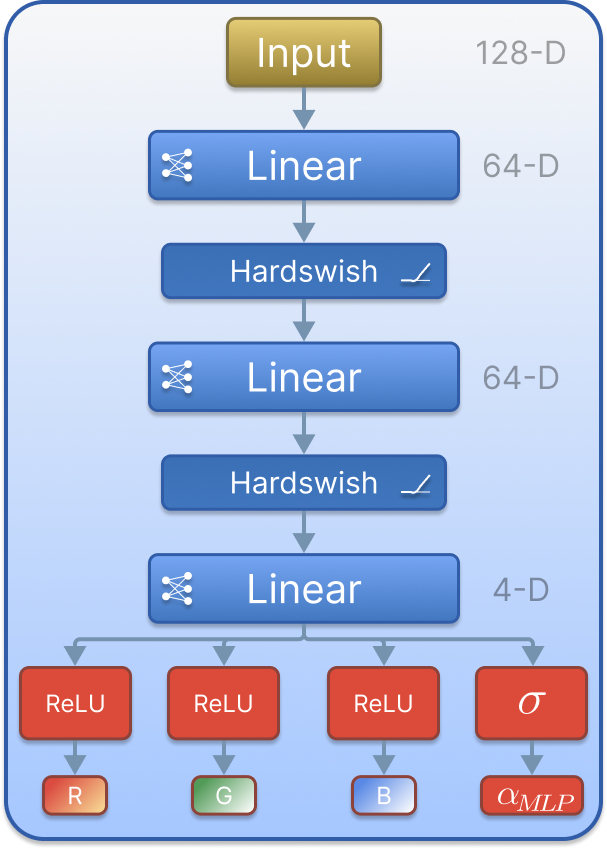}
    \caption{FlaRe MLP}
    \label{fig:mlp}
\end{figure}
\noindent
The architecture comprises a 128-wide input layer, two 64-neuron hidden layers each followed by a Hardswish activation, and a 4-neuron output layer, which branches into four channels processed by three ReLU and one $\sigma$ activation functions, respectively.

Prior to entering the network, input data are processed through specific encoding schemes: the spatial coordinates $\left( u^{\left( r \right)}_i, v^{\left( r \right)}_i \right)$ are mapped using LUT-Encoding to produce an $8$-dimensional vector ($4$ levels $\times$ $2$ feature dimensions), while the view direction vector $\mathrm{\vec{d}}^{\left( r \right)}$ is encoded via positional encoding across 4 frequency bands yielding a 24-dimensional vector ($3$ coordinates $\times$ $4$ bands $\times$ $2$ trigonometric functions). These encoded vectors are then concatenated with the $96$-dimensional individual per-primitive neural radiance field descriptor to form a 128-dimensional latent vector, which serves as the input to our MLP (see Figure~\ref{fig:architecture}). The forward pass of the neural network $\Phi$ can be expressed by the following formula:
$$
\Phi \left( \left( u^{\left( r \right)}_i, v^{\left( r \right)}_i \right), \mathrm{\vec{d}}^{\left( r \right)}, \mathrm{z}^{\left( r \right)}_i \;; \mathcal{G}, \mathcal{F}, \Theta \right) =
\mathrm{W}_3 \varphi_3 \left( \mathrm{W}_2 \varphi_2 \left( \mathrm{W}_1 \varphi_1 \left(
\begin{bmatrix}
E_{\text{LUT}} \left( u^{\left( r \right)}_i, v^{\left( r \right)}_i \right) \\
\gamma \left( \mathrm{\vec{d}}^{\left( r \right)} \right) \\
\mathrm{z}^{\left( r \right)}_i
\end{bmatrix}
\right) + \mathrm{b}_1 \right) + \mathrm{b}_2 \right) + \mathrm{b}_3
$$
where:
$$
\varphi_1 \left( \mathrm{x} \right) = \begin{bmatrix}
\operatorname{hardswish} \left( x_1 \right) \\
\vdots \\
\operatorname{hardswish} \left( x_{128} \right)
\end{bmatrix}
$$
$$
\varphi_2 \left( \mathrm{x} \right) = \begin{bmatrix}
\operatorname{hardswish} \left( x_1 \right) \\
\vdots \\
\operatorname{hardswish} \left( x_{64} \right)
\end{bmatrix}
$$
$$
\varphi_3 \left( \mathrm{x} \right) = \begin{bmatrix}
\operatorname{ReLU} \left( x_1 \right) \\
\operatorname{ReLU} \left( x_2 \right) \\
\operatorname{ReLU} \left( x_3 \right) \\
\sigma \left( x_4 \right) \\
\end{bmatrix}
$$
Our LUT-Encoding mechanism generates latent vectors functionally equivalent to the hash-based encodings proposed in Instant-NGP~\cite{instantngp}. However, unlike their approach, which relies on high-dimensional hash-grid lookups, we leverage the fact that our floating primitives already carry a significant portion of the geometric representation. Consequently, we can reduce the encoding overhead by employing a 4-level multiresolution structure with grid resolutions of: $16 \times 16$, $25 \times 25$, $40 \times 40$, and $64 \times 64$. The compact size of these 2D grids enables a direct, collision-free one-to-one mapping between grid vertices and feature vectors, effectively replacing the hashing mechanism with a straightforward lookup table (LUT) (see Figure~\ref{fig:lut}). The two-dimensional feature vectors are serialized in row-major order within fast on-chip memory across the four consecutive levels.
\begin{figure}[H]
    \centering
    \includegraphics[width=0.5\columnwidth]{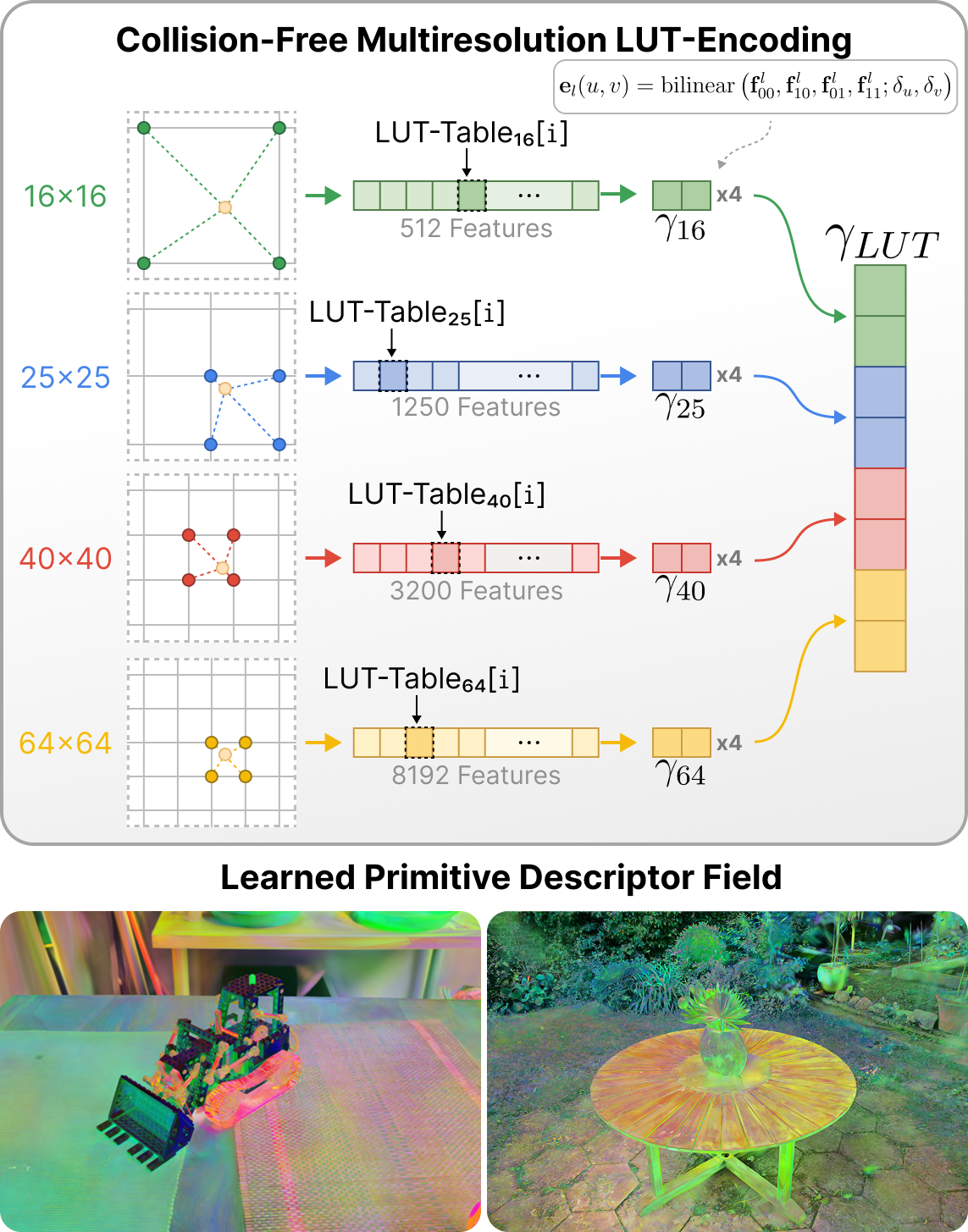}
    \caption{LUT Encoding}
    \label{fig:lut}
\end{figure}
\noindent
The calculations below illustrate how the indices of the four adjacent 2D LUT-grid entries are determined for a given level $i$. For brevity, we have omitted both the superscript $(r)$ (where $r \in \mathcal{B}$ specifies the particular ray from the batch $\mathcal{B}$) and the subscript $i$ (denoting, before notation changes, the index of the Gaussian primitive hit by ray $r$), for which the $(u,v)$ coordinates are computed. First, let us compute:
$$
\left( u^{\left( i \right)}, v^{\left( i \right)} \right) = \begin{pmatrix}
\frac{\frac{u}{R_{\ast}} + 1}{2} \cdot \left( N^{\left( i \right)} - 1 \right) \\
\\
\frac{\frac{v}{R_{\ast}} + 1}{2} \cdot \left( N^{\left( i \right)} - 1 \right)
\end{pmatrix}^{\mathrm{T}}
$$
Then, we can determine the coordinates of the top-left corner LUT feature map entry: $\left( u_{\text{floor}}^{\left( i \right)}, v_{\text{floor}}^{\left( i \right)} \right) \in \left( \mathbb{N} \cup \left \lbrace 0 \right \rbrace \right) \times \left( \mathbb{N} \cup \left \lbrace 0 \right \rbrace \right)$ used in the bilinear interpolation, as well as the interpolation coefficients: $\left( u_{\text{frac}}^{\left( i \right)}, v_{\text{frac}}^{\left( i \right)} \right) \in \mathbb{R}^2$, using the following formula:
$$
\left( u_{\text{floor}}^{\left( i \right)}, v_{\text{floor}}^{\left( i \right)} \right) = \begin{pmatrix}
\min \left \lbrace \left \lfloor \frac{\frac{u}{R_{\ast}} + 1}{2} \cdot \left( N^{\left( i \right)} - 1 \right) \right \rfloor, N^{\left( i \right)} - 2 \right \rbrace \\
\\
\min \left \lbrace \left \lfloor \frac{\frac{v}{R_{\ast}} + 1}{2} \cdot \left( N^{\left( i \right)} - 1 \right) \right \rfloor, N^{\left( i \right)} - 2 \right \rbrace
\end{pmatrix}^{\mathrm{T}}
$$
$$
\left( u_{\text{frac}}^{\left( i \right)}, v_{\text{frac}}^{\left( i \right)} \right) = \begin{pmatrix}
u^{\left( i \right)} - u_{\text{floor}}^{\left( i \right)} \\
\\
v^{\left( i \right)} - v_{\text{floor}}^{\left( i \right)}
\end{pmatrix}^{\mathrm{T}}
$$
where $i \in \left \lbrace 1, \ldots, 4 \right \rbrace$ denotes the LUT resolution level, $N^{\left( i \right)}$ defines the LUT feature map resolution at the $i$-th level and $R_{\ast}$ is the maximum radius (over all $\kappa$ values) of the Gaussian footprint in the reference frame with the identity Gaussian's covariance matrix acting as a normalization constant that ensures the coordinates of each of the four adjacent 2D LUT-grid entries remain within their bounds.

Note that although using a single normalization constant $R_{\ast}$ for all possible values of the parameter $\kappa$ may leave a significant number of LUT-grid entries unutilized, introducing individual normalization for each $\kappa$ value could lead to training instability (as varying $\kappa$ values would cause the same coordinates in the local primitives' reference frame to be mapped to different LUT entries across training iterations) and would significantly complicate the gradient formulation.

In order to derive the formula for the normalization constant $R_{\ast}$, let us take $\kappa \in \left(1, \infty \right)$. Then the Gaussian footprint radius $R\left( \kappa \right)$ for the given $\kappa$ and the fixed configurable $\alpha_{\text{min}}$ parameter is given by the following expression:
$$
R \left( \kappa \right) = \left( -2\kappa \cdot \ln \left( \alpha_{\text{min}} \right) \right)^{\frac{1}{2\kappa}}
$$
Now it can be shown using basic calculus that:
$$
R_{\ast} = \sup_{\kappa \in \left(1, \infty \right)} R\left( \kappa \right) = \begin{cases}
\sqrt{-2 \ln \left( \alpha_{\text{min}} \right)} & \; , \alpha_{\text{min}} < e^{-\frac{e}{2}} \\
\\
\left( \alpha_{\text{min}} \right)^{-\frac{1}{e}} & \; , \text{otherwise} \\
\end{cases}
$$
Let us define:
$$
\text{offsets} = \begin{bmatrix}
0 \\
0 + \left( 16 \times 16 \times 2 \right) \\
0 + \left( 16 \times 16 \times 2 \right) + \left( 25 \times 25 \times 2 \right) \\
0 + \left( 16 \times 16 \times 2 \right) + \left( 25 \times 25 \times 2 \right) + \left( 40 \times 40 \times 2 \right)
\end{bmatrix}^{\mathrm{T}} = \begin{bmatrix}
0, & 512, & 1762, & 4962
\end{bmatrix}
$$
Then, for a given LUT resolution level $i \in \left \lbrace 1, \ldots, 4 \right \rbrace$, one can compute the base memory addresses: $\text{ind}_{00}^{\left( i \right)}$, $\text{ind}_{01}^{\left( i \right)}$, $\text{ind}_{10}^{\left( i \right)}$ and $\text{ind}_{11}^{\left( i \right)}$ for the feature vectors: $\mathrm{f}_{00}^{\left( i \right)}, \mathrm{f}_{01}^{\left( i \right)}, \mathrm{f}_{10}^{\left( i \right)}, \mathrm{f}_{11}^{\left( i \right)} \in \mathbb{R}^2$ using the following formula:
$$
\left(
\text{ind}_{00}^{\left( i \right)}, \text{ind}_{01}^{\left( i \right)}, \text{ind}_{10}^{\left( i \right)}, \text{ind}_{11}^{\left( i \right)} \right) = \begin{pmatrix}
\begin{array}{l}
\text{offsets}_i + \left( \left( \left( v_{\text{floor}}^{\left( i \right)} \cdot N^{\left( i \right)} \right) + u_{\text{floor}}^{\left( i \right)} \right) \cdot 2 \right) \\
\text{offsets}_i + \left( \left( \left( v_{\text{floor}}^{\left( i \right)} \cdot N^{\left( i \right)} \right) + \left( u_{\text{floor}}^{\left( i \right)} + 1 \right) \right) \cdot 2 \right) \\
\text{offsets}_i + \left( \left( \left( \left( v_{\text{floor}}^{\left( i \right)} + 1 \right) \cdot N^{\left( i \right)} \right) + u_{\text{floor}}^{\left( i \right)} \right) \cdot 2 \right) \\
\text{offsets}_i + \left( \left( \left( \left( v_{\text{floor}}^{\left( i \right)} + 1 \right) \cdot N^{\left( i \right)} \right) + \left( u_{\text{floor}}^{\left( i \right)} + 1 \right) \right) \cdot 2 \right)
\end{array}
\end{pmatrix}^{\mathrm{T}}
$$
Consequently, as $\mathcal{F} = \left \lbrace \mathrm{f}_1, \ldots, \mathrm{f}_{6577} \right \rbrace$ we have:
{ \footnotesize
\begin{flalign*}
\left( \mathrm{f}_{00}^{\left( i \right)}, \mathrm{f}_{01}^{\left( i \right)}, \mathrm{f}_{10}^{\left( i \right)}, \mathrm{f}_{11}^{\left( i \right)} \right) &=
\left( \mathrm{f}_{\frac{\text{ind}_{00}^{\left( i \right)}}{2} + 1}, \mathrm{f}_{\frac{\text{ind}_{01}^{\left( i \right)}}{2} + 1}, \mathrm{f}_{\frac{\text{ind}_{10}^{\left( i \right)}}{2} + 1}, \mathrm{f}_{\frac{\text{ind}_{11}^{\left( i \right)}}{2} + 1} \right) =&\\
&= \left(
\begin{pmatrix}
\text{LUT} \left[ \text{ind}_{00}^{\left( i \right)} \right] \\
\text{LUT} \left[ \text{ind}_{00}^{\left( i \right)} + 1 \right] 
\end{pmatrix}^{\mathrm{T}},
\begin{pmatrix}
\text{LUT} \left[ \text{ind}_{01}^{\left( i \right)} \right] \\
\text{LUT} \left[ \text{ind}_{01}^{\left( i \right)} + 1 \right] 
\end{pmatrix}^{\mathrm{T}},
\begin{pmatrix}
\text{LUT} \left[ \text{ind}_{10}^{\left( i \right)} \right] \\
\text{LUT} \left[ \text{ind}_{10}^{\left( i \right)} + 1 \right] 
\end{pmatrix}^{\mathrm{T}},
\begin{pmatrix}
\text{LUT} \left[ \text{ind}_{11}^{\left( i \right)} \right] \\
\text{LUT} \left[ \text{ind}_{11}^{\left( i \right)} + 1 \right] 
\end{pmatrix}^{\mathrm{T}}
\right)
\end{flalign*}
}
The feature vectors $\mathrm{f}_{00}^{\left( i \right)}$, $\mathrm{f}_{01}^{\left( i \right)}$, $\mathrm{f}_{10}^{\left( i \right)}$, and $\mathrm{f}_{11}^{\left( i \right)}$ are then bilinearly interpolated analogously to the approach used in Instant-NGP:
$$
\hat{\mathrm{f}}^{\left( i \right)} =
\begin{array}{l}
\left( 1 - u_{\text{frac}}^{\left( i \right)} \right) \cdot \left( 1 - v_{\text{frac}}^{\left( i \right)} \right) \cdot \mathrm{f}_{00}^{\left( i \right)} + \\
+ \; u_{\text{frac}}^{\left( i \right)} \cdot \left( 1 - v_{\text{frac}}^{\left( i \right)} \right) \cdot \mathrm{f}_{01}^{\left( i \right)} + \\
+ \; \left( 1 - u_{\text{frac}}^{\left( i \right)} \right) \cdot v_{\text{frac}}^{\left( i \right)} \cdot \mathrm{f}_{10}^{\left( i \right)} + \\
+ \; u_{\text{frac}}^{\left( i \right)} \cdot v_{\text{frac}}^{\left( i \right)} \cdot \mathrm{f}_{11}^{\left( i \right)}
\end{array}
$$
Finally, we obtain:
$$
E_{\text{LUT}} \left( u, v \right) = {
\setlength{\arraycolsep}{2.5pt}
\begin{bmatrix}
\hat{\mathrm{f}}^{\left( 0 \right)}, & \hat{\mathrm{f}}^{\left( 1 \right)}, & \hat{\mathrm{f}}^{\left( 2 \right)}, & \hat{\mathrm{f}}^{\left( 3 \right)}
\end{bmatrix}
}
$$

\subsection{Mesh representation of 2D Gaussian primitive}
\label{mesh_representation}
In the spirit of the technique proposed in~\cite{3dgrt}, we use mesh proxies, defined as triangulated stretched regular polygons, as spatial containers to encapsulate our 2D Gaussian primitives. This allows us to leverage efficient built-in hardware-accelerated ray-triangle intersection tests, which, unlike custom \textsc{intersection} shaders, are computed on dedicated RTX cores.
\begin{figure}[H]
\centering
\includegraphics[width=0.5\textwidth]{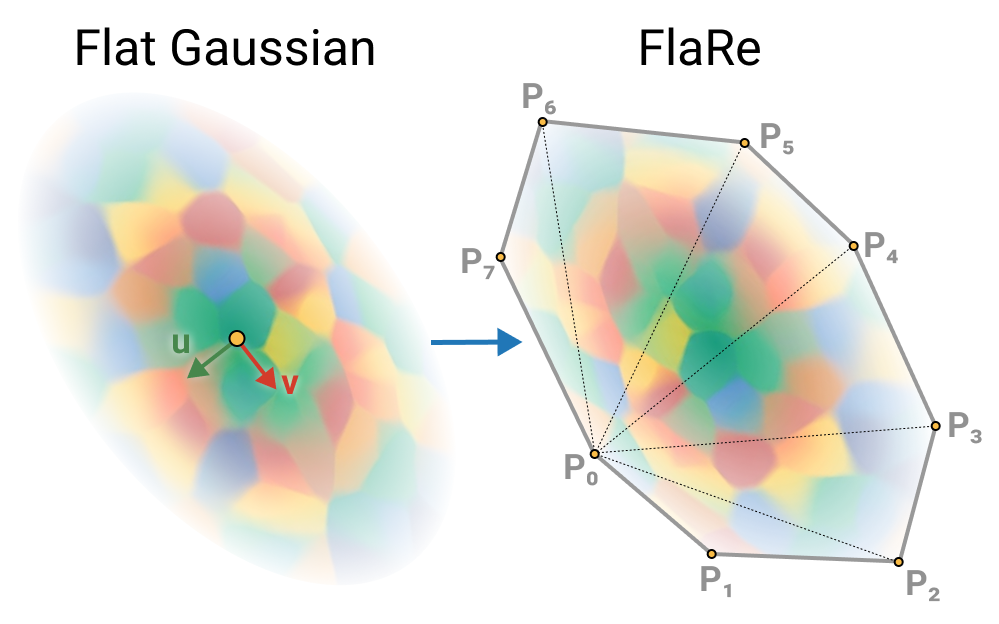}
\caption{Computing the mesh representation of the Gaussian primitive}
\end{figure}
\noindent
To minimize memory overhead, we utilize geometry instancing, where each proxy is represented as an affine transformation of a canonical mesh whose vertices are defined by the following formula:
$$
\mathrm{P}_i = {
\begin{bmatrix}
\cos \left( \frac{2\pi}{k} \cdot i \right) \\
\sin \left( \frac{2\pi}{k} \cdot i \right) \\
0
\end{bmatrix}
}
$$
where $k$ denotes the number of sides of the regular polygon. The per-Gaussian transformation matrix is given by:
$$
\mathrm{M} = \begin{bmatrix}
\mathrm{R} \cdot \mathrm{S} & \upmu^{\mathrm{T}} \\
\mathrm{0} & 1
\end{bmatrix}
$$
where:
$$
\mathrm{R} = {\begin{bmatrix}
1 - cc - dd & bc - ad & bd + ac \\
bc + ad & 1 - bb - dd & cd - ab \\
bd - ac & cd + ab & 1 - bb - cc
\end{bmatrix}
}
$$
for:
$$
s = \frac{2}{\mathrm{q}_1^2 + \mathrm{q}_2^2 + \mathrm{q}_3^2 + \mathrm{q}_4^2}
$$
$$
\begin{matrix}
bs = \mathrm{q}_2 \cdot s, & cs = \mathrm{q}_3 \cdot s, & ds = \mathrm{q}_4 \cdot s \\
ab = \mathrm{q}_1 \cdot bs, & ac = \mathrm{q}_1 \cdot cs, & ad = \mathrm{q}_1 \cdot ds \\
bb = \mathrm{q}_2 \cdot bs, & bc = \mathrm{q}_2 \cdot cs, & bd = \mathrm{q}_2 \cdot ds \\
cc = \mathrm{q}_3 \cdot cs, & cd = \mathrm{q}_3 \cdot ds, & dd = \mathrm{q}_4 \cdot ds
\end{matrix}
$$
and:
$$
\mathrm{S} = \operatorname{diag} \left( \mathrm{s}_1 \cdot {\Delta s}, \mathrm{s}_2 \cdot {\Delta s} \right)
$$
with the factor $\Delta s$ defined as follows, accounting for adaptive clamping to the target opacity $\alpha_{\text{min}}$ at the border of the Gaussian primitive:
$$
\Delta s = \begin{cases}
{ \left( -2{\kappa} \cdot \ln \left( \frac{ \alpha_{\text{min}} }{ \alpha_{\text{const}} } \right) \right) }^{\frac{1}{2{\kappa}}} & , \; \alpha_{\text{min}} \le \alpha_{\text{const}} \\
0 & , \; \text{otherwise}
\end{cases}
$$

\subsection{Mesh modification}
Given points $\mathrm{Q}_i = \mathrm{M} \cdot \mathrm{P}_i$ derived from an arbitrary transformation of the original Gaussian proxy polygon vertices $\mathrm{P}'_i$ in the world space, we will show how to adapt the 2D Gaussian parameters to the transformed scene without retraining, by directly computing the updated parameters from the base model's state and the deformed geometry $\mathrm{Q}_i$. We present a robust strategy that accounts for cases where the intrinsic structure of the proxy polygon is lost — for instance, due to non-rigid deformations that violate local coplanarity or induce orientation shifts. After computing the local centroid:
$$
\mathrm{O}' = \frac{1}{k} \sum\limits_{i=1}^k \mathrm{Q}_i
$$
of the $k$ vertices defining the transformed proxy polygon, we center the point set by translating the centroid to the origin, yielding normalized points:
$$
\mathrm{Q}'_i = \mathrm{Q}_i - \mathrm{O}'
$$
We then estimate the covariance matrix:
$$
\Sigma = \frac{1}{k} \sum\limits_{i=1}^k \mathrm{Q}'_i {\mathrm{Q}'_i}^{\mathrm{T}}
$$
of these centered points and perform SVD to extract the principal axes. This procedure identifies the best-fitting plane and the directions of maximum variance, which define the local coordinate frame of the transformed Gaussian. Since the vertices of the canonical proxy polygon are coplanar, the covariance matrix $\Sigma$ is singular or ill-conditioned. Consequently, the eigenvector corresponding to the smallest eigenvalue—nominally the surface normal—is highly sensitive to numerical instability. To ensure robustness, we compute the surface normal:
$$
\mathrm{\vec{N}} = \mathrm{\vec{U}} \times \mathrm{\vec{V}}
$$
via the cross product of the two principal eigenvectors $\mathrm{\vec{U}}$ and $\mathrm{\vec{V}}$, which span the tangent plane. This approach provides a numerically stable normal estimate, effectively mitigating the noise associated with the smallest singular vector in degenerate cases. With the basis vectors: $\mathrm{\vec{U}}, \mathrm{\vec{V}}$ spanning the tangent plane and the robustly computed surface normal $\mathrm{\vec{N}}$, we construct the rotation matrix:
$$
\mathrm{R} = \left[ \begin{array}{c|c|c} 
\mathrm{\vec{U}} & \mathrm{\vec{V}} & \mathrm{\vec{N}}
\end{array} \right]
$$
To derive the corresponding rotation quaternion $\mathrm{q}_i$, we employ the algorithm proposed by Bar-Itzhack et al., which ensures numerical stability even for near-singular configurations. To determine the Gaussian scale parameters, we project the normalized vertices $\mathrm{Q}'_i$ onto the local tangent plane defined by the basis vectors $\mathrm{\vec{U}}$ and $\mathrm{\vec{V}}$ obtaining:
$$
\hat{\mathrm{Q}'_i} = \operatorname{proj} \left( \mathrm{Q}'_i, \mathrm{\vec{N}} \right) = \mathrm{Q}'_i - \mathrm{\vec{N}}{ \left \langle \mathrm{\vec{N}}, \mathrm{Q}'_i \right \rangle }
$$
and compute the axis-aligned bounding box (AABB) of these projections in the local frame. The scale parameters along each axis are set proportional to the respective AABB dimensions, adjusted by a factor $\Delta s$ that incorporates adaptive clamping as described in the Subsection \textbf{Mesh representation of 2D Gaussian primitive}.
$$
s_1 = \log \left( \frac{1}{\Delta s} \cdot \max\limits_{i=1}^k \left| \left \langle \mathrm{\vec{U}}, \hat{\mathrm{Q}'_i} \right \rangle \right| \right)
$$
$$
s_2 = \log \left( \frac{1}{\Delta s} \cdot \max\limits_{i=1}^k \left| \left \langle \mathrm{\vec{V}}, \hat{\mathrm{Q}'_i} \right \rangle \right| \right)
$$
Finally, we obtain the Gaussian means $\upmu$ using the formula:
$$
\upmu = \mathrm{O}'
$$
While the aforementioned approach suffices for radially symmetric primitives, it becomes suboptimal for our \our{} model, where color and opacity are functions of $(u,v)$ coordinates. Specifically, when the local transformation involves shearing, the original orthogonal texture axes are mapped to non-orthogonal vectors. Consequently, the vectors $\mathrm{\vec{U}}$ and $\mathrm{\vec{V}}$ obtained via SVD, which serve as the axes of the primitive's local reference frame and define the transformed neural texture coordinates $(u,v)$, fail to coincide with the originally transformed and subsequently projected vectors $\mathrm{\vec{U}}_{\text{tex}}$ and $\mathrm{\vec{V}}_{\text{tex}}$. This misalignment leads to artifacts in the fine details modeled by our per-primitive 2D neural radiance field, despite the perfect geometric alignment of the resulting Gaussian parameters themselves. To ensure consistency between the color representation and the geometric deformation, we subsequently derive a linear transformation matrix $\mathrm{A}$, which maps the texture coordinates $(u,v)$ into the deformed space before they are fed into our MLP.
\begin{figure}[H]
\centering
\includegraphics[width=0.5\textwidth]{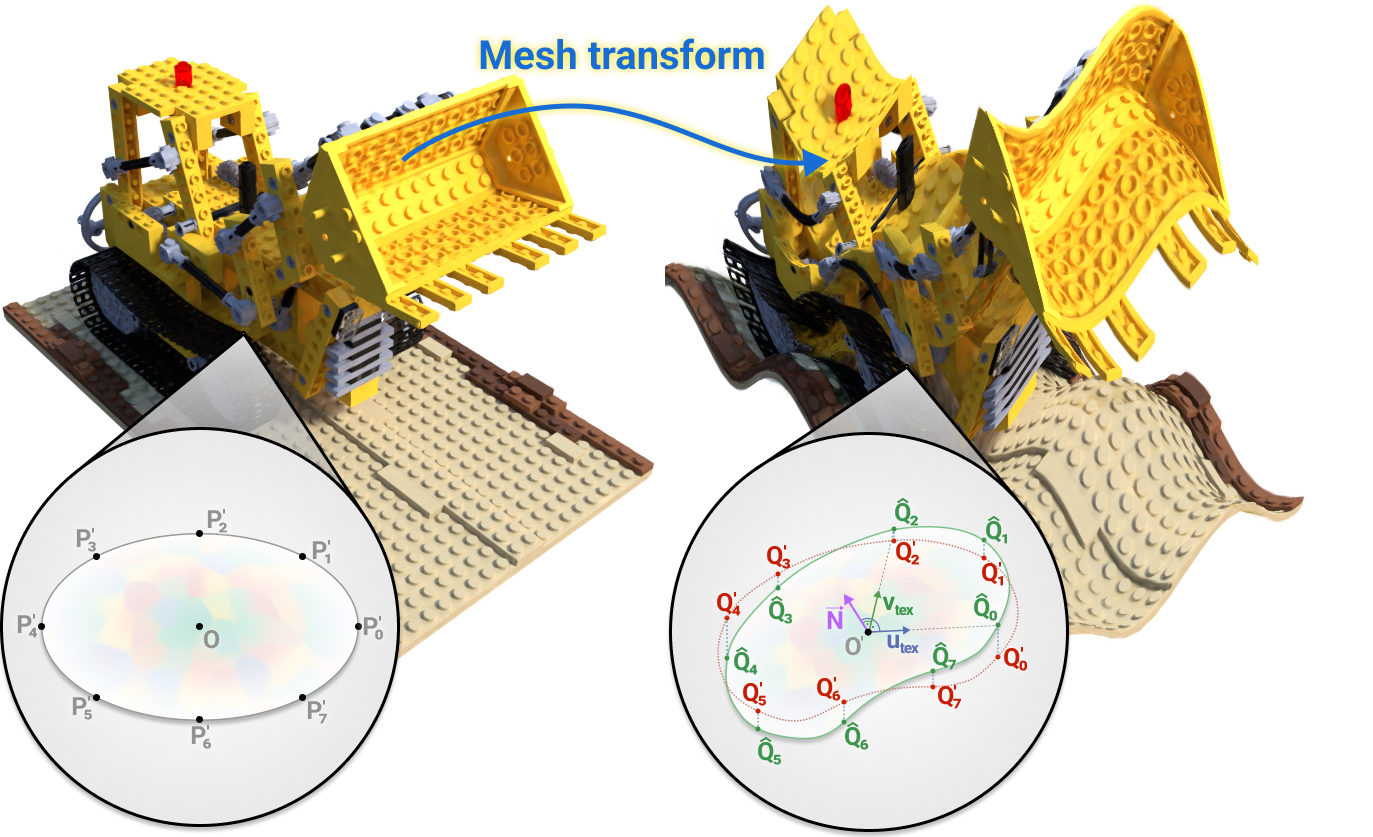}
\caption{Mesh representation transformation and subsequent retrieval of Gaussian parameters from the transformed mesh without model retraining}
\end{figure}
\noindent
If $4 \mid k$ one can simply assume that:
$$
\mathrm{\vec{U}}_{\text{tex}} = \vec{\hat{\mathrm{Q}}'}_0
$$
$$
\mathrm{\vec{V}}_{\text{tex}} = \vec{\hat{\mathrm{Q}}'}_{\frac{k}{4}}
$$
Otherwise, leveraging the fact that the vector $\vec{\mathrm{P}}'_0$ is parallel to the axis $\vec{\upvarepsilon}_1$ of the Gaussian's local reference frame, along with the known formulas (see \textbf{Mesh representation of 2D Gaussian primitive}) for
$\vec{\mathrm{P}}'_i$ in terms of the $(u,v)$ coordinates with respect to the $\vec{\upvarepsilon}_1$ and $\vec{\upvarepsilon}_2$ axes, one can express $\vec{\upvarepsilon}_2$ as a linear combination of $\vec{\upvarepsilon}_1$ and an arbitrary vector $\vec{\mathrm{P}}'_i$ for $i \ne 0$, and then compute the corresponding linear combination of $\vec{\hat{\mathrm{Q}}'}_i$ and $\mathrm{\vec{U}}_{\text{tex}} = \vec{\hat{\mathrm{Q}}'}_0$ using these same coefficients to obtain $\mathrm{\vec{V}}_{\text{tex}}$.
\newline
\newline
Let us assume, therefore, that we have determined the vectors: $\mathrm{\vec{U}}_{\text{tex}}$ and $\mathrm{\vec{V}}_{\text{tex}}$. Then, expressing the above vectors in the reference frame determined by the SVD and normalizing the coordinate frame so that the covariance matrix of the Gaussian is the identity, we obtain:
$$
\mathrm{\vec{U}}'_{\text{tex}} = \begin{bmatrix}
\frac{\left \langle \mathrm{\vec{U}}, \mathrm{\vec{U}}_{\text{tex}} \right \rangle}{s_1}, &
\frac{\left \langle \mathrm{\vec{V}}, \mathrm{\vec{U}}_{\text{tex}} \right \rangle}{s_2}
\end{bmatrix}
$$
$$
\mathrm{\vec{V}}'_{\text{tex}} = \begin{bmatrix}
\frac{\left \langle \mathrm{\vec{U}}, \mathrm{\vec{V}}_{\text{tex}} \right \rangle}{s_1}, &
\frac{\left \langle \mathrm{\vec{V}}, \mathrm{\vec{V}}_{\text{tex}} \right \rangle}{s_2}
\end{bmatrix}
$$
Then, the vectors orthogonal to them are given by:
$$
\mathrm{\vec{U}}'_{\perp,\text{tex}} = \begin{bmatrix}
-\frac{\left \langle \mathrm{\vec{V}}, \mathrm{\vec{U}}_{\text{tex}} \right \rangle}{s_2}, &
\frac{\left \langle \mathrm{\vec{U}}, \mathrm{\vec{U}}_{\text{tex}} \right \rangle}{s_1}
\end{bmatrix}
$$
$$
\mathrm{\vec{V}}'_{\perp,\text{tex}} = \begin{bmatrix}
-\frac{\left \langle \mathrm{\vec{V}}, \mathrm{\vec{V}}_{\text{tex}} \right \rangle}{s_2}, &
\frac{\left \langle \mathrm{\vec{U}}, \mathrm{\vec{V}}_{\text{tex}} \right \rangle}{s_1}
\end{bmatrix}
$$
We want to determine the scaling factors: $k_1$ and $k_2$ that will allow us to compute the target vectors: $\mathrm{\vec{U}}''_{\text{tex}}$ and $\mathrm{\vec{V}}''_{\text{tex}}$ spanning the Gaussian's tight-fitting parallelogram with sides parallel to: $\mathrm{\vec{U}}_{\text{tex}}$ and $\mathrm{\vec{V}}_{\text{tex}}$. Note that these factors satisfy the following system of two independent equations:
$$
\left \lbrace 
\begin{matrix}
\left \langle \frac{\mathrm{\vec{U}}'_{\perp,\text{tex}}}{\left \lVert \mathrm{\vec{U}}'_{\perp,\text{tex}} \right \rVert}, \mathrm{\vec{V}}'_{\text{tex}} k_2 \right \rangle = 1 \\
\left \langle \frac{\mathrm{\vec{V}}'_{\perp,\text{tex}}}{\left \lVert \mathrm{\vec{V}}'_{\perp,\text{tex}} \right \rVert}, \mathrm{\vec{U}}'_{\text{tex}} k_1 \right \rangle = 1
\end{matrix}
\right.
$$
By transforming the above system of equations, we successively obtain:
$$
\left \lbrace 
\begin{matrix}
\frac{k_2}{\left \lVert \mathrm{\vec{U}}'_{\perp,\text{tex}} \right \rVert} \left \langle \mathrm{\vec{U}}'_{\perp,\text{tex}}, \mathrm{\vec{V}}'_{\text{tex}} \right \rangle = 1 \\
\frac{k_1}{\left \lVert \mathrm{\vec{V}}'_{\perp,\text{tex}} \right \rVert} \left \langle \mathrm{\vec{V}}'_{\perp,\text{tex}}, \mathrm{\vec{U}}'_{\text{tex}} \right \rangle = 1
\end{matrix}
\right.
$$
$$
\left \lbrace 
\begin{matrix}
k_2 \frac{\left \langle \mathrm{\vec{U}}'_{\perp,\text{tex}}, \mathrm{\vec{V}}'_{\text{tex}} \right \rangle}{\left \lVert \mathrm{\vec{U}}'_{\perp,\text{tex}} \right \rVert} = 1 \\
k_1 \frac{\left \langle \mathrm{\vec{V}}'_{\perp,\text{tex}}, \mathrm{\vec{U}}'_{\text{tex}} \right \rangle}{\left \lVert \mathrm{\vec{V}}'_{\perp,\text{tex}} \right \rVert} = 1
\end{matrix}
\right.
$$
$$
\left \lbrace 
\begin{matrix}
k_1 = \frac{\left \lVert \mathrm{\vec{V}}'_{\perp,\text{tex}} \right \rVert}{\left \langle \mathrm{\vec{V}}'_{\perp,\text{tex}}, \mathrm{\vec{U}}'_{\text{tex}} \right \rangle} \\
k_2 = \frac{\left \lVert \mathrm{\vec{U}}'_{\perp,\text{tex}} \right \rVert}{\left \langle \mathrm{\vec{U}}'_{\perp,\text{tex}}, \mathrm{\vec{V}}'_{\text{tex}} \right \rangle}
\end{matrix}
\right.
$$
Hence, after transforming back to the original coordinate system, we obtain:
$$
\mathrm{\vec{U}}''_{\text{tex}} = k_1 \cdot \operatorname{diag} \cdot \left( s_1, s_2 \right) \cdot \mathrm{\vec{U}}'_{\text{tex}} = \begin{bmatrix}
k_1 \left \langle \mathrm{\vec{U}}, \mathrm{\vec{U}}_{\text{tex}} \right \rangle, &
k_1 \left \langle \mathrm{\vec{V}}, \mathrm{\vec{U}}_{\text{tex}} \right \rangle
\end{bmatrix}
$$
$$
\mathrm{\vec{V}}''_{\text{tex}} = k_2 \cdot \operatorname{diag} \cdot \left( s_1, s_2 \right) \cdot \mathrm{\vec{V}}'_{\text{tex}} = \begin{bmatrix}
k_2 \left \langle \mathrm{\vec{U}}, \mathrm{\vec{V}}_{\text{tex}} \right \rangle, &
k_2 \left \langle \mathrm{\vec{V}}, \mathrm{\vec{V}}_{\text{tex}} \right \rangle
\end{bmatrix}
$$
Now, since:
$$
(u,v)^{\mathrm{T}} = \left[ \begin{array}{c|c} 
\mathrm{\vec{U}}''_{\text{tex}} & \mathrm{\vec{V}}''_{\text{tex}} 
\end{array} \right] \cdot (u',v')^{\mathrm{T}}
$$
we finally obtain:
$$
(u',v')^{\mathrm{T}} = \mathrm{A} \cdot (u,v)^{\mathrm{T}}
$$
for:
$$
\mathrm{A} = \left[ \begin{array}{c|c} 
\mathrm{\vec{U}}''_{\text{tex}} & \mathrm{\vec{V}}''_{\text{tex}} 
\end{array} \right]^{-1}
$$
Finally, let us note that since the vectors: $\mathrm{\vec{U}}''_{\text{tex}}$ and $\mathrm{\vec{V}}''_{\text{tex}}$ span the Gaussian's tight-fitting parallelogram, and the Gaussian footprint is contained within this parallelogram, the $(u',v')$ coordinates will fall within the valid range for the encoding function implementing the LUT-encoding.

\subsection{Color aggregation along the ray}
Following the approach described in~\cite{kopanas21,kopanas22}, the FlaRe method maps 2D planar Gaussian primitives onto the image plane by blending colors for individual pixels. This is achieved by collecting samples at the intersection points of a ray, traced from the camera through the pixel center, with the Gaussian primitives' footprints, which are defined by the boundary of their proxy meshes. The final unclamped pixel color $\mathrm{\tilde{I}}^{\left( r \right)}$ for the ray $r \in \mathcal{B}$, where by $r \in \mathcal{B}$ we denote the batch of rays, is given by the following formula:
$$
\mathrm{\tilde{I}}^{\left( r \right)} = \sum_{i=1}^{\left \lvert \mathcal{G}^{\left( r \right)} \right \rvert} \mathrm{c}^{\left( r \right)} \alpha^{\left( r \right)}_i T^{\left( r \right)}_i = \sum_{i=1}^{\left \lvert \mathcal{G}^{\left( r \right)} \right \rvert} \mathrm{c}^{\left( r \right)}_{\text{const},i} \mathrm{c}^{\left( r \right)}_{\text{MLP},i} \alpha^{\left( r \right)}_i \prod_{j=1}^{i-1} \left( 1 - \alpha^{\left( r \right)}_j \right)
$$
where:
$$
\alpha^{\left( r \right)}_i = \alpha^{\left( r \right)}_{\text{const},i} \cdot \alpha^{\left( r \right)}_{\text{MLP},i} \cdot \exp \left( -\frac{ { \left( \left( u^{ \left( r \right)}_i \right)^2+\left( v^{\left( r \right)}_i \right)^2 \right) }^{\kappa^{\left( r \right)}_i} }{2\kappa^{\left( r \right)}_i} \right)
$$

\subsection{Loss function}
In our differentiable rendering framework, we define a composite loss function to guide the optimization of the scene representation. The total loss $\mathcal{L}$ is formulated as the sum of a convex combination of $\mathcal{L}_{\text{RGB},\text{FlaRe}}$ and $\mathcal{L}_{\text{RGB},\text{base}}$—representing the $\mathcal{L}_2$ photometric losses in our target FlaRe model and the simplified base model, respectively—and the common $\mathcal{L}_s$ term. The base model relies solely on the view-independent constant color parameters expressed in the RGB space, $\mathrm{c}_{\text{const},i}$, and opacity parameters ${\bar{\alpha}}_{\text{const},i}$. To prevent overfitting - which occurs in certain scenes when the FlaRe model tends to capture fine details via the 2D neural radiance field rather than the underlying geometry — we introduce a warmup phase. During this phase, the loss function smoothly transitions between the base and target FlaRe models. This warmup allows for the preliminary optimization of the geometry before the 2D neural radiance field is optimized in the target model. The loss function $\mathcal{L}$ itself is expressed as:
$$
\mathcal{L} = \lambda_1 \cdot \left( \left( 1 - \lambda \right) \cdot \mathcal{L}_{\text{RGB},\text{base}} + \lambda \cdot \mathcal{L}_{\text{RGB},\text{FlaRe}} \right) + \lambda_2 \cdot {\mathcal{L}_{s}}
$$
where:
$$
\lambda = \operatorname{clamp} \left( \frac{\text{iter}-\text{warmup\_start\_iter}}{{\text{warmup\_end\_iter}-\text{warmup\_start\_iter}}}, 0, 1 \right)
$$
and $\lambda_1, \lambda_2 \ge 0$ are non-negative hyperparameters that balance the contribution of each objective term. Both $\mathcal{L}_{\text{RGB},\text{FlaRe}}$ and $\mathcal{L}_{\text{RGB},\text{base}}$ are given by the same formula:
$$
\mathcal{L}_{\text{RGB},\text{FlaRe}} = \frac{1}{3 \left| \mathcal{B} \right|} \sum\limits_{r \in \mathcal{B}} \sum\limits_{i=1}^3 { \left( \mathrm{I}_i^{\left( r \right)} - \hat{\mathrm{I}}_i^{\left( r \right)} \right) }^2
$$
$$
\mathcal{L}_{\text{RGB},\text{base}} = \frac{1}{3 \left| \mathcal{B} \right|} \sum\limits_{r \in \mathcal{B}} \sum\limits_{i=1}^3 { \left( \mathrm{I}_i^{\left( r \right)} - \hat{\mathrm{I}}_i^{\left( r \right)} \right) }^2
$$
where $\mathrm{I}_i^{\left( r \right)}$ denotes the $i$-th color channel value of the rendered pixel corresponding to ray $r$ clamped to the interval $\left[ 0, 1 \right]$ (the value before clamping is denoted by $\tilde{\mathrm{I}}_i^{\left( r \right)}$), and $\hat{\mathrm{I}}_i^{ \left( r \right) }$ is the ground-truth value of the corresponding pixel in the reference image. Lastly, the common regularization component $\mathcal{L}_s$ defined as:
$$
\mathcal{L}_s = \frac{1}{N} \sum_{i=1}^N \left \lVert \mathrm{s}_i \right \rVert
$$
serves as a regularization term on the Gaussians' scale parameters $\mathrm{\tilde{s}}_i$, preventing them from growing excessively. This constraint promotes compactness, which effectively improves both training and inference efficiency.

We intentionally omitted the depth distortion and normal consistency regularization component from 2DGS~\cite{2dgs}. In the case of the former, we obtained worse results in terms of PSNR on test poses in the scene reconstruction task, whereas the latter is inherently incompatible with our batch training paradigm. Specifically, batch sampling does not guarantee that the rays corresponding to the direct spatial neighbors of a ray for the pixel $\left( i,j \right)$ in a training image will be present simultaneously, which are required for estimating surface normals.

\subsection{Implementation details}
Since the \textsc{OptixCoopVec} abstraction introduced in OptiX 9.0 is architecture-agnostic and does not allow the developer to manually manipulate the matrix layout within shader code — with the matrix remaining static throughout the entire \textsc{optixLaunch} execution once packed via the host-side \textsc{optixCoopVecMatrixConvert} function — and because matrix multiplication currently supports only half-precision FP16 accumulation, which (due to accumulated numerical errors and resulting rendering artifacts) precludes its application even to small networks with $64$–$128$ neurons and $2$ hidden layers, we adopted a hybrid approach in our method. In an interleaved fashion, we perform \textsc{optixLaunch} calls — used solely to determine the indices of consecutive planar Gaussian primitives hit by the ray within the \textsc{HIT\_BUFFER\_SIZE} limit and their corresponding ray $t$-parameters — along with a custom CUDA kernel designed to prepare the input matrix for our auto-decoder MLP by manually manipulating the data layout in the registers and evaluating it through the network for both the forward and backward passes. Following the 3DGRT~\cite{3dgrt} approach, we set the hit buffer size \textsc{HIT\_BUFFER\_SIZE} to \textsc{16} and execute these two operations iteratively within a loop, followed by stream compaction to remove inactive rays until all rays are fully processed.

\section{Derivation of Gradient Formulas}
\subsection{Computing upstream gradients}
\noindent
\newline
\paragraph{Computing $\dv{}{\mathrm{c}^{\left( r \right)}_{\text{const},i}}
\mathcal{L}_{\text{RGB},\text{FlaRe}}$}
\begin{flalign*}
\dv{}{\mathrm{c}^{\left( r \right)}_{\text{const},i}}
\mathcal{L}_{\text{RGB},\text{FlaRe}} &=
\dv{}{\mathrm{c}^{\left( r \right)}_{\text{const},i}} \left( \frac{1}{3 \left| \mathcal{B} \right|} \sum\limits_{r \in \mathcal{B}} \sum\limits_{j=1}^3 { \left( \mathrm{I}_j^{\left( r \right)} - \hat{\mathrm{I}}_j^{\left( r \right)} \right) }^2 \right) =&\\
&= \frac{1}{3 \left| \mathcal{B} \right|} \cdot \dv{}{\mathrm{c}^{\left( r \right)}_{\text{const},i}} \left( \sum\limits_{r \in \mathcal{B}} \sum\limits_{j=1}^3 { \left( \mathrm{I}_j^{\left( r \right)} - \hat{\mathrm{I}}_j^{\left( r \right)} \right) }^2 \right) =&\\
&= \frac{1}{3 \left| \mathcal{B} \right|} \sum\limits_{r \in \mathcal{B}} \dv{}{\mathrm{c}^{\left( r \right)}_{\text{const},i}} \left( \sum\limits_{j=1}^3 { \left( \mathrm{I}_j^{\left( r \right)} - \hat{\mathrm{I}}_j^{\left( r \right)} \right) }^2 \right) =&\\
&= \frac{1}{3 \left| \mathcal{B} \right|} \sum\limits_{r \in \mathcal{B}} \sum\limits_{j=1}^3 \dv{}{\mathrm{c}^{\left( r \right)}_{\text{const},i}} \left( { \left( \mathrm{I}_j^{\left( r \right)} - \hat{\mathrm{I}}_j^{\left( r \right)} \right) }^2 \right) =&\\
&= \frac{1}{3 \left| \mathcal{B} \right|} \sum\limits_{r \in \mathcal{B}} \sum\limits_{j=1}^3 2 \left( \mathrm{I}_j^{\left( r \right)} - \hat{\mathrm{I}}_j^{\left( r \right)} \right) \dv{}{\mathrm{c}^{\left( r \right)}_{\text{const},i}} \mathrm{I}_j^{\left( r \right)} =&\\
&= \frac{2}{3 \left| \mathcal{B} \right|} \sum\limits_{r \in \mathcal{B}} \sum\limits_{j=1}^3 \left( \mathrm{I}_j^{\left( r \right)} - \hat{\mathrm{I}}_j^{\left( r \right)} \right) \dv{}{\mathrm{c}^{\left( r \right)}_{\text{const},i}} \mathrm{I}_j^{\left( r \right)} =&\\
&= \frac{2}{3 \left| \mathcal{B} \right|} \sum\limits_{r \in \mathcal{B}} \sum\limits_{j=1}^3 \left( \mathrm{I}_j^{\left( r \right)} - \hat{\mathrm{I}}_j^{\left( r \right)} \right) \dv{}{\mathrm{c}^{\left( r \right)}_{\text{const},i}} \operatorname{clamp} \left( \mathrm{\tilde{I}}_j^{\left( r \right)}, 0, 1 \right) =&\\
&= \frac{2}{3 \left| \mathcal{B} \right|} \sum\limits_{r \in \mathcal{B}} \sum\limits_{j=1}^3 \left( \mathrm{I}_j^{\left( r \right)} - \hat{\mathrm{I}}_j^{\left( r \right)} \right) \cdot \mathds{1}_{\mathrm{\tilde{I}}_j^{\left( r \right)} \in \left[ 0, 1 \right]} \cdot \dv{}{\mathrm{c}^{\left( r \right)}_{\text{const},i}} \mathrm{\tilde{I}}_j^{\left( r \right)} =&\\
&= \frac{2}{3 \left| \mathcal{B} \right|} \sum\limits_{r \in \mathcal{B}} \sum\limits_{j=1}^3 \left( \mathrm{I}_j^{\left( r \right)} - \hat{\mathrm{I}}_j^{\left( r \right)} \right) \cdot \mathds{1}_{\mathrm{\tilde{I}}_j^{\left( r \right)} \in \left[ 0, 1 \right]} \cdot \dv{}{\mathrm{c}^{\left( r \right)}_{\text{const},i}} \left( \sum_{k=1}^{\left \lvert \mathcal{G}^{\left( r \right)} \right \rvert} \mathrm{c}^{\left( r \right)}_{\text{const},k,j} \mathrm{c}^{\left( r \right)}_{\text{MLP},k,j} \alpha^{\left( r \right)}_k T^{\left( r \right)}_k \right) =&\\
&= \frac{2}{3 \left| \mathcal{B} \right|} \sum\limits_{r \in \mathcal{B}} \sum\limits_{j=1}^3 \left( \mathrm{I}_j^{\left( r \right)} - \hat{\mathrm{I}}_j^{\left( r \right)} \right) \cdot \mathds{1}_{\mathrm{\tilde{I}}_j^{\left( r \right)} \in \left[ 0, 1 \right]} \cdot \left( \sum_{k=1}^{\left \lvert \mathcal{G}^{\left( r \right)} \right \rvert} \dv{}{\mathrm{c}^{\left( r \right)}_{\text{const},i}} \left( \mathrm{c}^{\left( r \right)}_{\text{const},k,j} \mathrm{c}^{\left( r \right)}_{\text{MLP},k,j} \alpha^{\left( r \right)}_k T^{\left( r \right)}_k \right) \right) =&\\
&= \frac{2}{3 \left| \mathcal{B} \right|} \sum\limits_{r \in \mathcal{B}} \sum\limits_{j=1}^3 \left( \mathrm{I}_j^{\left( r \right)} - \hat{\mathrm{I}}_j^{\left( r \right)} \right) \cdot \mathds{1}_{\mathrm{\tilde{I}}_j^{\left( r \right)} \in \left[ 0, 1 \right]} \cdot \left( \sum_{k=1}^{\left \lvert \mathcal{G}^{\left( r \right)} \right \rvert} \mathrm{c}^{\left( r \right)}_{\text{MLP},k,j} \alpha^{\left( r \right)}_k T^{\left( r \right)}_k \cdot \dv{\mathrm{c}^{\left( r \right)}_{\text{const},k,j}}{\mathrm{c}^{\left( r \right)}_{\text{const},i}} \right) =&\\
&= \frac{2}{3 \left| \mathcal{B} \right|} \sum\limits_{r \in \mathcal{B}} \sum\limits_{j=1}^3 \left( \mathrm{I}_j^{\left( r \right)} - \hat{\mathrm{I}}_j^{\left( r \right)} \right) \cdot \mathds{1}_{\mathrm{\tilde{I}}_j^{\left( r \right)} \in \left[ 0, 1 \right]} \cdot \left( \sum_{k=1}^{\left \lvert \mathcal{G}^{\left( r \right)} \right \rvert} \mathrm{c}^{\left( r \right)}_{\text{MLP},k,j} \alpha^{\left( r \right)}_k T^{\left( r \right)}_k \cdot \delta_{i{k}} \cdot \vec{\upvarepsilon}_j \right) =&\\
&= \frac{2}{3 \left| \mathcal{B} \right|} \sum\limits_{r \in \mathcal{B}} \sum\limits_{j=1}^3 \left( \mathrm{I}_j^{\left( r \right)} - \hat{\mathrm{I}}_j^{\left( r \right)} \right) \cdot \mathds{1}_{\mathrm{\tilde{I}}_j^{\left( r \right)} \in \left[ 0, 1 \right]} \cdot \mathrm{c}^{\left( r \right)}_{\text{MLP},i,j} \alpha^{\left( r \right)}_i T^{\left( r \right)}_i \cdot \vec{\upvarepsilon}_j =&\\
&= \frac{2}{3 \left| \mathcal{B} \right|} \sum\limits_{r \in \mathcal{B}} \left( \mathrm{I}^{\left( r \right)} - \hat{\mathrm{I}}^{\left( r \right)} \right) \cdot \mathds{1}_{\mathrm{\tilde{I}}^{\left( r \right)} \in \left[ 0, 1 \right]} \cdot \mathrm{c}^{\left( r \right)}_{\text{MLP},i} \alpha^{\left( r \right)}_i T^{\left( r \right)}_i
\end{flalign*}

\paragraph{Computing $\dv{}{\tilde{\alpha}^{\left( r \right)}_{\text{const},i}}
\mathcal{L}_{\text{RGB},\text{FlaRe}}$}
\begin{flalign*}
\dv{}{\tilde{\alpha}^{\left( r \right)}_{\text{const},i}}
\mathcal{L}_{\text{RGB},\text{FlaRe}} &=
\dv{}{\tilde{\alpha}^{\left( r \right)}_{\text{const},i}} \left( \frac{1}{3 \left| \mathcal{B} \right|} \sum\limits_{r \in \mathcal{B}} \sum\limits_{j=1}^3 { \left( \mathrm{I}_j^{\left( r \right)} - \hat{\mathrm{I}}_j^{\left( r \right)} \right) }^2 \right) =&\\
&= \frac{1}{3 \left| \mathcal{B} \right|} \cdot \dv{}{\tilde{\alpha}^{\left( r \right)}_{\text{const},i}} \left( \sum\limits_{r \in \mathcal{B}} \sum\limits_{j=1}^3 { \left( \mathrm{I}_j^{\left( r \right)} - \hat{\mathrm{I}}_j^{\left( r \right)} \right) }^2 \right) =&\\
&= \frac{1}{3 \left| \mathcal{B} \right|} \sum\limits_{r \in \mathcal{B}} \dv{}{\tilde{\alpha}^{\left( r \right)}_{\text{const},i}} \left( \sum\limits_{j=1}^3 { \left( \mathrm{I}_j^{\left( r \right)} - \hat{\mathrm{I}}_j^{\left( r \right)} \right) }^2 \right) =&\\
&= \frac{1}{3 \left| \mathcal{B} \right|} \sum\limits_{r \in \mathcal{B}} \sum\limits_{j=1}^3 \dv{}{\tilde{\alpha}^{\left( r \right)}_{\text{const},i}} \left( { \left( \mathrm{I}_j^{\left( r \right)} - \hat{\mathrm{I}}_j^{\left( r \right)} \right) }^2 \right) =&\\
&= \frac{1}{3 \left| \mathcal{B} \right|} \sum\limits_{r \in \mathcal{B}} \sum\limits_{j=1}^3 2 \left( \mathrm{I}_j^{\left( r \right)} - \hat{\mathrm{I}}_j^{\left( r \right)} \right) \dv{}{\tilde{\alpha}^{\left( r \right)}_{\text{const},i}} \mathrm{I}_j^{\left( r \right)} =&\\
&= \frac{2}{3 \left| \mathcal{B} \right|} \sum\limits_{r \in \mathcal{B}} \sum\limits_{j=1}^3 \left( \mathrm{I}_j^{\left( r \right)} - \hat{\mathrm{I}}_j^{\left( r \right)} \right) \dv{}{\tilde{\alpha}^{\left( r \right)}_{\text{const},i}} \mathrm{I}_j^{\left( r \right)} =&\\
&= \frac{2}{3 \left| \mathcal{B} \right|} \sum\limits_{r \in \mathcal{B}} \sum\limits_{j=1}^3 \left( \mathrm{I}_j^{\left( r \right)} - \hat{\mathrm{I}}_j^{\left( r \right)} \right) \dv{}{\tilde{\alpha}^{\left( r \right)}_{\text{const},i}} \operatorname{clamp} \left( \mathrm{\tilde{I}}_j^{\left( r \right)}, 0, 1 \right) =&\\
&= \frac{2}{3 \left| \mathcal{B} \right|} \sum\limits_{r \in \mathcal{B}} \sum\limits_{j=1}^3 \left( \mathrm{I}_j^{\left( r \right)} - \hat{\mathrm{I}}_j^{\left( r \right)} \right) \cdot \mathds{1}_{\mathrm{\tilde{I}}_j^{\left( r \right)} \in \left[ 0, 1 \right]} \cdot \dv{\mathrm{\tilde{I}}_j^{\left( r \right)}}{\tilde{\alpha}^{\left( r \right)}_{\text{const},i}} =&\\
&= \frac{2}{3 \left| \mathcal{B} \right|} \sum\limits_{r \in \mathcal{B}} \sum\limits_{j=1}^3 \left( \mathrm{I}_j^{\left( r \right)} - \hat{\mathrm{I}}_j^{\left( r \right)} \right) \cdot \mathds{1}_{\mathrm{\tilde{I}}_j^{\left( r \right)} \in \left[ 0, 1 \right]} \cdot \dv{\mathrm{\tilde{I}}_j^{\left( r \right)}}{\alpha^{\left( r \right)}_i} \cdot \dv{\alpha^{\left( r \right)}_i}{\tilde{\alpha}^{\left( r \right)}_{\text{const},i}}
\end{flalign*}

\noindent
\newline
\paragraph{Computing $\dv{}{\tilde{\kappa}^{\left( r \right)}_i}
\mathcal{L}_{\text{RGB},\text{FlaRe}}$}
\begin{flalign*}
\dv{}{\tilde{\kappa}^{\left( r \right)}_i}
\mathcal{L}_{\text{RGB},\text{FlaRe}} &=
\dv{}{\tilde{\kappa}^{\left( r \right)}_i} \left( \frac{1}{3 \left| \mathcal{B} \right|} \sum\limits_{r \in \mathcal{B}} \sum\limits_{j=1}^3 { \left( \mathrm{I}_j^{\left( r \right)} - \hat{\mathrm{I}}_j^{\left( r \right)} \right) }^2 \right) =&\\
&= \frac{1}{3 \left| \mathcal{B} \right|} \cdot \dv{}{\tilde{\kappa}^{\left( r \right)}_i} \left( \sum\limits_{r \in \mathcal{B}} \sum\limits_{j=1}^3 { \left( \mathrm{I}_j^{\left( r \right)} - \hat{\mathrm{I}}_j^{\left( r \right)} \right) }^2 \right) =&\\
&= \frac{1}{3 \left| \mathcal{B} \right|} \sum\limits_{r \in \mathcal{B}} \dv{}{\tilde{\kappa}^{\left( r \right)}_i} \left( \sum\limits_{j=1}^3 { \left( \mathrm{I}_j^{\left( r \right)} - \hat{\mathrm{I}}_j^{\left( r \right)} \right) }^2 \right) =&\\
&= \frac{1}{3 \left| \mathcal{B} \right|} \sum\limits_{r \in \mathcal{B}} \sum\limits_{j=1}^3 \dv{}{\tilde{\kappa}^{\left( r \right)}_i} \left( { \left( \mathrm{I}_j^{\left( r \right)} - \hat{\mathrm{I}}_j^{\left( r \right)} \right) }^2 \right) =&\\
&= \frac{1}{3 \left| \mathcal{B} \right|} \sum\limits_{r \in \mathcal{B}} \sum\limits_{j=1}^3 2 \left( \mathrm{I}_j^{\left( r \right)} - \hat{\mathrm{I}}_j^{\left( r \right)} \right) \dv{}{\tilde{\kappa}^{\left( r \right)}_i} \mathrm{I}_j^{\left( r \right)} =&\\
&= \frac{2}{3 \left| \mathcal{B} \right|} \sum\limits_{r \in \mathcal{B}} \sum\limits_{j=1}^3 \left( \mathrm{I}_j^{\left( r \right)} - \hat{\mathrm{I}}_j^{\left( r \right)} \right) \dv{}{\tilde{\kappa}^{\left( r \right)}_i} \mathrm{I}_j^{\left( r \right)} =&\\
&= \frac{2}{3 \left| \mathcal{B} \right|} \sum\limits_{r \in \mathcal{B}} \sum\limits_{j=1}^3 \left( \mathrm{I}_j^{\left( r \right)} - \hat{\mathrm{I}}_j^{\left( r \right)} \right) \dv{}{\tilde{\kappa}^{\left( r \right)}_i} \operatorname{clamp} \left( \mathrm{\tilde{I}}_j^{\left( r \right)}, 0, 1 \right) =&\\
&= \frac{2}{3 \left| \mathcal{B} \right|} \sum\limits_{r \in \mathcal{B}} \sum\limits_{j=1}^3 \left( \mathrm{I}_j^{\left( r \right)} - \hat{\mathrm{I}}_j^{\left( r \right)} \right) \cdot \mathds{1}_{\mathrm{\tilde{I}}_j^{\left( r \right)} \in \left[ 0, 1 \right]} \cdot \dv{\mathrm{\tilde{I}}_j^{\left( r \right)}}{\tilde{\kappa}^{\left( r \right)}_i} =&\\
&= \frac{2}{3 \left| \mathcal{B} \right|} \sum\limits_{r \in \mathcal{B}} \sum\limits_{j=1}^3 \left( \mathrm{I}_j^{\left( r \right)} - \hat{\mathrm{I}}_j^{\left( r \right)} \right) \cdot \mathds{1}_{\mathrm{\tilde{I}}_j^{\left( r \right)} \in \left[ 0, 1 \right]} \cdot \dv{\mathrm{\tilde{I}}_j^{\left( r \right)}}{\alpha^{\left( r \right)}_i} \cdot \dv{\alpha^{\left( r \right)}_i}{\tilde{\kappa}^{\left( r \right)}_i}
\end{flalign*}

\noindent
\newline
\paragraph{Computing $\dv{}{p} \mathcal{L}_{\text{RGB},\text{FlaRe}}$ for $p \in \pi_{\mathrm{z}} \left( \mathcal{G} \right) \cup \mathcal{F} \cup \Theta \cup { \left \lbrace u_i^{\left( r\right)} \right \rbrace}_{
\begin{subarray}{l}
r \in \mathcal{B} \\ i \in \left \lbrace 1, \ldots, \left \lvert \mathcal{G}^{\left( r \right)} \right \rvert \right \rbrace
\end{subarray}
} \cup { \left \lbrace v_i^{\left( r\right)} \right \rbrace}_{
\begin{subarray}{l}
r \in \mathcal{B} \\ i \in \left \lbrace 1, \ldots, \left \lvert \mathcal{G}^{\left( r \right)} \right \rvert \right \rbrace
\end{subarray}
}$}
\noindent
\newline
\newline
Let $p \in \pi_{\mathrm{z}} \left( \mathcal{G} \right) \cup \mathcal{F} \cup \Theta \cup { \left \lbrace u_i^{\left( r\right)} \right \rbrace}_{
\begin{subarray}{l}
r \in \mathcal{B} \\ i \in \left \lbrace 1, \ldots, \left \lvert \mathcal{G}^{\left( r \right)} \right \rvert \right \rbrace
\end{subarray}
} \cup { \left \lbrace v_i^{\left( r\right)} \right \rbrace}_{
\begin{subarray}{l}
r \in \mathcal{B} \\ i \in \left \lbrace 1, \ldots, \left \lvert \mathcal{G}^{\left( r \right)} \right \rvert \right \rbrace
\end{subarray}
}$. Then:
\begin{flalign*}
\dv{}{p}
\mathcal{L}_{\text{RGB},\text{FlaRe}} &=
\dv{}{p} \left( \frac{1}{3 \left| \mathcal{B} \right|} \sum\limits_{r \in \mathcal{B}} \sum\limits_{j=1}^3 { \left( \mathrm{I}_j^{\left( r \right)} - \hat{\mathrm{I}}_j^{\left( r \right)} \right) }^2 \right) =&\\
&= \frac{1}{3 \left| \mathcal{B} \right|} \cdot \dv{}{p} \left( \sum\limits_{r \in \mathcal{B}} \sum\limits_{j=1}^3 { \left( \mathrm{I}_j^{\left( r \right)} - \hat{\mathrm{I}}_j^{\left( r \right)} \right) }^2 \right) =&\\
&= \frac{1}{3 \left| \mathcal{B} \right|} \sum\limits_{r \in \mathcal{B}} \dv{}{p} \left( \sum\limits_{j=1}^3 { \left( \mathrm{I}_j^{\left( r \right)} - \hat{\mathrm{I}}_j^{\left( r \right)} \right) }^2 \right) =&\\
&= \frac{1}{3 \left| \mathcal{B} \right|} \sum\limits_{r \in \mathcal{B}} \sum\limits_{j=1}^3 \dv{}{p} \left( { \left( \mathrm{I}_j^{\left( r \right)} - \hat{\mathrm{I}}_j^{\left( r \right)} \right) }^2 \right) =&\\
&= \frac{1}{3 \left| \mathcal{B} \right|} \sum\limits_{r \in \mathcal{B}} \sum\limits_{j=1}^3 2 \left( \mathrm{I}_j^{\left( r \right)} - \hat{\mathrm{I}}_j^{\left( r \right)} \right) \dv{}{p} \mathrm{I}_j^{\left( r \right)} =&\\
&= \frac{2}{3 \left| \mathcal{B} \right|} \sum\limits_{r \in \mathcal{B}} \sum\limits_{j=1}^3 \left( \mathrm{I}_j^{\left( r \right)} - \hat{\mathrm{I}}_j^{\left( r \right)} \right) \dv{}{p} \mathrm{I}_j^{\left( r \right)} =&\\
&= \frac{2}{3 \left| \mathcal{B} \right|} \sum\limits_{r \in \mathcal{B}} \sum\limits_{j=1}^3 \left( \mathrm{I}_j^{\left( r \right)} - \hat{\mathrm{I}}_j^{\left( r \right)} \right) \dv{}{p} \operatorname{clamp} \left( \mathrm{\tilde{I}}_j^{\left( r \right)}, 0, 1 \right) =&\\
&= \frac{2}{3 \left| \mathcal{B} \right|} \sum\limits_{r \in \mathcal{B}} \sum\limits_{j=1}^3 \left( \mathrm{I}_j^{\left( r \right)} - \hat{\mathrm{I}}_j^{\left( r \right)} \right) \cdot \mathds{1}_{\mathrm{\tilde{I}}_j^{\left( r \right)} \in \left[ 0, 1 \right]} \cdot \dv{\mathrm{\tilde{I}}_j^{\left( r \right)}}{p}
\end{flalign*}
Let us put:
{
\small
$$
\delta_{4,i}^{\left( r \right)} = \begin{bmatrix}
\frac{2}{3 \left| \mathcal{B} \right|} \left( \mathrm{I}_1^{\left( r \right)} - \hat{\mathrm{I}}_1^{\left( r \right)} \right) \cdot \mathds{1}_{\mathrm{\tilde{I}}_1^{\left( r \right)} \in \left[ 0, 1 \right]} \cdot \mathrm{c}^{\left( r \right)}_{\text{const},i,1} \alpha^{\left( r \right)}_i T^{\left( r \right)}_i \cdot \dv{}{p} \Phi_1 \left( \left( u^{\left( r \right)}_i, v^{\left( r \right)}_i \right), \mathrm{\vec{d}}^{\left( r \right)}, \mathrm{z}^{\left( r \right)}_i \;; \mathcal{G}, \mathcal{F}, \Theta \right) \\
\frac{2}{3 \left| \mathcal{B} \right|} \left( \mathrm{I}_2^{\left( r \right)} - \hat{\mathrm{I}}_2^{\left( r \right)} \right) \cdot \mathds{1}_{\mathrm{\tilde{I}}_2^{\left( r \right)} \in \left[ 0, 1 \right]} \cdot \mathrm{c}^{\left( r \right)}_{\text{const},i,2} \alpha^{\left( r \right)}_i T^{\left( r \right)}_i \cdot \dv{}{p} \Phi_2 \left( \left( u^{\left( r \right)}_i, v^{\left( r \right)}_i \right), \mathrm{\vec{d}}^{\left( r \right)}, \mathrm{z}^{\left( r \right)}_i \;; \mathcal{G}, \mathcal{F}, \Theta \right) \\
\frac{2}{3 \left| \mathcal{B} \right|} \left( \mathrm{I}_3^{\left( r \right)} - \hat{\mathrm{I}}_3^{\left( r \right)} \right) \cdot \mathds{1}_{\mathrm{\tilde{I}}_3^{\left( r \right)} \in \left[ 0, 1 \right]} \cdot \mathrm{c}^{\left( r \right)}_{\text{const},i,3} \alpha^{\left( r \right)}_i T^{\left( r \right)}_i \cdot \dv{}{p} \Phi_3 \left( \left( u^{\left( r \right)}_i, v^{\left( r \right)}_i \right), \mathrm{\vec{d}}^{\left( r \right)}, \mathrm{z}^{\left( r \right)}_i \;; \mathcal{G}, \mathcal{F}, \Theta \right) \\
\frac{2}{3 \left| \mathcal{B} \right|} \left( \sum\limits_{j=1}^3 \left( \mathrm{I}_j^{\left( r \right)} - \hat{\mathrm{I}}_j^{\left( r \right)} \right) \cdot \mathds{1}_{\mathrm{\tilde{I}}_j^{\left( r \right)} \in \left[ 0, 1 \right]} \cdot \dv{\mathrm{\tilde{I}}_j^{\left( r \right)}}{\alpha^{\left( r \right)}_i} \right) \cdot \alpha^{\left( r \right)}_{\text{const},i} \cdot \exp \left( -\frac{ { \left( \left( u^{ \left( r \right)}_i \right)^2+\left( v^{\left( r \right)}_i \right)^2 \right) }^{\kappa^{\left( r \right)}_i} }{2\kappa^{\left( r \right)}_i} \right) \cdot \dv{}{p} \Phi_4 \left( \left( u^{\left( r \right)}_i, v^{\left( r \right)}_i \right), \mathrm{\vec{d}}^{\left( r \right)}, \mathrm{z}^{\left( r \right)}_i \;; \mathcal{G}, \mathcal{F}, \Theta \right) \\
\end{bmatrix}
$$
}
Since:
\begin{flalign*}
\dv{\alpha^{\left( r \right)}_i}{p} &=
\dv{}{p} \left( \alpha^{\left( r \right)}_{\text{const},i} \cdot \alpha^{\left( r \right)}_{\text{MLP},i} \cdot \exp \left( -\frac{ { \left( \left( u^{ \left( r \right)}_i \right)^2+\left( v^{\left( r \right)}_i \right)^2 \right) }^{\kappa^{\left( r \right)}_i} }{2\kappa^{\left( r \right)}_i} \right) \right) =&\\
&= \alpha^{\left( r \right)}_{\text{const},i} \cdot \exp \left( -\frac{ { \left( \left( u^{ \left( r \right)}_i \right)^2+\left( v^{\left( r \right)}_i \right)^2 \right) }^{\kappa^{\left( r \right)}_i} }{2\kappa^{\left( r \right)}_i} \right) \cdot \dv{}{p} \alpha^{\left( r \right)}_{\text{MLP},i} =&\\
&= \alpha^{\left( r \right)}_{\text{const},i} \cdot \exp \left( -\frac{ { \left( \left( u^{ \left( r \right)}_i \right)^2+\left( v^{\left( r \right)}_i \right)^2 \right) }^{\kappa^{\left( r \right)}_i} }{2\kappa^{\left( r \right)}_i} \right) \cdot \dv{}{p} \Phi_4 \left( \left( u^{\left( r \right)}_i, v^{\left( r \right)}_i \right), \mathrm{\vec{d}}^{\left( r \right)}, \mathrm{z}^{\left( r \right)}_i \;; \mathcal{G}, \mathcal{F}, \Theta \right)
\end{flalign*}
and:
\begin{flalign*}
\dv{\mathrm{\tilde{I}}_j^{\left( r \right)}}{\mathrm{c}^{\left( r \right)}_{\text{MLP},i,j}} &=
\dv{}{\mathrm{c}^{\left( r \right)}_{\text{MLP},i,j}} \left( \sum_{k=1}^{\left \lvert \mathcal{G}^{\left( r \right)} \right \rvert} \mathrm{c}^{\left( r \right)}_{\text{const},k,j} \mathrm{c}^{\left( r \right)}_{\text{MLP},k,j} \alpha^{\left( r \right)}_k T^{\left( r \right)}_k \right) =&\\
&= \sum_{k=1}^{\left \lvert \mathcal{G}^{\left( r \right)} \right \rvert} \dv{}{\mathrm{c}^{\left( r \right)}_{\text{MLP},i,j}} \left( \mathrm{c}^{\left( r \right)}_{\text{const},k,j} \mathrm{c}^{\left( r \right)}_{\text{MLP},k,j} \alpha^{\left( r \right)}_k T^{\left( r \right)}_k \right) =&\\
&= \sum_{k=1}^{\left \lvert \mathcal{G}^{\left( r \right)} \right \rvert} \mathrm{c}^{\left( r \right)}_{\text{const},k,j} \alpha^{\left( r \right)}_k T^{\left( r \right)}_k \cdot \dv{\mathrm{c}^{\left( r \right)}_{\text{MLP},k,j}}{\mathrm{c}^{\left( r \right)}_{\text{MLP},i,j}} =&\\
&= \sum_{k=1}^{\left \lvert \mathcal{G}^{\left( r \right)} \right \rvert} \mathrm{c}^{\left( r \right)}_{\text{const},k,j} \alpha^{\left( r \right)}_k T^{\left( r \right)}_k \cdot \delta_{i{k}} =&\\
&= \mathrm{c}^{\left( r \right)}_{\text{const},i,j} \alpha^{\left( r \right)}_i T^{\left( r \right)}_i
\end{flalign*}
we have:
\begin{flalign*}
\dv{}{p} \mathcal{L}_{\text{RGB},\text{FlaRe}} &=
\frac{2}{3 \left| \mathcal{B} \right|} \sum\limits_{r \in \mathcal{B}} \sum\limits_{j=1}^3 \left( \mathrm{I}_j^{\left( r \right)} - \hat{\mathrm{I}}_j^{\left( r \right)} \right) \cdot \mathds{1}_{\mathrm{\tilde{I}}_j^{\left( r \right)} \in \left[ 0, 1 \right]} \cdot \dv{\mathrm{\tilde{I}}_j^{\left( r \right)}}{p} =&\\
&= \frac{2}{3 \left| \mathcal{B} \right|} \sum\limits_{r \in \mathcal{B}} \sum\limits_{j=1}^3 \left( \mathrm{I}_j^{\left( r \right)} - \hat{\mathrm{I}}_j^{\left( r \right)} \right) \cdot \mathds{1}_{\mathrm{\tilde{I}}_j^{\left( r \right)} \in \left[ 0, 1 \right]} \cdot \sum_{k=1}^{\left \lvert \mathcal{G}^{\left( r \right)} \right \rvert} \left( \dv{\mathrm{\tilde{I}}_j^{\left( r \right)}}{\alpha^{\left( r \right)}_k} \cdot \dv{\alpha^{\left( r \right)}_k}{p} + \dv{\mathrm{\tilde{I}}_j^{\left( r \right)}}{\mathrm{c}^{\left( r \right)}_{\text{MLP},k,j}} \cdot \dv{\mathrm{c}^{\left( r \right)}_{\text{MLP},k,j}}{p} \right) =&\\
&= \sum\limits_{r \in \mathcal{B}} \sum_{k=1}^{\left \lvert \mathcal{G}^{\left( r \right)} \right \rvert} \left \langle \delta_{4,k}^{\left( r \right)}, \dv{\Phi \left( \left( u^{\left( r \right)}_k, v^{\left( r \right)}_k \right), \mathrm{\vec{d}}^{\left( r \right)}, \mathrm{z}^{\left( r \right)}_k \;; \mathcal{G}, \mathcal{F}, \Theta \right)}{p} \right \rangle
\end{flalign*}

\noindent
\newline
\paragraph{Computing $\dv{}{p} \mathcal{L}_{\text{RGB},\text{FlaRe}}$ for $p \in \pi_{\upmu} \left( \mathcal{G} \right) \cup \pi_{\mathrm{\tilde{s}}} \left( \mathcal{G} \right) \cup \pi_{\mathrm{q}} \left( \mathcal{G} \right)$}
\noindent
\newline
\newline
Let $p \in \pi_{\upmu} \left( \mathcal{G} \right) \cup \pi_{\mathrm{\tilde{s}}} \left( \mathcal{G} \right) \cup \pi_{\mathrm{q}} \left( \mathcal{G} \right)$. Then:
\begin{flalign*}
\dv{}{p}
\mathcal{L}_{\text{RGB},\text{FlaRe}} &=
\dv{}{p} \left( \frac{1}{3 \left| \mathcal{B} \right|} \sum\limits_{r \in \mathcal{B}} \sum\limits_{j=1}^3 { \left( \mathrm{I}_j^{\left( r \right)} - \hat{\mathrm{I}}_j^{\left( r \right)} \right) }^2 \right) =&\\
&= \frac{1}{3 \left| \mathcal{B} \right|} \cdot \dv{}{p} \left( \sum\limits_{r \in \mathcal{B}} \sum\limits_{j=1}^3 { \left( \mathrm{I}_j^{\left( r \right)} - \hat{\mathrm{I}}_j^{\left( r \right)} \right) }^2 \right) =&\\
&= \frac{1}{3 \left| \mathcal{B} \right|} \sum\limits_{r \in \mathcal{B}} \dv{}{p} \left( \sum\limits_{j=1}^3 { \left( \mathrm{I}_j^{\left( r \right)} - \hat{\mathrm{I}}_j^{\left( r \right)} \right) }^2 \right) =&\\
&= \frac{1}{3 \left| \mathcal{B} \right|} \sum\limits_{r \in \mathcal{B}} \sum\limits_{j=1}^3 \dv{}{p} \left( { \left( \mathrm{I}_j^{\left( r \right)} - \hat{\mathrm{I}}_j^{\left( r \right)} \right) }^2 \right) =&\\
&= \frac{1}{3 \left| \mathcal{B} \right|} \sum\limits_{r \in \mathcal{B}} \sum\limits_{j=1}^3 2 \left( \mathrm{I}_j^{\left( r \right)} - \hat{\mathrm{I}}_j^{\left( r \right)} \right) \dv{}{p} \mathrm{I}_j^{\left( r \right)} =&\\
&= \frac{2}{3 \left| \mathcal{B} \right|} \sum\limits_{r \in \mathcal{B}} \sum\limits_{j=1}^3 \left( \mathrm{I}_j^{\left( r \right)} - \hat{\mathrm{I}}_j^{\left( r \right)} \right) \dv{}{p} \mathrm{I}_j^{\left( r \right)} =&\\
&= \frac{2}{3 \left| \mathcal{B} \right|} \sum\limits_{r \in \mathcal{B}} \sum\limits_{j=1}^3 \left( \mathrm{I}_j^{\left( r \right)} - \hat{\mathrm{I}}_j^{\left( r \right)} \right) \dv{}{p} \operatorname{clamp} \left( \mathrm{\tilde{I}}_j^{\left( r \right)}, 0, 1 \right) =&\\
&= \frac{2}{3 \left| \mathcal{B} \right|} \sum\limits_{r \in \mathcal{B}} \sum\limits_{j=1}^3 \left( \mathrm{I}_j^{\left( r \right)} - \hat{\mathrm{I}}_j^{\left( r \right)} \right) \cdot \mathds{1}_{\mathrm{\tilde{I}}_j^{\left( r \right)} \in \left[ 0, 1 \right]} \cdot \dv{\mathrm{\tilde{I}}_j^{\left( r \right)}}{p}
\end{flalign*}
Since:
\begin{flalign*}
\dv{\alpha^{\left( r \right)}_i}{u^{\left( r \right)}_i} &=
\dv{}{u^{\left( r \right)}_i} \left( \alpha^{\left( r \right)}_{\text{const},i} \cdot \alpha^{\left( r \right)}_{\text{MLP},i} \cdot \exp \left( -\frac{ { \left( \left( u^{ \left( r \right)}_i \right)^2+\left( v^{\left( r \right)}_i \right)^2 \right) }^{\kappa^{\left( r \right)}_i} }{2\kappa^{\left( r \right)}_i} \right) \right) =&\\
\\
\\
&= \alpha^{\left( r \right)}_{\text{const},i} \cdot \alpha^{\left( r \right)}_{\text{MLP},i} \cdot \dv{}{u^{\left( r \right)}_i} \exp \left( -\frac{ { \left( \left( u^{ \left( r \right)}_i \right)^2+\left( v^{\left( r \right)}_i \right)^2 \right) }^{\kappa^{\left( r \right)}_i} }{2\kappa^{\left( r \right)}_i} \right) +&\\
&+ \alpha^{\left( r \right)}_{\text{const},i} \cdot \exp \left( -\frac{ { \left( \left( u^{ \left( r \right)}_i \right)^2+\left( v^{\left( r \right)}_i \right)^2 \right) }^{\kappa^{\left( r \right)}_i} }{2\kappa^{\left( r \right)}_i} \right) \cdot \dv{}{u^{\left( r \right)}_i} \alpha^{\left( r \right)}_{\text{MLP},i} =&\\
\\
\\
&= \alpha^{\left( r \right)}_{\text{const},i} \cdot \alpha^{\left( r \right)}_{\text{MLP},i} \cdot \exp \left( -\frac{ { \left( \left( u^{ \left( r \right)}_i \right)^2+\left( v^{\left( r \right)}_i \right)^2 \right) }^{\kappa^{\left( r \right)}_i} }{2\kappa^{\left( r \right)}_i} \right) \cdot \left( -\frac{1}{2\kappa^{\left( r \right)}_i} \right) \cdot {\kappa^{\left( r \right)}_i} \cdot { \left( \left( u^{ \left( r \right)}_i \right)^2+\left( v^{\left( r \right)}_i \right)^2 \right) }^{\kappa^{\left( r \right)}_i - 1} \cdot {2 u^{ \left( r \right)}_i} +&\\
&+ \alpha^{\left( r \right)}_{\text{const},i} \cdot \exp \left( -\frac{ { \left( \left( u^{ \left( r \right)}_i \right)^2+\left( v^{\left( r \right)}_i \right)^2 \right) }^{\kappa^{\left( r \right)}_i} }{2\kappa^{\left( r \right)}_i} \right) \cdot \dv{}{u^{\left( r \right)}_i} \Phi_4 \left( \left( u^{\left( r \right)}_i, v^{\left( r \right)}_i \right), \mathrm{\vec{d}}^{\left( r \right)}, \mathrm{z}^{\left( r \right)}_i \;; \mathcal{G}, \mathcal{F}, \Theta \right) =&\\
\\
\\
&= \alpha^{\left( r \right)}_{\text{const},i} \cdot \alpha^{\left( r \right)}_{\text{MLP},i} \cdot \exp \left( -\frac{ { \left( \left( u^{ \left( r \right)}_i \right)^2+\left( v^{\left( r \right)}_i \right)^2 \right) }^{\kappa^{\left( r \right)}_i} }{2\kappa^{\left( r \right)}_i} \right) \cdot \left( -\frac{2\kappa^{\left( r \right)}_i}{2\kappa^{\left( r \right)}_i} \right) \cdot { \left( \left( u^{ \left( r \right)}_i \right)^2+\left( v^{\left( r \right)}_i \right)^2 \right) }^{\kappa^{\left( r \right)}_i - 1} \cdot {u^{ \left( r \right)}_i} +&\\
&+ \alpha^{\left( r \right)}_{\text{const},i} \cdot \exp \left( -\frac{ { \left( \left( u^{ \left( r \right)}_i \right)^2+\left( v^{\left( r \right)}_i \right)^2 \right) }^{\kappa^{\left( r \right)}_i} }{2\kappa^{\left( r \right)}_i} \right) \cdot \dv{}{u^{\left( r \right)}_i} \Phi_4 \left( \left( u^{\left( r \right)}_i, v^{\left( r \right)}_i \right), \mathrm{\vec{d}}^{\left( r \right)}, \mathrm{z}^{\left( r \right)}_i \;; \mathcal{G}, \mathcal{F}, \Theta \right) =&\\
\\
\\
&= -\alpha^{\left( r \right)}_{\text{const},i} \cdot \alpha^{\left( r \right)}_{\text{MLP},i} \cdot \exp \left( -\frac{ { \left( \left( u^{ \left( r \right)}_i \right)^2+\left( v^{\left( r \right)}_i \right)^2 \right) }^{\kappa^{\left( r \right)}_i} }{2\kappa^{\left( r \right)}_i} \right) \cdot { \left( \left( u^{ \left( r \right)}_i \right)^2+\left( v^{\left( r \right)}_i \right)^2 \right) }^{\kappa^{\left( r \right)}_i - 1} \cdot {u^{ \left( r \right)}_i} +&\\
&+ \alpha^{\left( r \right)}_{\text{const},i} \cdot \exp \left( -\frac{ { \left( \left( u^{ \left( r \right)}_i \right)^2+\left( v^{\left( r \right)}_i \right)^2 \right) }^{\kappa^{\left( r \right)}_i} }{2\kappa^{\left( r \right)}_i} \right) \cdot \dv{}{u^{\left( r \right)}_i} \Phi_4 \left( \left( u^{\left( r \right)}_i, v^{\left( r \right)}_i \right), \mathrm{\vec{d}}^{\left( r \right)}, \mathrm{z}^{\left( r \right)}_i \;; \mathcal{G}, \mathcal{F}, \Theta \right) =&\\
\\
\\
&= -\alpha^{\left( r \right)}_i \cdot { \left( \left( u^{ \left( r \right)}_i \right)^2+\left( v^{\left( r \right)}_i \right)^2 \right) }^{\kappa^{\left( r \right)}_i - 1} \cdot {u^{ \left( r \right)}_i} +&\\
&+ \alpha^{\left( r \right)}_{\text{const},i} \cdot \exp \left( -\frac{ { \left( \left( u^{ \left( r \right)}_i \right)^2+\left( v^{\left( r \right)}_i \right)^2 \right) }^{\kappa^{\left( r \right)}_i} }{2\kappa^{\left( r \right)}_i} \right) \cdot \dv{}{u^{\left( r \right)}_i} \Phi_4 \left( \left( u^{\left( r \right)}_i, v^{\left( r \right)}_i \right), \mathrm{\vec{d}}^{\left( r \right)}, \mathrm{z}^{\left( r \right)}_i \;; \mathcal{G}, \mathcal{F}, \Theta \right) =&\\
\end{flalign*}
and (the derivation is analogous to the above):
\begin{flalign*}
\dv{\alpha^{\left( r \right)}_i}{v^{\left( r \right)}_i} &=
\dv{}{v^{\left( r \right)}_i} \left( \alpha^{\left( r \right)}_{\text{const},i} \cdot \alpha^{\left( r \right)}_{\text{MLP},i} \cdot \exp \left( -\frac{ { \left( \left( u^{ \left( r \right)}_i \right)^2+\left( v^{\left( r \right)}_i \right)^2 \right) }^{\kappa^{\left( r \right)}_i} }{2\kappa^{\left( r \right)}_i} \right) \right) =&\\
\\
\\
&= -\alpha^{\left( r \right)}_i \cdot { \left( \left( u^{ \left( r \right)}_i \right)^2+\left( v^{\left( r \right)}_i \right)^2 \right) }^{\kappa^{\left( r \right)}_i - 1} \cdot {v^{ \left( r \right)}_i} +&\\
&+ \alpha^{\left( r \right)}_{\text{const},i} \cdot \exp \left( -\frac{ { \left( \left( u^{ \left( r \right)}_i \right)^2+\left( v^{\left( r \right)}_i \right)^2 \right) }^{\kappa^{\left( r \right)}_i} }{2\kappa^{\left( r \right)}_i} \right) \cdot \dv{}{v^{\left( r \right)}_i} \Phi_4 \left( \left( u^{\left( r \right)}_i, v^{\left( r \right)}_i \right), \mathrm{\vec{d}}^{\left( r \right)}, \mathrm{z}^{\left( r \right)}_i \;; \mathcal{G}, \mathcal{F}, \Theta \right) =&\\
\end{flalign*}
given that:
\newline
\begin{flalign*}
\dv{\mathrm{\tilde{I}}_j^{\left( r \right)}}{p} &=
\sum_{k=1}^{\left \lvert \mathcal{G}^{\left( r \right)} \right \rvert} \left( \dv{\mathrm{\tilde{I}}_j^{\left( r \right)}}{\alpha^{\left( r \right)}_k} \cdot \left( \dv{\alpha^{\left( r \right)}_k}{u^{\left( r \right)}_i} \cdot \dv{u^{\left( r \right)}_i}{p} + \dv{\alpha^{\left( r \right)}_k}{v^{\left( r \right)}_i} \cdot \dv{v^{\left( r \right)}_i}{p} \right) + \dv{\mathrm{\tilde{I}}_j^{\left( r \right)}}{\mathrm{c}^{\left( r \right)}_{\text{MLP},k,j}} \cdot \left( \dv{\mathrm{c}^{\left( r \right)}_{\text{MLP},k,j}}{u^{\left( r \right)}_i} \cdot \dv{u^{\left( r \right)}_i}{p} + \dv{\mathrm{c}^{\left( r \right)}_{\text{MLP},k,j}}{v^{\left( r \right)}_i} \cdot \dv{v^{\left( r \right)}_i}{p} \right) \right) =&\\
&= \sum_{k=1}^{\left \lvert \mathcal{G}^{\left( r \right)} \right \rvert} \left( \left( \dv{\mathrm{\tilde{I}}_j^{\left( r \right)}}{\alpha^{\left( r \right)}_k} \cdot \dv{\alpha^{\left( r \right)}_k}{u^{\left( r \right)}_i} + \dv{\mathrm{\tilde{I}}_j^{\left( r \right)}}{\mathrm{c}^{\left( r \right)}_{\text{MLP},k,j}} \cdot \dv{\mathrm{c}^{\left( r \right)}_{\text{MLP},k,j}}{u^{\left( r \right)}_i} \right) \cdot \dv{u^{\left( r \right)}_i}{p} + \left( \dv{\mathrm{\tilde{I}}_j^{\left( r \right)}}{\alpha^{\left( r \right)}_k} \cdot \dv{\alpha^{\left( r \right)}_k}{v^{\left( r \right)}_i} + \dv{\mathrm{\tilde{I}}_j^{\left( r \right)}}{\mathrm{c}^{\left( r \right)}_{\text{MLP},k,j}} \cdot \dv{\mathrm{c}^{\left( r \right)}_{\text{MLP},k,j}}{v^{\left( r \right)}_i} \right) \cdot \dv{v^{\left( r \right)}_i}{p} \right)
\end{flalign*}
\newline
\newline
after a few algebraic steps, we obtain:
\newline
\begin{flalign*}
& \dv{}{p} \mathcal{L}_{\text{RGB},\text{FlaRe}} =&\\
\\
&= \sum\limits_{r \in \mathcal{B}} \sum_{i=1}^{\left \lvert \mathcal{G}^{\left( r \right)} \right \rvert} \left \langle \delta_{4,i}^{\left( r \right)}, \dv{\Phi \left( \left( u^{\left( r \right)}_i, v^{\left( r \right)}_i \right), \mathrm{\vec{d}}^{\left( r \right)}, \mathrm{z}^{\left( r \right)}_i \;; \mathcal{G}, \mathcal{F}, \Theta \right)}{u^{\left( r \right)}_i} \right \rangle \cdot \dv{u^{\left( r \right)}_i}{p} +&\\
&+ \sum\limits_{r \in \mathcal{B}} \sum_{i=1}^{\left \lvert \mathcal{G}^{\left( r \right)} \right \rvert} \left \langle \delta_{4,i}^{\left( r \right)}, \dv{\Phi \left( \left( u^{\left( r \right)}_i, v^{\left( r \right)}_i \right), \mathrm{\vec{d}}^{\left( r \right)}, \mathrm{z}^{\left( r \right)}_i \;; \mathcal{G}, \mathcal{F}, \Theta \right)}{v^{\left( r \right)}_i} \right \rangle \cdot \dv{v^{\left( r \right)}_i}{p} -&\\
&- \frac{2}{3 \left| \mathcal{B} \right|} \sum\limits_{r \in \mathcal{B}} \sum\limits_{j=1}^3 \left( \mathrm{I}_j^{\left( r \right)} - \hat{\mathrm{I}}_j^{\left( r \right)} \right) \cdot \mathds{1}_{\mathrm{\tilde{I}}_j^{\left( r \right)} \in \left[ 0, 1 \right]} \cdot \sum_{i=1}^{\left \lvert \mathcal{G}^{\left( r \right)} \right \rvert} \dv{\mathrm{\tilde{I}}_j^{\left( r \right)}}{\alpha^{\left( r \right)}_i} \cdot \alpha^{\left( r \right)}_i \cdot { \left( \left( u^{ \left( r \right)}_i \right)^2+\left( v^{\left( r \right)}_i \right)^2 \right) }^{\kappa^{\left( r \right)}_i - 1} \cdot \left( {u^{ \left( r \right)}_i} \cdot \dv{u^{\left( r \right)}_i}{p} + {v^{ \left( r \right)}_i} \cdot \dv{v^{\left( r \right)}_i}{p} \right)
\end{flalign*}

\noindent
\newline
\paragraph{Computing $\dv{\mathcal{L}_s}{\tilde{\mathrm{s}}}$}
\begin{flalign*}
\dv{\mathcal{L}_s}{\mathrm{\tilde{s}}_{i,j}} &=
\dv{}{\tilde{\mathrm{s}}_{i,j}} \left( \frac{1}{N} \sum_{k=1}^N \left \lVert \mathrm{s}_k \right \rVert \right) =&\\
&= \frac{1}{N} \cdot \dv{}{\tilde{\mathrm{s}}_{i,j}} \left( \sum_{k=1}^N \left \lVert \mathrm{s}_k \right \rVert \right) =&\\
&= \frac{1}{N} \cdot \sum_{k=1}^N \dv{}{\tilde{\mathrm{s}}_{i,j}} \left \lVert \mathrm{s}_k \right \rVert =&\\
&= \frac{1}{N} \cdot \sum_{k=1}^N \dv{}{\tilde{\mathrm{s}}_{i,j}} \sqrt{ \left \langle \mathrm{s}_k, \mathrm{s}_k \right \rangle } =&\\
&= \frac{1}{N} \cdot \sum_{k=1}^N \frac{1}{ 2 \sqrt{ \left \langle \mathrm{s}_k, \mathrm{s}_k \right \rangle } } \cdot \dv{}{\tilde{\mathrm{s}}_{i,j}} \left \langle \mathrm{s}_k, \mathrm{s}_k \right \rangle =&\\
&= \frac{1}{N} \cdot \sum_{k=1}^N \frac{1}{ 2 \sqrt{ \left \langle \mathrm{s}_k, \mathrm{s}_k \right \rangle } } \cdot \dv{}{\mathrm{s}_{i,j}} \left \langle \mathrm{s}_k, \mathrm{s}_k \right \rangle \cdot \dv{\mathrm{s}_{i,j}}{\tilde{\mathrm{s}}_{i,j}} =&\\
&= \frac{1}{N} \cdot \sum_{k=1}^N \frac{1}{ 2 \sqrt{ \left \langle \mathrm{s}_k, \mathrm{s}_k \right \rangle } } \cdot 2 \left \langle \mathrm{s}_k, \dv{\mathrm{s}_k}{\mathrm{s}_{i,j}} \right \rangle \cdot \dv{\mathrm{s}_{i,j}}{\tilde{\mathrm{s}}_{i,j}} =&\\
&= \frac{1}{N} \cdot \sum_{k=1}^N \frac{1}{ 2 \sqrt{ \left \langle \mathrm{s}_k, \mathrm{s}_k \right \rangle } } \cdot 2 \left \langle \mathrm{s}_k, \delta_{i{k}} \cdot \vec{\varepsilon}_j \right \rangle \cdot \dv{\mathrm{s}_{i,j}}{\tilde{\mathrm{s}}_{i,j}} =&\\
&= \frac{1}{N} \cdot \sum_{k=1}^N \frac{1}{ 2 \sqrt{ \left \langle \mathrm{s}_k, \mathrm{s}_k \right \rangle } } \cdot 2 \left \langle \mathrm{s}_k, \vec{\varepsilon}_j \right \rangle \cdot \delta_{i{k}} \cdot \dv{\mathrm{s}_{i,j}}{\tilde{\mathrm{s}}_{i,j}} =&\\
&= \frac{1}{N} \cdot \frac{1}{ 2 \sqrt{ \left \langle \mathrm{s}_i, \mathrm{s}_i \right \rangle } } \cdot 2 \left \langle \mathrm{s}_i, \vec{\varepsilon}_j \right \rangle \cdot \dv{\mathrm{s}_{i,j}}{\tilde{\mathrm{s}}_{i,j}} =&\\
&= \frac{1}{N} \cdot \frac{1}{ 2 \sqrt{ \left \langle \mathrm{s}_i, \mathrm{s}_i \right \rangle } } \cdot 2 \cdot \mathrm{s}_{i,j} \cdot \dv{\mathrm{s}_{i,j}}{\tilde{\mathrm{s}}_{i,j}} =&\\
&= \frac{1}{N} \cdot \frac{\mathrm{s}_{i,j}}{ \sqrt{ \left \langle \mathrm{s}_i, \mathrm{s}_i \right \rangle } } \cdot \dv{\mathrm{s}_{i,j}}{\tilde{\mathrm{s}}_{i,j}} =&\\
&= \frac{1}{N} \cdot \frac{\mathrm{s}_{i,j}}{ \sqrt{ \left \langle \mathrm{s}_i, \mathrm{s}_i \right \rangle } } \cdot \dv{}{\tilde{\mathrm{s}}_{i,j}} \exp \left( \mathrm{\tilde{s}}_{i,j} \right) =&\\
&= \frac{1}{N} \cdot \frac{\mathrm{s}_{i,j}}{ \sqrt{ \left \langle \mathrm{s}_i, \mathrm{s}_i \right \rangle } } \cdot \exp \left( \mathrm{\tilde{s}}_{i,j} \right) \cdot 1 =&\\
&= \frac{1}{N} \cdot \frac{\mathrm{s}_{i,j}}{ \sqrt{ \left \langle \mathrm{s}_i, \mathrm{s}_i \right \rangle } } \cdot \mathrm{s}_{i,j} =&\\
&= \frac{1}{N} \cdot \frac{\mathrm{s}_{i,j}^2}{ \sqrt{ \left \langle \mathrm{s}_i, \mathrm{s}_i \right \rangle } }
\end{flalign*}
Hence:
\begin{flalign*}
& \dv{\mathcal{L}_s}{\mathrm{\tilde{s}}_i} = \frac{1}{N} \cdot \frac{\mathrm{s}_i^2}{ \sqrt{ \left \langle \mathrm{s}_i, \mathrm{s}_i \right \rangle } }
&
\end{flalign*}

\paragraph{Computing $\dv{\mathrm{\tilde{I}}^{\left( r \right)}}{\alpha^{\left( r \right)}_i}$}
\noindent
\newline
\newline
As for $\dv{\mathrm{\tilde{I}}^{\left( r \right)}}{\alpha^{\left( r \right)}_i}$, it can be efficiently computed making use of the following three lemmas:
\newline
\begin{lemma}
\label{lemma1}
Let $r \in \mathcal{B}$. Then:
$$
\dv{\mathrm{\tilde{I}}^{\left( r \right)}}{\alpha_1^{\left( r \right)}} = \frac{-\mathrm{\tilde{I}}^{\left( r \right)} + \mathrm{c}_1^{\left( r \right)}}{1-\alpha_1^{\left( r \right)}}
$$
\end{lemma}
\begin{proof}
Since:
$$
\mathrm{\tilde{I}}^{\left( r \right)} =
\sum\limits_{i=1}^{\left \lvert \mathcal{G}^{\left( r \right)} \right \rvert} \mathrm{c}_i^{\left( r \right)} \alpha_i^{\left( r \right)} \prod\limits_{j=1}^{i-1} \left( 1 - \alpha_j^{\left( r \right)} \right)
$$
we have:
$$
\dv{\mathrm{\tilde{I}}^{\left( r \right)}}{\alpha_1^{\left( r \right)}} =
\mathrm{c}_1^{\left( r \right)} - \sum\limits_{i=2}^{\left \lvert \mathcal{G}^{\left( r \right)} \right \rvert} \mathrm{c}_i^{\left( r \right)} \alpha_i^{\left( r \right)} \prod\limits_{\substack{j=1 \\ j \ne 1 }}^{i-1} \left( 1-\alpha_j^{\left( r \right)} \right)
$$
Let us multiply both sides of the equation by $1-\alpha_1^{\left( r \right)}$. Then:
\begin{flalign*}
\dv{\mathrm{\tilde{I}}^{\left( r \right)}}{\alpha_1^{\left( r \right)}}\left( 1 - \alpha_1^{\left( r \right)} \right) &= \mathrm{c}_1^{\left( r \right)} \left( 1 - \alpha_1^{\left( r \right)} \right) - \sum\limits_{i=2}^{\left \lvert \mathcal{G}^{\left( r \right)} \right \rvert} \mathrm{c}_i^{\left( r \right)} \alpha_i^{\left( r \right)} \prod\limits_{j=1}^{i-1} \left( 1 - \alpha_j^{\left( r \right)} \right) =&\\
&= \mathrm{c}_1^{\left( r \right)} \left( 1 - \alpha_1^{\left( r \right)} \right) - \left( 0 + \sum\limits_{i=2}^{\left \lvert \mathcal{G}^{\left( r \right)} \right \rvert} \mathrm{c}_i^{\left( r \right)} \alpha_i^{\left( r \right)} \prod\limits_{j=1}^{i-1} \left( 1 - \alpha_j^{\left( r \right)} \right) \right) =&\\
&= \mathrm{c}_1^{\left( r \right)} \left( 1 - \alpha_1^{\left( r \right)} \right) - \left( \left( \mathrm{c}_1^{\left( r \right)} \alpha_1^{\left( r \right)} - \mathrm{c}_1^{\left( r \right)} \alpha_1^{\left( r \right)} \right) + \sum\limits_{i=2}^{\left \lvert \mathcal{G}^{\left( r \right)} \right \rvert} \mathrm{c}_i^{\left( r \right)} \alpha_i^{\left( r \right)} \prod\limits_{j=1}^{i-1} \left( 1 - \alpha_j^{\left( r \right)} \right) \right) =&\\
&= \mathrm{c}_1^{\left( r \right)} \left( \left( 1 - \alpha_1^{\left( r \right)} \right) + \alpha_1^{\left( r \right)} \right) - \sum\limits_{i=1}^{\left \lvert \mathcal{G}^{\left( r \right)} \right \rvert} \mathrm{c}_i^{\left( r \right)} \alpha_i^{\left( r \right)} \prod\limits_{j=1}^{i-1} \left( 1-\alpha_j \right) =&\\
&= \mathrm{c}_1^{\left( r \right)} \cdot 1 - \mathrm{\tilde{I}}^{\left( r \right)} =&\\
&= \mathrm{c}_1^{\left( r \right)} - \mathrm{\tilde{I}}^{\left( r \right)} =&\\
&= -\mathrm{\tilde{I}}^{\left( r \right)} + \mathrm{c}_1^{\left( r \right)}
\end{flalign*}
Hence:
$$
\dv{\mathrm{\tilde{I}}^{\left( r \right)}}{\alpha_1^{\left( r \right)}} = \frac{-\mathrm{\tilde{I}}^{\left( r \right)} + \mathrm{c}_1^{\left( r \right)}}{1-\alpha_1^{\left( r \right)}}
$$
\qedhere
\end{proof}

\begin{lemma}
\label{lemma2}
Let $r \in \mathcal{B}$ and $1 \le i < \left \lvert \mathcal{G} \left( r \right) \right \rvert$. Then:
$$
\dv{\mathrm{\tilde{I}}^{\left( r \right)}}{\alpha_{i+1}^{\left( r \right)}} = \frac{ \dv{\mathrm{\tilde{I}}^{\left( r \right)}}{\alpha_i^{\left( r \right)}} \cdot \left( 1 - \alpha_i^{\left( r \right)} \right) + \left( \mathrm{c}_{i+1}^{\left( r \right)} - \mathrm{c}_i^{\left( r \right)} \right) \prod\limits_{k=1}^i \left( 1 - \alpha_k^{\left( r \right)} \right)}{1 - \alpha_{i+1}^{\left( r \right)}}
$$
\end{lemma}
\begin{proof}
Since:
$$
\dv{\mathrm{\tilde{I}}^{\left( r \right)}}{\alpha_i^{\left( r \right)}} =
\mathrm{c}_i^{\left( r \right)} \prod\limits_{j=1}^{i-1} \left( 1 - \alpha_j^{\left( r \right)} \right) - \sum\limits_{j=i+1}^{\left \lvert \mathcal{G} \left( r \right) \right \rvert} \mathrm{c}_j^{\left( r \right)} \alpha_j^{\left( r \right)} \prod\limits_{\substack{k=1 \\ k \ne i }}^{j-1} \left( 1 - \alpha_k^{\left( r \right)} \right)
$$
after multiplying both sides of the equation by $1-\alpha_i^{\left( r \right)}$ we obtain:
\begin{flalign*}
\dv{\mathrm{\tilde{I}}^{\left( r \right)}}{\alpha_i^{\left( r \right)}}\left( 1 - \alpha_i^{\left( r \right)} \right) &= \mathrm{c}_i^{\left( r \right)} \left( 1 - \alpha_i^{\left( r \right)} \right) \prod\limits_{j=1}^{i-1} \left( 1 - \alpha_j^{\left( r \right)} \right) - \sum\limits_{j=i+1}^{\left \lvert \mathcal{G} \left( r \right) \right \rvert} \mathrm{c}_j^{\left( r \right)} \alpha_j^{\left( r \right)} \prod\limits_{k=1}^{j-1} \left( 1 - \alpha_k^{\left( r \right)} \right) =&\\
 &= \mathrm{c}_i^{\left( r \right)} \left( 1 - \alpha_i^{\left( r \right)} \right) \prod\limits_{j=1}^{i-1} \left( 1 - \alpha_j^{\left( r \right)} \right) + 0 - \sum\limits_{j=i+1}^{\left \lvert \mathcal{G} \left( r \right) \right \rvert} \mathrm{c}_j^{\left( r \right)} \alpha_j^{\left( r \right)} \prod\limits_{k=1}^{j-1} \left( 1 - \alpha_k^{\left( r \right)} \right) =&\\
 &= \mathrm{c}_i^{\left( r \right)} \left( 1 - \alpha_i^{\left( r \right)} \right) \prod\limits_{j=1}^{i-1} \left( 1 - \alpha_j^{\left( r \right)} \right) + 0 \cdot \prod\limits_{j=1}^{i-1} \left( 1 - \alpha_j^{\left( r \right)} \right) - \sum\limits_{j=i+1}^{\left \lvert \mathcal{G} \left( r \right) \right \rvert} \mathrm{c}_j^{\left( r \right)} \alpha_j^{\left( r \right)} \prod\limits_{k=1}^{j-1} \left( 1 - \alpha_k^{\left( r \right)} \right) =&\\
 &= \mathrm{c}_i^{\left( r \right)} \left( 1 - \alpha_i^{\left( r \right)} \right) \prod\limits_{j=1}^{i-1} \left( 1 - \alpha_j^{\left( r \right)} \right) + \left( \mathrm{c}_i^{\left( r \right)} \alpha_i^{\left( r \right)} - \mathrm{c}_i^{\left( r \right)} \alpha_i^{\left( r \right)} \right) \prod\limits_{j=1}^{i-1} \left( 1 - \alpha_j^{\left( r \right)} \right) - \sum\limits_{j=i+1}^{\left \lvert \mathcal{G} \left( r \right) \right \rvert} \mathrm{c}_j^{\left( r \right)} \alpha_j^{\left( r \right)} \prod\limits_{k=1}^{j-1} \left( 1 - \alpha_k^{\left( r \right)} \right) =&\\
 &= \mathrm{c}_i^{\left( r \right)} \left( \left( 1 -\alpha_i^{\left( r \right)} \right) + \alpha_i^{\left( r \right)} \right) \prod\limits_{j=1}^{i-1} \left( 1 -\alpha_j^{\left( r \right)} \right) - \sum\limits_{j=i}^{\left \lvert \mathcal{G} \left( r \right) \right \rvert} \mathrm{c}_j^{\left( r \right)} \alpha_j^{\left( r \right)} \prod\limits_{k=1}^{j-1} \left( 1 - \alpha_k^{\left( r \right)} \right) =&\\
 &= \mathrm{c}_i^{\left( r \right)} \cdot 1 \cdot \prod\limits_{j=1}^{i-1} \left( 1 - \alpha_j^{\left( r \right)} \right) - \sum\limits_{j=i}^{\left \lvert \mathcal{G} \left( r \right) \right \rvert} \mathrm{c}_j^{\left( r \right)} \alpha_j^{\left( r \right)} \prod\limits_{k=1}^{j-1} \left( 1 - \alpha_k^{\left( r \right)} \right) =&\\
 &= \mathrm{c}_i^{\left( r \right)} \prod\limits_{j=1}^{i-1} \left( 1 -\alpha_j^{\left( r \right)} \right) + 0 - \sum\limits_{j=i}^{\left \lvert \mathcal{G} \left( r \right) \right \rvert} \mathrm{c}_j^{\left( r \right)} \alpha_j^{\left( r \right)} \prod\limits_{k=1}^{j-1} \left( 1 - \alpha_k^{\left( r \right)} \right) =&\\
 &= \mathrm{c}_i^{\left( r \right)} \prod\limits_{j=1}^{i-1} \left( 1 -\alpha_j^{\left( r \right)} \right) + 0 \cdot \prod\limits_{j=1}^{i-1} \left( 1 - \alpha_j^{\left( r \right)} \right) - \sum\limits_{j=i}^{\left \lvert \mathcal{G} \left( r \right) \right \rvert} \mathrm{c}_j^{\left( r \right)} \alpha_j^{\left( r \right)} \prod\limits_{k=1}^{j-1} \left( 1 - \alpha_k^{\left( r \right)} \right) =&\\
 &= \mathrm{c}_i^{\left( r \right)} \prod\limits_{j=1}^{i-1} \left( 1 -\alpha_j^{\left( r \right)} \right) + \left( \mathrm{c}_{i-1}^{\left( r \right)} - \mathrm{c}_{i-1}^{\left( r \right)} \right) \prod\limits_{j=1}^{i-1} \left( 1 - \alpha_j^{\left( r \right)} \right) - \sum\limits_{j=i}^{\left \lvert \mathcal{G} \left( r \right) \right \rvert} \mathrm{c}_j^{\left( r \right)} \alpha_j^{\left( r \right)} \prod\limits_{k=1}^{j-1} \left( 1 - \alpha_k^{\left( r \right)} \right) =&\\
 &= \left( \mathrm{c}_i^{\left( r \right)} - \mathrm{c}_{i-1}^{\left( r \right)} \right) \prod\limits_{j=1}^{i-1} \left( 1 - \alpha_j^{\left( r \right)} \right) + \mathrm{c}_{i-1}^{\left( r \right)} \prod\limits_{j=1}^{i-1} \left( 1 - \alpha_j^{\left( r \right)} \right) - \sum\limits_{j=i}^{\left \lvert \mathcal{G} \left( r \right) \right \rvert} \mathrm{c}_j^{\left( r \right)} \alpha_j^{\left( r \right)} \prod\limits_{k=1}^{j-1} \left( 1 - \alpha_k^{\left( r \right)} \right) =&\\
 &= \left( \mathrm{c}_i^{\left( r \right)} - \mathrm{c}_{i-1}^{\left( r \right)} \right) \prod\limits_{j=1}^{i-1} \left( 1 - \alpha_j^{\left( r \right)} \right) + \left( 1 - \alpha_{i-1}^{\left( r \right)} \right) \left( \mathrm{c}_{i-1}^{\left( r \right)} \prod \limits_{j=1}^{i-2} \left( 1 - \alpha_j^{\left( r \right)} \right) - \sum\limits_{j=i}^{\left \lvert \mathcal{G} \left( r \right) \right \rvert} \mathrm{c}_j^{\left( r \right)} \alpha_j^{\left( r \right)} \prod\limits_{\substack{k=1 \\ k \ne i-1}}^{j-1} \left( 1 - \alpha_k^{\left( r \right)} \right) \right) =&\\
 &= \left( \mathrm{c}_i^{\left( r \right)} - \mathrm{c}_{i-1}^{\left( r \right)} \right) \prod\limits_{j=1}^{i-1} \left( 1 - \alpha_j^{\left( r \right)} \right) + \left( 1 - \alpha_{i-1}^{\left( r \right)} \right) \dv{\mathrm{\tilde{I}}^{\left( r \right)}}{\alpha_{i-1}^{\left( r \right)}} =&\\
 &= \left( 1 - \alpha_{i-1}^{\left( r \right)} \right) \dv{\mathrm{\tilde{I}}^{\left( r \right)}}{\alpha_{i-1}^{\left( r \right)}} + \left( \mathrm{c}_i^{\left( r \right)} - \mathrm{c}_{i-1}^{\left( r \right)} \right) \prod\limits_{j=1}^{i-1} \left( 1 - \alpha_j^{\left( r \right)} \right) =&\\
 &= \dv{\mathrm{\tilde{I}}^{\left( r \right)}}{\alpha_{i-1}^{\left( r \right)}} \left( 1 - \alpha_{i-1}^{\left( r \right)} \right) + \left( \mathrm{c}_i^{\left( r \right)} - \mathrm{c}_{i-1}^{\left( r \right)} \right) \prod\limits_{j=1}^{i-1} \left( 1 - \alpha_j^{\left( r \right)} \right)
 \end{flalign*}
Hence:
$$
\dv{\mathrm{\tilde{I}}^{\left( r \right)}}{\alpha_i^{\left( r \right)}} = \frac{\dv{\mathrm{\tilde{I}}^{\left( r \right)}}{\alpha_{i-1}^{\left( r \right)}} \left( 1 - \alpha_{i-1}^{\left( r \right)} \right) + \left( \mathrm{c}_i^{\left( r \right)} - \mathrm{c}_{i-1}^{\left( r \right)} \right) \prod\limits_{j=1}^{i-1} \left( 1 - \alpha_j^{\left( r \right)} \right)}{1 - \alpha_i^{\left( r \right)}}
$$
\qedhere
\end{proof}

\begin{lemma}
\label{lemma3}
Let $r \in \mathcal{B}$ and $1 \le i \le \left \lvert \mathcal{G} \left( r \right) \right \rvert$. Then:
$$
\dv{\mathrm{\tilde{I}}^{\left( r \right)}}{\alpha_i^{\left( r \right)}} = \frac{-\mathrm{\tilde{I}}^{\left( r \right)} + \sum\limits_{j=1}^{i-1} \mathrm{c}_j^{\left( r \right)} \alpha_j^{\left( r \right)} \prod\limits_{k=1}^{j-1} \left( 1 - \alpha_k^{\left( r \right)} \right) + \mathrm{c}_i^{\left( r \right)} \prod\limits_{k=1}^{i-1} \left( 1 - \alpha_k^{\left( r \right)} \right)}{1 - \alpha_i^{\left( r \right)}}
$$
\end{lemma}
\begin{proof}
We will prove the lemma by induction. As for the base case it immediately follows from Lemma~\ref{lemma1}. Let us assume then that for some fixed $i \ge 1$ we have:
$$
\dv{\mathrm{\tilde{I}}^{\left( r \right)}}{\alpha_i^{\left( r \right)}} = \frac{-\mathrm{\tilde{I}}^{\left( r \right)} + \sum\limits_{j=1}^{i-1} \mathrm{c}_j^{\left( r \right)} \alpha_j^{\left( r \right)} \prod\limits_{k=1}^{j-1} \left( 1 - \alpha_k^{\left( r \right)} \right) + \mathrm{c}_i^{\left( r \right)} \prod\limits_{k=1}^{i-1} \left( 1 - \alpha_k^{\left( r \right)} \right)}{1 - \alpha_i^{\left( r \right)}}
$$
We will show that:
$$
\dv{\mathrm{\tilde{I}}^{\left( r \right)}}{\alpha_{i+1}^{\left( r \right)}} = \frac{-\mathrm{\tilde{I}}^{\left( r \right)} + \sum\limits_{j=1}^i \mathrm{c}_j^{\left( r \right)} \alpha_j^{\left( r \right)} \prod\limits_{k=1}^{j-1} \left( 1 - \alpha_k^{\left( r \right)} \right) + \mathrm{c}_{i+1}^{\left( r \right)} \prod\limits_{k=1}^i \left( 1 - \alpha_k^{\left( r \right)} \right)}{1 - \alpha_{i+1}^{\left( r \right)}}
$$
Rewriting the RHS yields:
\begin{flalign*}
& \frac{-\mathrm{\tilde{I}}^{\left( r \right)} + \sum\limits_{j=1}^i \mathrm{c}_j^{\left( r \right)} \alpha_j^{\left( r \right)} \prod\limits_{k=1}^{j-1} \left( 1 - \alpha_k^{\left( r \right)} \right) + \mathrm{c}_{i+1}^{\left( r \right)} \prod\limits_{k=1}^i \left( 1 - \alpha_k^{\left( r \right)} \right)}{1 - \alpha_{i+1}^{\left( r \right)}} =&\\
&= \frac{-\mathrm{\tilde{I}}^{\left( r \right)} + \left( \sum\limits_{j=1}^{i-1} \mathrm{c}_j^{\left( r \right)} \alpha_j^{\left( r \right)} \prod\limits_{k=1}^{j-1} \left( 1 - \alpha_k^{\left( r \right)} \right) + \mathrm{c}_i^{\left( r \right)} \alpha_i^{\left( r \right)} \prod\limits_{k=1}^{i-1} \left( 1 - \alpha_k^{\left( r \right)} \right) \right) + \mathrm{c}_{i+1}^{\left( r \right)} \prod\limits_{k=1}^i \left( 1 - \alpha_k^{\left( r \right)} \right) + 0}{1 - \alpha_{i+1}^{\left( r \right)}} =&\\
&= \frac{-\mathrm{\tilde{I}}^{\left( r \right)} + \sum\limits_{j=1}^{i-1} \mathrm{c}_j^{\left( r \right)} \alpha_j^{\left( r \right)} \prod\limits_{k=1}^{j-1} \left( 1 - \alpha_k^{\left( r \right)} \right) + \mathrm{c}_i^{\left( r \right)} \alpha_i^{\left( r \right)} \prod\limits_{k=1}^{i-1} \left( 1 - \alpha_k^{\left( r \right)} \right) + \mathrm{c}_{i+1}^{\left( r \right)} \prod\limits_{k=1}^i \left( 1 - \alpha_k^{\left( r \right)} \right) + 0 \prod\limits_{k=1}^{i-1} \left( 1 - \alpha_k^{\left( r \right)} \right)}{1 - \alpha_{i+1}^{\left( r \right)}} =&\\
&= \frac{-\mathrm{\tilde{I}}^{\left( r \right)} + \sum\limits_{j=1}^{i-1} \mathrm{c}_j^{\left( r \right)} \alpha_j^{\left( r \right)} \prod\limits_{k=1}^{j-1} \left( 1 - \alpha_k^{\left( r \right)} \right) + \mathrm{c}_i^{\left( r \right)} \alpha_i^{\left( r \right)} \prod\limits_{k=1}^{i-1} \left( 1 - \alpha_k^{\left( r \right)} \right) + \mathrm{c}_{i+1}^{\left( r \right)} \prod\limits_{k=1}^i \left( 1 - \alpha_k^{\left( r \right)} \right) + \left( \mathrm{c}_i^{\left( r \right)} - \mathrm{c}_i^{\left( r \right)} \right) \prod\limits_{k=1}^{i-1} \left( 1 - \alpha_k^{\left( r \right)} \right)}{1 - \alpha_{i+1}^{\left( r \right)}} =&\\
&= \frac{-\mathrm{\tilde{I}}^{\left( r \right)} + \sum\limits_{j=1}^{i-1} \mathrm{c}_j^{\left( r \right)} \alpha_j^{\left( r \right)} \prod\limits_{k=1}^{j-1} \left( 1 - \alpha_k^{\left( r \right)} \right) + \mathrm{c}_i^{\left( r \right)} \prod\limits_{k=1}^{i-1} \left( 1 - \alpha_k^{\left( r \right)} \right) + \left( \mathrm{c}_i^{\left( r \right)} \alpha_i^{\left( r \right)} - \mathrm{c}_i^{\left( r \right)} \right) \prod\limits_{k=1}^{i-1} \left( 1 - \alpha_k^{\left( r \right)} \right) + \mathrm{c}_{i+1}^{\left( r \right)} \prod\limits_{k=1}^i \left( 1 - \alpha_k^{\left( r \right)} \right)}{1 - \alpha_{i+1}^{\left( r \right)}} =&\\
&= \frac{-\mathrm{\tilde{I}}^{\left( r \right)} + \sum\limits_{j=1}^{i-1} \mathrm{c}_j^{\left( r \right)} \alpha_j^{\left( r \right)} \prod\limits_{k=1}^{j-1} \left( 1 - \alpha_k^{\left( r \right)} \right) + \mathrm{c}_i^{\left( r \right)} \prod\limits_{k=1}^{i-1} \left( 1 - \alpha_k^{\left( r \right)} \right) - \mathrm{c}_i^{\left( r \right)} \left( 1 - \alpha_i^{\left( r \right)} \right) \prod\limits_{k=1}^{i-1} \left( 1 - \alpha_k^{\left( r \right)} \right) + \mathrm{c}_{i+1}^{\left( r \right)} \left( 1 - \alpha_i^{\left( r \right)} \right) \prod\limits_{k=1}^{i-1} \left( 1 - \alpha_k^{\left( r \right)} \right)}{1 - \alpha_{i+1}^{\left( r \right)}} =&\\
&= \frac{-\mathrm{\tilde{I}}^{\left( r \right)} + \sum\limits_{j=1}^{i-1} \mathrm{c}_j^{\left( r \right)} \alpha_j^{\left( r \right)} \prod\limits_{k=1}^{j-1} \left( 1 - \alpha_k^{\left( r \right)} \right) + \mathrm{c}_i^{\left( r \right)} \prod\limits_{k=1}^{i-1} \left( 1 - \alpha_k^{\left( r \right)} \right) + \left( -\mathrm{c}_i^{\left( r \right)} \left( 1 - \alpha_i^{\left( r \right)} \right) + \mathrm{c}_{i+1}^{\left( r \right)} \left( 1 - \alpha_i^{\left( r \right)} \right) \right) \prod\limits_{k=1}^{i-1} \left( 1 - \alpha_k^{\left( r \right)} \right)}{1 - \alpha_{i+1}^{\left( r \right)}} =&\\
&= \frac{-\mathrm{\tilde{I}}^{\left( r \right)} + \sum\limits_{j=1}^{i-1} \mathrm{c}_j^{\left( r \right)} \alpha_j^{\left( r \right)} \prod\limits_{k=1}^{j-1} \left( 1 - \alpha_k^{\left( r \right)} \right) + \mathrm{c}_i^{\left( r \right)} \prod\limits_{k=1}^{i-1} \left( 1 - \alpha_k^{\left( r \right)} \right)}{1 - \alpha_{i+1}^{\left( r \right)}} + \frac{ \left( \mathrm{c}_{i+1}^{\left( r \right)} - \mathrm{c}_i^{\left( r \right)} \right) \left( 1 - \alpha_i^{\left( r \right)} \right) \prod\limits_{k=1}^{i-1} \left( 1 - \alpha_k^{\left( r \right)} \right)}{1 - \alpha_{i+1}^{\left( r \right)}} =&\\
&= \frac{-\mathrm{\tilde{I}}^{\left( r \right)} + \sum\limits_{j=1}^{i-1} \mathrm{c}_j^{\left( r \right)} \alpha_j^{\left( r \right)} \prod\limits_{k=1}^{j-1} \left( 1 - \alpha_k^{\left( r \right)} \right) + \mathrm{c}_i^{\left( r \right)} \prod\limits_{k=1}^{i-1} \left( 1 - \alpha_k^{\left( r \right)} \right)}{1 - \alpha_{i+1}^{\left( r \right)}} \cdot 1 + \frac{ \left( \mathrm{c}_{i+1}^{\left( r \right)} - \mathrm{c}_i^{\left( r \right)} \right) \prod\limits_{k=1}^i \left( 1 - \alpha_k^{\left( r \right)} \right)}{1 - \alpha_{i+1}^{\left( r \right)}} =&\\
&= \frac{-\mathrm{\tilde{I}}^{\left( r \right)} + \sum\limits_{j=1}^{i-1} \mathrm{c}_j^{\left( r \right)} \alpha_j^{\left( r \right)} \prod\limits_{k=1}^{j-1} \left( 1 - \alpha_k^{\left( r \right)} \right) + \mathrm{c}_i^{\left( r \right)} \prod\limits_{k=1}^{i-1} \left( 1 - \alpha_k^{\left( r \right)} \right)}{1 - \alpha_{i+1}^{\left( r \right)}} \cdot \frac{1 - \alpha_i^{\left( r \right)}}{1 - \alpha_i^{\left( r \right)}} + \frac{ \left( \mathrm{c}_{i+1}^{\left( r \right)} - \mathrm{c}_i^{\left( r \right)} \right) \prod\limits_{k=1}^i \left( 1 - \alpha_k^{\left( r \right)} \right)}{1 - \alpha_{i+1}^{\left( r \right)}} =&\\
&= \frac{-\mathrm{\tilde{I}}^{\left( r \right)} + \sum\limits_{j=1}^{i-1} \mathrm{c}_j^{\left( r \right)} \alpha_j^{\left( r \right)} \prod\limits_{k=1}^{j-1} \left( 1 - \alpha_k^{\left( r \right)} \right) + \mathrm{c}_i^{\left( r \right)} \prod\limits_{k=1}^{i-1} \left( 1 - \alpha_k^{\left( r \right)} \right)}{1 - \alpha_i^{\left( r \right)}} \cdot \frac{1 - \alpha_i^{\left( r \right)}}{1 - \alpha_{i+1}^{\left( r \right)}} + \frac{ \left( \mathrm{c}_{i+1}^{\left( r \right)} - \mathrm{c}_i^{\left( r \right)} \right) \prod\limits_{k=1}^i \left( 1 - \alpha_k^{\left( r \right)} \right)}{1 - \alpha_{i+1}^{\left( r \right)}}
\end{flalign*}
By inductive hypothesis we get:
$$
\frac{-\mathrm{\tilde{I}}^{\left( r \right)} + \sum\limits_{j=1}^i \mathrm{c}_j^{\left( r \right)} \alpha_j^{\left( r \right)} \prod\limits_{k=1}^{j-1} \left( 1 - \alpha_k^{\left( r \right)} \right) + \mathrm{c}_{i+1}^{\left( r \right)} \prod\limits_{k=1}^i \left( 1 - \alpha_k^{\left( r \right)} \right)}{1 - \alpha_i^{\left( r \right)}} = \dv{\mathrm{\tilde{I}}^{\left( r \right)}}{\alpha_i^{\left( r \right)}}
$$
Therefore:
\begin{flalign*}
& \frac{-\mathrm{\tilde{I}}^{\left( r \right)} + \sum\limits_{j=1}^i \mathrm{c}_j^{\left( r \right)} \alpha_j^{\left( r \right)} \prod\limits_{k=1}^{j-1} \left( 1 - \alpha_k^{\left( r \right)} \right) + \mathrm{c}_{i+1}^{\left( r \right)} \prod\limits_{k=1}^i \left( 1 - \alpha_k^{\left( r \right)} \right)}{1 - \alpha_{i+1}^{\left( r \right)}} =&\\
&= \dv{\mathrm{\tilde{I}}^{\left( r \right)}}{\alpha_i^{\left( r \right)}} \cdot \frac{1 - \alpha_i^{\left( r \right)}}{1 - \alpha_{i+1}^{\left( r \right)}} + \frac{ \left( \mathrm{c}_{i+1}^{\left( r \right)} - \mathrm{c}_i^{\left( r \right)} \right) \prod\limits_{k=1}^i \left( 1 - \alpha_k^{\left( r \right)} \right)}{1 - \alpha_{i+1}^{\left( r \right)}} =&\\
&= \frac{ \dv{\mathrm{\tilde{I}}^{\left( r \right)}}{\alpha_i^{\left( r \right)}} \cdot \left( 1 - \alpha_i^{\left( r \right)} \right) + \left( \mathrm{c}_{i+1}^{\left( r \right)} - \mathrm{c}_i^{\left( r \right)} \right) \prod\limits_{k=1}^i \left( 1 - \alpha_k^{\left( r \right)} \right)}{1 - \alpha_{i+1}^{\left( r \right)}}
\end{flalign*} 
Thus by Lemma~\ref{lemma2}:
$$
\frac{-\mathrm{\tilde{I}}^{\left( r \right)} + \sum\limits_{j=1}^i \mathrm{c}_j^{\left( r \right)} \alpha_j^{\left( r \right)} \prod\limits_{k=1}^{j-1} \left( 1 - \alpha_k^{\left( r \right)} \right) + \mathrm{c}_{i+1}^{\left( r \right)} \prod\limits_{k=1}^i \left( 1 - \alpha_k^{\left( r \right)} \right)}{1 - \alpha_{i+1}^{\left( r \right)}} = \dv{\mathrm{\tilde{I}}^{\left( r \right)}}{\alpha_{i+1}^{\left( r \right)}}
$$
Consequently by mathematical induction, the Lemma~\ref{lemma3} holds for all $i \le \left \lvert \mathcal{G} \left( r \right) \right \rvert$.
\qedhere
\end{proof}
\noindent
\newline
Finally, let us note that while computing $\dv{\mathrm{\tilde{I}}^{\left( r \right)}}{\alpha_i^{\left( r \right)}}$ for each $i \in \left \lbrace 1, \ldots, \left \lvert \mathcal{G} \left( r \right) \right \rvert \right \rbrace$ in the naïve manner directly from the gradient formula:
$$
\dv{\mathrm{\tilde{I}}^{\left( r \right)}}{\alpha_i^{\left( r \right)}} =
\mathrm{c}_i^{\left( r \right)} \prod\limits_{j=1}^{i-1} \left( 1 - \alpha_j^{\left( r \right)} \right) - \sum\limits_{j=i+1}^{\left \lvert \mathcal{G} \left( r \right) \right \rvert} \mathrm{c}_j^{\left( r \right)} \alpha_j^{\left( r \right)} \prod\limits_{\substack{k=1 \\ k \ne i }}^{j-1} \left( 1 - \alpha_k^{\left( r \right)} \right)
$$
requires $\mathcal{O}(N)$ operations for single $i \in \left \lbrace 1, \ldots, \left \lvert \mathcal{G} \left( r \right) \right \rvert \right \rbrace$, by employing the above lemmas along with a single prefix sum and a single prefix product, we can perform each step in constant time $\mathcal{O}(1)$, yielding a total of $\mathcal{O}(N)$ steps for the entire ray. Finally, let us note that the formula given by Lemma~\ref{lemma3} is undefined when $\alpha_i^{\left( r \right)} = 1$. However, since in our algorithm we terminate the ray whenever the transmittance $T_i^{\left( r \right)}$ falls below a certain configurable threshold $\varepsilon$, we assume that the $i$-th intersected planar Gaussian primitive for which $T_i^{\left( r \right)} < \varepsilon$ is the last primitive intersected by the ray and apply to it the regular formula above, which does not require division by $1 - \alpha_i^{\left( r \right)}$, thereby protecting us against division by zero.
\noindent
\newline
\paragraph{Computing $\dv{\alpha^{\left( r \right)}_i}{\tilde{\alpha}^{\left( r \right)}_{\text{const},i}}$}
\noindent
\newline
\begin{flalign*}
\dv{\alpha^{\left( r \right)}_i}{\tilde{\alpha}^{\left( r \right)}_{\text{const},i}} &=
\dv{}{\tilde{\alpha}^{\left( r \right)}_{\text{const},i}} \left( \alpha^{\left( r \right)}_{\text{const},i} \cdot \alpha^{\left( r \right)}_{\text{MLP},i} \cdot \exp \left( -\frac{ { \left( \left( u^{ \left( r \right)}_i \right)^2+\left( v^{\left( r \right)}_i \right)^2 \right) }^{\kappa^{\left( r \right)}_i} }{2\kappa^{\left( r \right)}_i} \right) \right) =&\\
&= \alpha^{\left( r \right)}_{\text{MLP},i} \cdot \exp \left( -\frac{ { \left( \left( u^{ \left( r \right)}_i \right)^2+\left( v^{\left( r \right)}_i \right)^2 \right) }^{\kappa^{\left( r \right)}_i} }{2\kappa^{\left( r \right)}_i} \right) \cdot \dv{\alpha^{\left( r \right)}_{\text{const},i}}{\tilde{\alpha}^{\left( r \right)}_{\text{const},i}} =&\\
&= \alpha^{\left( r \right)}_{\text{MLP},i} \cdot \exp \left( -\frac{ { \left( \left( u^{ \left( r \right)}_i \right)^2+\left( v^{\left( r \right)}_i \right)^2 \right) }^{\kappa^{\left( r \right)}_i} }{2\kappa^{\left( r \right)}_i} \right) \cdot \dv{}{\tilde{\alpha}^{\left( r \right)}_{\text{const},i}} \sigma \left( \tilde{\alpha}^{\left( r \right)}_{\text{const},i} \right) =&\\
&= \alpha^{\left( r \right)}_{\text{MLP},i} \cdot \exp \left( -\frac{ { \left( \left( u^{ \left( r \right)}_i \right)^2+\left( v^{\left( r \right)}_i \right)^2 \right) }^{\kappa^{\left( r \right)}_i} }{2\kappa^{\left( r \right)}_i} \right) \cdot \dv{}{\tilde{\alpha}^{\left( r \right)}_{\text{const},i}} \left( \frac{1}{1 + e^{-\tilde{\alpha}^{\left( r \right)}_{\text{const},i}}}\right) =&\\
&= \alpha^{\left( r \right)}_{\text{MLP},i} \cdot \exp \left( -\frac{ { \left( \left( u^{ \left( r \right)}_i \right)^2+\left( v^{\left( r \right)}_i \right)^2 \right) }^{\kappa^{\left( r \right)}_i} }{2\kappa^{\left( r \right)}_i} \right) \cdot \left( -\frac{1}{\left( 1 + e^{-\tilde{\alpha}^{\left( r \right)}_{\text{const},i}} \right)^2} \right) \cdot \left( -e^{-\tilde{\alpha}^{\left( r \right)}_{\text{const},i}} \right) =&\\
&= \alpha^{\left( r \right)}_{\text{MLP},i} \cdot \exp \left( -\frac{ { \left( \left( u^{ \left( r \right)}_i \right)^2+\left( v^{\left( r \right)}_i \right)^2 \right) }^{\kappa^{\left( r \right)}_i} }{2\kappa^{\left( r \right)}_i} \right) \cdot \frac{e^{-\tilde{\alpha}^{\left( r \right)}_{\text{const},i}}}{\left( 1 + e^{-\tilde{\alpha}^{\left( r \right)}_{\text{const},i}} \right)^2} =&\\
&= \alpha^{\left( r \right)}_{\text{MLP},i} \cdot \exp \left( -\frac{ { \left( \left( u^{ \left( r \right)}_i \right)^2+\left( v^{\left( r \right)}_i \right)^2 \right) }^{\kappa^{\left( r \right)}_i} }{2\kappa^{\left( r \right)}_i} \right) \cdot \frac{e^{-\tilde{\alpha}^{\left( r \right)}_{\text{const},i}} + 0}{\left( 1 + e^{-\tilde{\alpha}^{\left( r \right)}_{\text{const},i}} \right)^2} =&\\
&= \alpha^{\left( r \right)}_{\text{MLP},i} \cdot \exp \left( -\frac{ { \left( \left( u^{ \left( r \right)}_i \right)^2+\left( v^{\left( r \right)}_i \right)^2 \right) }^{\kappa^{\left( r \right)}_i} }{2\kappa^{\left( r \right)}_i} \right) \cdot \frac{e^{-\tilde{\alpha}^{\left( r \right)}_{\text{const},i}} + \left( 1 - 1 \right)}{\left( 1 + e^{-\tilde{\alpha}^{\left( r \right)}_{\text{const},i}} \right)^2} =&\\
&= \alpha^{\left( r \right)}_{\text{MLP},i} \cdot \exp \left( -\frac{ { \left( \left( u^{ \left( r \right)}_i \right)^2+\left( v^{\left( r \right)}_i \right)^2 \right) }^{\kappa^{\left( r \right)}_i} }{2\kappa^{\left( r \right)}_i} \right) \cdot \frac{\left( 1 + e^{-\tilde{\alpha}^{\left( r \right)}_{\text{const},i}} \right) - 1}{\left( 1 + e^{-\tilde{\alpha}^{\left( r \right)}_{\text{const},i}} \right)^2} =&\\
&= \alpha^{\left( r \right)}_{\text{MLP},i} \cdot \exp \left( -\frac{ { \left( \left( u^{ \left( r \right)}_i \right)^2+\left( v^{\left( r \right)}_i \right)^2 \right) }^{\kappa^{\left( r \right)}_i} }{2\kappa^{\left( r \right)}_i} \right) \cdot \left( \frac{1}{1 + e^{-\tilde{\alpha}^{\left( r \right)}_{\text{const},i}}} - \frac{1}{\left( 1 + e^{-\tilde{\alpha}^{\left( r \right)}_{\text{const},i}} \right)^2} \right) =&\\
&= \alpha^{\left( r \right)}_{\text{MLP},i} \cdot \exp \left( -\frac{ { \left( \left( u^{ \left( r \right)}_i \right)^2+\left( v^{\left( r \right)}_i \right)^2 \right) }^{\kappa^{\left( r \right)}_i} }{2\kappa^{\left( r \right)}_i} \right) \cdot \frac{1}{1 + e^{-\tilde{\alpha}^{\left( r \right)}_{\text{const},i}}} \cdot \left( 1 - \frac{1}{1 + e^{-\tilde{\alpha}^{\left( r \right)}_{\text{const},i}}} \right) =&\\
&= \alpha^{\left( r \right)}_{\text{MLP},i} \cdot \exp \left( -\frac{ { \left( \left( u^{ \left( r \right)}_i \right)^2+\left( v^{\left( r \right)}_i \right)^2 \right) }^{\kappa^{\left( r \right)}_i} }{2\kappa^{\left( r \right)}_i} \right) \cdot \sigma \left( \tilde{\alpha}^{\left( r \right)}_{\text{const},i} \right) \cdot \left( 1 - \sigma \left( \tilde{\alpha}^{\left( r \right)}_{\text{const},i} \right) \right) =&\\
&= \alpha^{\left( r \right)}_{\text{MLP},i} \cdot \exp \left( -\frac{ { \left( \left( u^{ \left( r \right)}_i \right)^2+\left( v^{\left( r \right)}_i \right)^2 \right) }^{\kappa^{\left( r \right)}_i} }{2\kappa^{\left( r \right)}_i} \right) \cdot \alpha^{\left( r \right)}_{\text{const},i} \cdot \left( 1 - \alpha^{\left( r \right)}_{\text{const},i} \right) =&\\
&= \alpha^{\left( r \right)}_{\text{const},i} \cdot \alpha^{\left( r \right)}_{\text{MLP},i} \cdot \exp \left( -\frac{ { \left( \left( u^{ \left( r \right)}_i \right)^2+\left( v^{\left( r \right)}_i \right)^2 \right) }^{\kappa^{\left( r \right)}_i} }{2\kappa^{\left( r \right)}_i} \right) \cdot \left( 1 - \alpha^{\left( r \right)}_{\text{const},i} \right) =&\\
&= \alpha^{\left( r \right)}_i \cdot \left( 1 - \alpha^{\left( r \right)}_{\text{const},i} \right)
\end{flalign*}

\paragraph{Computing $\dv{\alpha^{\left( r \right)}_i}{\tilde{\kappa}_i^{\left( r \right)}}$}
\noindent
\newline
{
\small
\begin{flalign*}
& \dv{\alpha^{\left( r \right)}_i}{\tilde{\kappa}_i^{\left( r \right)}} =&\\
&= \dv{}{\tilde{\kappa}_i^{\left( r \right)}} \left( \alpha^{\left( r \right)}_{\text{const},i} \cdot \alpha^{\left( r \right)}_{\text{MLP},i} \cdot \exp \left( -\frac{ { \left( \left( u^{ \left( r \right)}_i \right)^2+\left( v^{\left( r \right)}_i \right)^2 \right) }^{\kappa^{\left( r \right)}_i} }{2\kappa^{\left( r \right)}_i} \right) \right) =&\\
&= \alpha^{\left( r \right)}_{\text{const},i} \cdot \alpha^{\left( r \right)}_{\text{MLP},i} \cdot \dv{}{\tilde{\kappa}^{\left( r \right)}_i} \exp \left( -\frac{ { \left( \left( u^{ \left( r \right)}_i \right)^2+\left( v^{\left( r \right)}_i \right)^2 \right) }^{\kappa^{\left( r \right)}_i} }{2\kappa^{\left( r \right)}_i} \right) =&\\
&= \alpha^{\left( r \right)}_{\text{const},i} \cdot \alpha^{\left( r \right)}_{\text{MLP},i} \cdot \exp \left( -\frac{ { \left( \left( u^{ \left( r \right)}_i \right)^2+\left( v^{\left( r \right)}_i \right)^2 \right) }^{\kappa^{\left( r \right)}_i} }{2\kappa^{\left( r \right)}_i} \right) \cdot \dv{}{\tilde{\kappa}^{\left( r \right)}_i} \left( -\frac{ { \left( \left( u^{ \left( r \right)}_i \right)^2+\left( v^{\left( r \right)}_i \right)^2 \right) }^{\kappa^{\left( r \right)}_i} }{2\kappa^{\left( r \right)}_i} \right) =&\\
&= -\alpha^{\left( r \right)}_i \cdot \dv{}{\tilde{\kappa}^{\left( r \right)}_i} \left( \frac{ { \left( \left( u^{ \left( r \right)}_i \right)^2+\left( v^{\left( r \right)}_i \right)^2 \right) }^{\kappa^{\left( r \right)}_i} }{2\kappa^{\left( r \right)}_i} \right) =&\\
&= -\alpha^{\left( r \right)}_i \cdot \left( \frac{ \dv{}{\tilde{\kappa}^{\left( r \right)}_i} \left( { \left( \left( u^{ \left( r \right)}_i \right)^2+\left( v^{\left( r \right)}_i \right)^2 \right) }^{\kappa^{\left( r \right)}_i} \right) \cdot \left( 2\kappa^{\left( r \right)}_i \right) - { \left( \left( u^{ \left( r \right)}_i \right)^2+\left( v^{\left( r \right)}_i \right)^2 \right) }^{\kappa^{\left( r \right)}_i} \cdot \dv{}{\tilde{\kappa}^{\left( r \right)}_i} \left( 2\kappa^{\left( r \right)}_i \right) }{ \left( 2\kappa^{\left( r \right)}_i \right)^2 } \right) =&\\
&= -\alpha^{\left( r \right)}_i \cdot \frac{ \left( { \left( \left( u^{ \left( r \right)}_i \right)^2+\left( v^{\left( r \right)}_i \right)^2 \right) }^{\kappa^{\left( r \right)}_i} \cdot \ln \left( \left( u^{ \left( r \right)}_i \right)^2+\left( v^{\left( r \right)}_i \right)^2 \right) \cdot \dv{\kappa^{\left( r \right)}_i}{\tilde{\kappa}^{\left( r \right)}_i} \right) \cdot \left( 2\kappa^{\left( r \right)}_i \right) - \left( 2 \cdot { \left( \left( u^{ \left( r \right)}_i \right)^2+\left( v^{\left( r \right)}_i \right)^2 \right) }^{\kappa^{\left( r \right)}_i} \cdot \dv{\kappa^{\left( r \right)}_i}{ \tilde{\kappa}^{\left( r \right)}_i} \right) }{ 4 \left( \kappa^{\left( r \right)}_i \right)^2 } =&\\
&= -\alpha^{\left( r \right)}_i \cdot \frac{ 2{\kappa^{\left( r \right)}_i} \cdot { \left( \left( u^{ \left( r \right)}_i \right)^2+\left( v^{\left( r \right)}_i \right)^2 \right) }^{\kappa^{\left( r \right)}_i} \cdot \ln \left( \left( u^{ \left( r \right)}_i \right)^2+\left( v^{\left( r \right)}_i \right)^2 \right) \cdot \dv{\kappa^{\left( r \right)}_i}{\tilde{\kappa}^{\left( r \right)}_i} - 2 \cdot { \left( \left( u^{ \left( r \right)}_i \right)^2+\left( v^{\left( r \right)}_i \right)^2 \right) }^{\kappa^{\left( r \right)}_i} \cdot \dv{\kappa^{\left( r \right)}_i}{ \tilde{\kappa}^{\left( r \right)}_i} }{ 4 \left( \kappa^{\left( r \right)}_i \right)^2 } =&\\
&= -\alpha^{\left( r \right)}_i \cdot \frac{ \kappa^{\left( r \right)}_i \cdot { \left( \left( u^{ \left( r \right)}_i \right)^2+\left( v^{\left( r \right)}_i \right)^2 \right) }^{\kappa^{\left( r \right)}_i} \cdot \ln \left( \left( u^{ \left( r \right)}_i \right)^2+\left( v^{\left( r \right)}_i \right)^2 \right) \cdot \dv{\kappa^{\left( r \right)}_i}{\tilde{\kappa}^{\left( r \right)}_i} - { \left( \left( u^{ \left( r \right)}_i \right)^2+\left( v^{\left( r \right)}_i \right)^2 \right) }^{\kappa^{\left( r \right)}_i} \cdot \dv{\kappa^{\left( r \right)}_i}{ \tilde{\kappa}^{\left( r \right)}_i} }{ 2 \cdot \left( \kappa^{\left( r \right)}_i \right)^2 } =&\\
&= \alpha^{\left( r \right)}_i \cdot \frac{ { \left( \left( u^{ \left( r \right)}_i \right)^2+\left( v^{\left( r \right)}_i \right)^2 \right) }^{\kappa^{\left( r \right)}_i} }{ 2 \kappa^{\left( r \right)}_i } \cdot \left( \frac{1}{ \kappa^{\left( r \right)}_i } - \ln \left( \left( u^{ \left( r \right)}_i \right)^2+\left( v^{\left( r \right)}_i \right)^2 \right) \right) \cdot \dv{\kappa^{\left( r \right)}_i}{ \tilde{\kappa}^{\left( r \right)}_i} =&\\
&= \alpha^{\left( r \right)}_i \cdot \frac{ { \left( \left( u^{ \left( r \right)}_i \right)^2+\left( v^{\left( r \right)}_i \right)^2 \right) }^{\kappa^{\left( r \right)}_i} }{ 2 \kappa^{\left( r \right)}_i } \cdot \left( \frac{1}{ \kappa^{\left( r \right)}_i } - \ln \left( \left( u^{ \left( r \right)}_i \right)^2+\left( v^{\left( r \right)}_i \right)^2 \right) \right) \cdot \dv{}{\tilde{\kappa}^{\left( r \right)}_i} \left( \ln \left( 1 + \exp \left( \tilde{\kappa}^{\left( r \right)}_i \right) \right)\right)=&\\
&= \alpha^{\left( r \right)}_i \cdot \frac{ { \left( \left( u^{ \left( r \right)}_i \right)^2+\left( v^{\left( r \right)}_i \right)^2 \right) }^{\kappa^{\left( r \right)}_i} }{ 2 \kappa^{\left( r \right)}_i } \cdot \left( \frac{1}{ \kappa^{\left( r \right)}_i } - \ln \left( \left( u^{ \left( r \right)}_i \right)^2+\left( v^{\left( r \right)}_i \right)^2 \right) \right) \cdot \frac{\exp \left( \tilde{\kappa}^{\left( r \right)}_i \right)}{1 + \exp \left( \tilde{\kappa}^{\left( r \right)}_i \right)} =&\\
&= \alpha^{\left( r \right)}_i \cdot \frac{ { \left( \left( u^{ \left( r \right)}_i \right)^2+\left( v^{\left( r \right)}_i \right)^2 \right) }^{\kappa^{\left( r \right)}_i} }{ 2 \kappa^{\left( r \right)}_i } \cdot \left( \frac{1}{ \kappa^{\left( r \right)}_i } - \ln \left( \left( u^{ \left( r \right)}_i \right)^2+\left( v^{\left( r \right)}_i \right)^2 \right) \right) \cdot \frac{\exp \left( \tilde{\kappa}^{\left( r \right)}_i \right)}{1 + \exp \left( \tilde{\kappa}^{\left( r \right)}_i \right)} \cdot 1 =&\\
&= \alpha^{\left( r \right)}_i \cdot \frac{ { \left( \left( u^{ \left( r \right)}_i \right)^2+\left( v^{\left( r \right)}_i \right)^2 \right) }^{\kappa^{\left( r \right)}_i} }{ 2 \kappa^{\left( r \right)}_i } \cdot \left( \frac{1}{ \kappa^{\left( r \right)}_i } - \ln \left( \left( u^{ \left( r \right)}_i \right)^2+\left( v^{\left( r \right)}_i \right)^2 \right) \right) \cdot \frac{\exp \left( \tilde{\kappa}^{\left( r \right)}_i \right)}{1 + \exp \left( \tilde{\kappa}^{\left( r \right)}_i \right)} \cdot \frac{ \frac{1}{\exp \left( \tilde{\kappa}^{\left( r \right)}_i \right)} }{ \frac{1}{\exp \left( \tilde{\kappa}^{\left( r \right)}_i \right)} } =&\\
&= \alpha^{\left( r \right)}_i \cdot \frac{ { \left( \left( u^{ \left( r \right)}_i \right)^2+\left( v^{\left( r \right)}_i \right)^2 \right) }^{\kappa^{\left( r \right)}_i} }{ 2 \kappa^{\left( r \right)}_i } \cdot \left( \frac{1}{ \kappa^{\left( r \right)}_i } - \ln \left( \left( u^{ \left( r \right)}_i \right)^2+\left( v^{\left( r \right)}_i \right)^2 \right) \right) \cdot \frac{1}{ \frac{1}{\exp \left( \tilde{\kappa}^{\left( r \right)}_i \right)} + 1 } =&\\
&= \alpha^{\left( r \right)}_i \cdot \frac{ { \left( \left( u^{ \left( r \right)}_i \right)^2+\left( v^{\left( r \right)}_i \right)^2 \right) }^{\kappa^{\left( r \right)}_i} }{ 2 \kappa^{\left( r \right)}_i } \cdot \left( \frac{1}{ \kappa^{\left( r \right)}_i } - \ln \left( \left( u^{ \left( r \right)}_i \right)^2+\left( v^{\left( r \right)}_i \right)^2 \right) \right) \cdot \frac{1}{ \exp \left( -\tilde{\kappa}^{\left( r \right)}_i \right) + 1 } =&\\
&= \alpha^{\left( r \right)}_i \cdot \frac{ { \left( \left( u^{ \left( r \right)}_i \right)^2+\left( v^{\left( r \right)}_i \right)^2 \right) }^{\kappa^{\left( r \right)}_i} }{ 2 \kappa^{\left( r \right)}_i } \cdot \left( \frac{1}{ \kappa^{\left( r \right)}_i } - \ln \left( \left( u^{ \left( r \right)}_i \right)^2+\left( v^{\left( r \right)}_i \right)^2 \right) \right) \cdot \frac{1}{ 1 + \exp \left( -\tilde{\kappa}^{\left( r \right)}_i \right) } =&\\
&= \alpha^{\left( r \right)}_i \cdot \frac{ { \left( \left( u^{ \left( r \right)}_i \right)^2+\left( v^{\left( r \right)}_i \right)^2 \right) }^{\kappa^{\left( r \right)}_i} }{ 2 \kappa^{\left( r \right)}_i } \cdot \left( \frac{1}{ \kappa^{\left( r \right)}_i } - \ln \left( \left( u^{ \left( r \right)}_i \right)^2+\left( v^{\left( r \right)}_i \right)^2 \right) \right) \cdot \sigma \left( \tilde{\kappa}^{\left( r \right)}_i \right) =&\\
\end{flalign*}
}

\subsection{Computing MLP backpropagation gradients}
\noindent
\newline
As detailed in the Subsection \textbf{Global auto-decoder}:
$$
\Phi \left( \left( u^{\left( r \right)}_i, v^{\left( r \right)}_i \right), \mathrm{\vec{d}}^{\left( r \right)}, \mathrm{z}^{\left( r \right)}_i \;; \mathcal{G}, \mathcal{F}, \Theta \right) =
\mathrm{W}_3 \varphi_3 \left( \mathrm{W}_2 \varphi_2 \left( \mathrm{W}_1 \varphi_1 \left(
\begin{bmatrix}
E_{\text{LUT}} \left( u^{\left( r \right)}_i, v^{\left( r \right)}_i \right) \\
\gamma \left( \mathrm{\vec{d}}^{\left( r \right)} \right) \\
\mathrm{z}^{\left( r \right)}_i
\end{bmatrix}
\right) + \mathrm{b}_1 \right) + \mathrm{b}_2 \right) + \mathrm{b}_3
$$
For brevity, we will omit both the superscript $(r)$ (where $r \in \mathcal{B}$ specifies the particular ray from the batch $\mathcal{B}$) and the subscript $i$ (denoting the index of the Gaussian primitive hit by ray $r$). Let us put:
$$
\mathrm{x}_0 \coloneqq \begin{bmatrix}
E_{\text{LUT}} \left( u, v \right) \\
\gamma \left( \mathrm{\vec{d}} \right) \\
\mathrm{z}
\end{bmatrix}
$$
$$
\mathrm{h}_1 \coloneqq \mathrm{W}_1 \varphi_1 \left( \mathrm{x}_0 \right) + \mathrm{b}_1, \; \mathrm{x}_1 \coloneqq \varphi_1 \left( \mathrm{h}_1 \right)
$$
$$
\mathrm{h}_2 \coloneqq \mathrm{W}_2 \varphi_2 \left( \mathrm{x}_1 \right) + \mathrm{b}_2, \; \mathrm{x}_2 \coloneqq \varphi_2 \left( \mathrm{h}_2 \right)
$$
$$
\mathrm{h}_3 \coloneqq \mathrm{W}_3 \varphi_3 \left( \mathrm{x}_2 \right) + \mathrm{b}_3, \; \mathrm{x}_3 \coloneqq \varphi_3 \left( \mathrm{h}_3 \right)
$$
and:
$$
\delta_3 \coloneqq \mathrm{W}_3^{\mathrm{T}} \left( \delta_4 \odot \varphi'_3 \left( \mathrm{h}_3 \right) \right)
$$
$$
\delta_2 \coloneqq \mathrm{W}_2^{\mathrm{T}} \left( \delta_3 \odot \varphi'_2 \left( \mathrm{h}_2 \right) \right)
$$
$$
\delta_1 \coloneqq \mathrm{W}_1^{\mathrm{T}} \left( \delta_2 \odot \varphi'_1 \left( \mathrm{h}_1 \right) \right)
$$
Before deriving formulas for the inner products: $\left \langle \delta_4, \dv{\Phi}{\mathrm{W}_3} \right \rangle$, $\left \langle \delta_4, \dv{\Phi}{\mathrm{b}_3} \right \rangle$, $\left \langle \delta_4, \dv{\Phi}{\mathrm{W}_2} \right \rangle$, $\left \langle \delta_4, \dv{\Phi}{\mathrm{b}_2} \right \rangle$, $\left \langle \delta_4, \dv{\Phi}{\mathrm{W}_1} \right \rangle$, $\left \langle \delta_4, \dv{\Phi}{\mathrm{b}_1} \right \rangle$, $\left \langle \delta_4, \dv{\Phi}{\mathrm{z}} \right \rangle$, $\left \langle \delta_4, \dv{\Phi}{\mathcal{F}} \right \rangle$ that define the gradients of the loss function $\mathcal{L}$ with respect the parameters: $\mathrm{W}_3$, $\mathrm{b}_3$, $\mathrm{W}_2$, $\mathrm{b}_2$, $\mathrm{W}_1$, $\mathrm{b}_1$, $\mathrm{z}$, $\mathcal{F}$ we will show how to compute the auxiliary identities: $\left \langle \delta_4, \dv{\Phi}{\mathrm{x}_3} \right \rangle$, $\left \langle \delta_4, \dv{\Phi}{\mathrm{x}_2} \right \rangle$, $\left \langle \delta_4, \dv{\Phi}{\mathrm{x}_1} \right \rangle$ and $\left \langle \delta_4, \dv{\Phi}{\mathrm{x}_0} \right \rangle$. Before doing so, however, we will state and prove a lemma that, as we shall see later, will prove useful in deriving the relevant formulas.
\newline

\begin{lemma}
\label{lemma4}
Let $\mathrm{a}, \mathrm{b}, \mathrm{c} \in \mathbb{R}^n$. Then the mixed Hadamard product: $\left \langle \mathrm{a}, \mathrm{b} \odot \mathrm{c} \right \rangle$ is invariant under any permutation of the arguments: $\mathrm{a}$, $\mathrm{b}$, and $\mathrm{c}$.
\end{lemma}
\begin{proof}
\begin{flalign*}
\left \langle \mathrm{a}, \mathrm{b} \odot \mathrm{c} \right \rangle &=
\sum_{i=1}^n \mathrm{a}_i \left( \mathrm{b} \odot \mathrm{c} \right)_i =
\sum_{i=1}^n \mathrm{a}_i \left( \mathrm{b}_i \cdot \mathrm{c}_i \right) =
\sum_{i=1}^n \mathrm{a}_i \cdot \mathrm{b}_i \cdot \mathrm{c}_i =
\sum_{i=1}^n \mathrm{b}_i \cdot \mathrm{c}_i \cdot \mathrm{a}_i =
\sum_{i=1}^n \mathrm{b}_i \cdot \left( \mathrm{c}_i \cdot \mathrm{a}_i \right) =&\\
&= \sum_{i=1}^n \mathrm{b}_i \cdot \left( \mathrm{c} \odot \mathrm{a} \right)_i =
\left \langle \mathrm{b}, \mathrm{c} \odot \mathrm{a} \right \rangle
\end{flalign*}
On the other hand:
\begin{flalign*}
\left \langle \mathrm{a}, \mathrm{b} \odot \mathrm{c} \right \rangle &=
\sum_{i=1}^n \mathrm{a}_i \left( \mathrm{b} \odot \mathrm{c} \right)_i =
\sum_{i=1}^n \mathrm{a}_i \left( \mathrm{b}_i \cdot \mathrm{c}_i \right) =
\sum_{i=1}^n \mathrm{a}_i \cdot \mathrm{b}_i \cdot \mathrm{c}_i =
\sum_{i=1}^n \mathrm{c}_i \cdot \mathrm{a}_i \cdot \mathrm{b}_i =
\sum_{i=1}^n \mathrm{c}_i \cdot \left( \mathrm{a}_i \cdot \mathrm{b}_i \right) =&\\
&= \sum_{i=1}^n \mathrm{c}_i \cdot \left( \mathrm{a} \odot \mathrm{b} \right)_i =
\left \langle \mathrm{c}, \mathrm{a} \odot \mathrm{b} \right \rangle
\end{flalign*}
The proof of invariance under the remaining three possible permutations is trivial and follows immediately from the commutativity of the Hadamard product.
\qedhere
\end{proof}

\paragraph{Computing $\left \langle \delta_4, \dv{\Phi}{\mathrm{x}_3} \right \rangle$}
\begin{flalign*}
& \left \langle \delta_4, \dv{\Phi}{\mathrm{x}_{3,i}} \right \rangle =
\left \langle \delta_4, \dv{\mathrm{x}_3}{\mathrm{x}_{3,i}} \right \rangle =
\left \langle \delta_4, \vec{\upvarepsilon}_i \right \rangle =
\delta_{4,i} &
\end{flalign*}
Hence:
\begin{flalign*}
& \left \langle \delta_4, \dv{\Phi}{\mathrm{x}_3} \right \rangle =
\delta_4 &
\end{flalign*}

\paragraph{Computing $\left \langle \delta_4, \dv{\Phi}{\mathrm{x}_2} \right \rangle$}
\begin{flalign*}
\left \langle \delta_4, \dv{\Phi}{\mathrm{x}_{2,i}} \right \rangle &=
\left \langle \delta_4, \dv{\mathrm{x}_3}{\mathrm{x}_{2,i}} \right \rangle =&\\
&= \left \langle \delta_4, \dv{}{\mathrm{x}_{2,i}} \left( \varphi_3 \left( \mathrm{h}_3 \right) \right) \right \rangle =&\\
&= \left \langle \delta_4, \varphi'_3 \left( \mathrm{h}_3 \right) \odot \dv{\mathrm{h}_3}{\mathrm{x}_{2,i}} \right \rangle =&\\
&= \left \langle \delta_4, \varphi'_3 \left( \mathrm{h}_3 \right) \odot \dv{}{\mathrm{x}_{2,i}} \left( \mathrm{W}_3 \mathrm{x}_2 + \mathrm{b}_3 \right) \right \rangle =&\\
&= \left \langle \delta_4, \varphi'_3 \left( \mathrm{h}_3 \right) \odot \left( \mathrm{W}_3 \dv{\mathrm{x}_2}{\mathrm{x}_{2,i}} \right) \right \rangle =&\\
&= \left \langle \delta_4, \varphi'_3 \left( \mathrm{h}_3 \right) \odot \left( \mathrm{W}_3 \vec{\upvarepsilon}_i \right) \right \rangle &
\end{flalign*}
By Lemma~\ref{lemma4} we get:
\begin{flalign*}
\left \langle \delta_4, \dv{\Phi}{\mathrm{x}_{2,i}} \right \rangle &=
\left \langle \delta_4 \odot \varphi'_3 \left( \mathrm{h}_3 \right), \mathrm{W}_3 \vec{\upvarepsilon}_i \right \rangle =&\\
&= \left( \mathrm{W}_3 \vec{\upvarepsilon}_i \right)^{\mathrm{T}} \left( \delta_4 \odot \varphi'_3 \left( \mathrm{h}_3 \right) \right) =&\\
&= \vec{\upvarepsilon}_i^{\mathrm{T}} \mathrm{W}_3^{\mathrm{T}} \left( \delta_4 \odot \varphi'_3 \left( \mathrm{h}_3 \right) \right) =&\\
&= \vec{\upvarepsilon}_i^{\mathrm{T}} \delta_3 =&\\
&= \delta_{3,i}
\end{flalign*}
Hence:
\begin{flalign*}
& \left \langle \delta_4, \dv{\Phi}{\mathrm{x}_2} \right \rangle =
\delta_3 &
\end{flalign*}

\paragraph{Computing $\left \langle \delta_4, \dv{\Phi}{\mathrm{x}_1} \right \rangle$}
\begin{flalign*}
\left \langle \delta_4, \dv{\Phi}{\mathrm{x}_{1,j}} \right \rangle &=
\left \langle \delta_4, \sum_{i=1}^{64} \dv{\Phi}{\mathrm{x}_{2,i}} \cdot \dv{\mathrm{x}_{2,i}}{\mathrm{x}_{1,j}} \right \rangle =&\\
&= \sum_{i=1}^{64} \left \langle \delta_4, \dv{\Phi}{\mathrm{x}_{2,i}} \cdot \dv{\mathrm{x}_{2,i}}{\mathrm{x}_{1,j}} \right \rangle =&\\
&= \sum_{i=1}^{64} \left \langle \delta_4, \dv{\Phi}{\mathrm{x}_{2,i}} \right \rangle \cdot \dv{\mathrm{x}_{2,i}}{\mathrm{x}_{1,j}} =&\\
&= \sum_{i=1}^{64} \delta_{3,i} \cdot \dv{\mathrm{x}_{2,i}}{\mathrm{x}_{1,j}} =&\\
&= \sum_{i=1}^{64} \delta_{3,i} \cdot \dv{}{\mathrm{x}_{1,j}} { \left( \varphi_2 \left( \mathrm{h}_2 \right) \right) }_i =&\\
&= \sum_{i=1}^{64} \delta_{3,i} \cdot \left( \varphi'_2 \left( \mathrm{h}_2 \right) \odot \dv{\mathrm{h}_2}{\mathrm{x}_{1,j}} \right)_i =&\\
&= \sum_{i=1}^{64} \delta_{3,i} \cdot \left( \varphi'_2 \left( \mathrm{h}_2 \right) \odot \dv{}{\mathrm{x}_{1,j}} \left( \mathrm{W}_2 \mathrm{x}_1 + \mathrm{b}_2 \right) \right)_i =&\\
&= \sum_{i=1}^{64} \delta_{3,i} \cdot \left( \varphi'_2 \left( \mathrm{h}_2 \right) \odot \left( \mathrm{W}_2 \dv{\mathrm{x}_1}{\mathrm{x}_{1,j}} \right) \right)_i =&\\
&= \sum_{i=1}^{64} \delta_{3,i} \cdot \left( \varphi'_2 \left( \mathrm{h}_2 \right) \odot \left( \mathrm{W}_2 \vec{\upvarepsilon}_j \right) \right)_i =&\\
&= \left \langle \delta_3, \varphi'_2 \left( \mathrm{h}_2 \right) \odot \left( \mathrm{W}_2 \vec{\upvarepsilon}_j \right) \right \rangle
\end{flalign*}
By Lemma~\ref{lemma4} we get:
\begin{flalign*}
\left \langle \delta_4, \dv{\Phi}{\mathrm{x}_{1,j}} \right \rangle &=
\left \langle \delta_3 \odot \varphi'_2 \left( \mathrm{h}_2 \right) , \mathrm{W}_2 \vec{\upvarepsilon}_j \right \rangle =&\\
&= \left( \mathrm{W}_2 \vec{\upvarepsilon}_j \right)^{\mathrm{T}} \left( \delta_3 \odot \varphi'_2 \left( \mathrm{h}_2 \right) \right) =&\\
&= \vec{\upvarepsilon}_j^{\mathrm{T}} \mathrm{W}_2^{\mathrm{T}} \left( \delta_3 \odot \varphi'_2 \left( \mathrm{h}_2 \right) \right) =&\\
&= \vec{\upvarepsilon}_j^{\mathrm{T}} \delta_2 =&\\
&= \delta_{2,j}
\end{flalign*}
Hence:
\begin{flalign*}
& \left \langle \delta_4, \dv{\Phi}{\mathrm{x}_1} \right \rangle =
\delta_2 &
\end{flalign*}

\paragraph{Computing $\left \langle \delta_4, \dv{\Phi}{\mathrm{x}_0} \right \rangle$}
\begin{flalign*}
\left \langle \delta_4, \dv{\Phi}{\mathrm{x}_{0,j}} \right \rangle &=
\left \langle \delta_4, \sum_{i=1}^{64} \dv{\Phi}{\mathrm{x}_{1,i}} \cdot \dv{\mathrm{x}_{1,i}}{\mathrm{x}_{0,j}} \right \rangle =&\\
&= \sum_{i=1}^{64} \left \langle \delta_4, \dv{\Phi}{\mathrm{x}_{1,i}} \cdot \dv{\mathrm{x}_{1,i}}{\mathrm{x}_{0,j}} \right \rangle =&\\
&= \sum_{i=1}^{64} \left \langle \delta_4, \dv{\Phi}{\mathrm{x}_{1,i}} \right \rangle \cdot \dv{\mathrm{x}_{1,i}}{\mathrm{x}_{0,j}} =&\\
&= \sum_{i=1}^{64} \delta_{2,i} \cdot \dv{\mathrm{x}_{1,i}}{\mathrm{x}_{0,j}} =&\\
&= \sum_{i=1}^{64} \delta_{2,i} \cdot \dv{}{\mathrm{x}_{0,j}} { \left( \varphi_1 \left( \mathrm{h}_1 \right) \right) }_i =&\\
&= \sum_{i=1}^{64} \delta_{2,i} \cdot \left( \varphi'_1 \left( \mathrm{h}_1 \right) \odot \dv{\mathrm{h}_1}{\mathrm{x}_{0,j}} \right)_i =&\\
&= \sum_{i=1}^{64} \delta_{2,i} \cdot \left( \varphi'_1 \left( \mathrm{h}_1 \right) \odot \dv{}{\mathrm{x}_{0,j}} \left( \mathrm{W}_1 \mathrm{x}_0 + \mathrm{b}_1 \right) \right)_i =&\\
&= \sum_{i=1}^{64} \delta_{2,i} \cdot \left( \varphi'_1 \left( \mathrm{h}_1 \right) \odot \left( \mathrm{W}_1 \dv{\mathrm{x}_0}{\mathrm{x}_{0,j}} \right) \right)_i =&\\
&= \sum_{i=1}^{64} \delta_{2,i} \cdot \left( \varphi'_1 \left( \mathrm{h}_1 \right) \odot \left( \mathrm{W}_1 \vec{\upvarepsilon}_j \right) \right)_i =&\\
&= \left \langle \delta_2 , \varphi'_1 \left( \mathrm{h}_1 \right) \odot \left( \mathrm{W}_1 \vec{\upvarepsilon}_j \right) \right \rangle
\end{flalign*}
By Lemma~\ref{lemma4} we get:
\begin{flalign*}
\left \langle \delta_4, \dv{\Phi}{\mathrm{x}_{0,j}} \right \rangle &=
\left \langle \delta_2 \odot \varphi'_1 \left( \mathrm{h}_1 \right) , \mathrm{W}_1 \vec{\upvarepsilon}_j \right \rangle =&\\
&= \left( \mathrm{W}_1 \vec{\upvarepsilon}_j \right)^{\mathrm{T}} \left( \delta_2 \odot \varphi'_1 \left( \mathrm{h}_1 \right) \right) =&\\
&= \vec{\upvarepsilon}_j^{\mathrm{T}} \mathrm{W}_1^{\mathrm{T}} \left( \delta_2 \odot \varphi'_1 \left( \mathrm{h}_1 \right) \right) =&\\
&= \vec{\upvarepsilon}_j^{\mathrm{T}} \delta_1 =&\\
&= \delta_{1,j}
\end{flalign*}
Hence:
\begin{flalign*}
& \left \langle \delta_4, \dv{\Phi}{\mathrm{x}_0} \right \rangle =
\delta_1 &
\end{flalign*}
\newline
Now, we are ready to evaluate the inner products:
$\left \langle \delta_4, \dv{\Phi}{\mathrm{W}_3} \right \rangle$, $\left \langle \delta_4, \dv{\Phi}{\mathrm{b}_3} \right \rangle$, $\left \langle \delta_4, \dv{\Phi}{\mathrm{W}_2} \right \rangle$, $\left \langle \delta_4, \dv{\Phi}{\mathrm{b}_2} \right \rangle$, $\left \langle \delta_4, \dv{\Phi}{\mathrm{W}_1} \right \rangle$, $\left \langle \delta_4, \dv{\Phi}{\mathrm{b}_1} \right \rangle$, $\left \langle \delta_4, \dv{\Phi}{\mathrm{z}} \right \rangle$, $\left \langle \delta_4, \dv{\Phi}{\mathrm{f}} \right \rangle$.

\paragraph{Computing $\left \langle \delta_4, \dv{\Phi}{\mathrm{W}_3} \right \rangle$}
\begin{flalign*}
\left \langle \delta_4, \dv{\Phi}{\mathrm{W}_{3,j,k}} \right \rangle &=
\left \langle \delta_4, \sum_{i=1}^{4} \dv{\Phi}{\mathrm{x}_{3,i}} \cdot \dv{\mathrm{x}_{3,i}}{\mathrm{W}_{3,j,k}} \right \rangle =&\\
&= \sum_{i=1}^{4} \left \langle \delta_4, \dv{\Phi}{\mathrm{x}_{3,i}} \cdot \dv{\mathrm{x}_{3,i}}{\mathrm{W}_{3,j,k}} \right \rangle =&\\
&= \sum_{i=1}^{4} \left \langle \delta_4, \dv{\Phi}{\mathrm{x}_{3,i}} \right \rangle \cdot \dv{\mathrm{x}_{3,i}}{\mathrm{W}_{3,j,k}} =&\\
&= \sum_{i=1}^{4} \delta_{4,i} \cdot \dv{\mathrm{x}_{3,i}}{\mathrm{W}_{3,j,k}} =&\\
&= \sum_{i=1}^{4} \delta_{4,i} \cdot \dv{}{\mathrm{W}_{3,j,k}} { \left( \varphi_3 \left( \mathrm{h}_3 \right) \right) }_i =&\\
&= \sum_{i=1}^{4} \delta_{4,i} \cdot \left( \varphi'_3 \left( \mathrm{h}_3 \right) \odot \dv{\mathrm{h}_3}{\mathrm{W}_{3,j,k}} \right)_i =&\\
&= \sum_{i=1}^{4} \delta_{4,i} \cdot \left( \varphi'_3 \left( \mathrm{h}_3 \right) \odot \dv{}{\mathrm{W}_{3,j,k}} \left( \mathrm{W}_3 \mathrm{x}_2 + \mathrm{b}_3 \right) \right)_i =&\\
&= \sum_{i=1}^{4} \delta_{4,i} \cdot \left( \varphi'_3 \left( \mathrm{h}_3 \right) \odot \left( \dv{\mathrm{W}_3}{\mathrm{W}_{3,j,k}} \mathrm{x}_2 \right) \right)_i =&\\
&= \sum_{i=1}^{4} \delta_{4,i} \cdot \left( \varphi'_3 \left( \mathrm{h}_3 \right) \odot \left( \vec{\upvarepsilon}_j \vec{\upvarepsilon}_k^{\mathrm{T}} \mathrm{x}_2 \right) \right)_i =&\\
&= \left \langle \delta_4 , \varphi'_3 \left( \mathrm{h}_3 \right) \odot \left( \vec{\upvarepsilon}_j \vec{\upvarepsilon}_k^{\mathrm{T}} \mathrm{x}_2 \right) \right \rangle =&\\
&= \left \langle \delta_4 , \varphi'_3 \left( \mathrm{h}_3 \right) \odot \left( \vec{\upvarepsilon}_j \mathrm{x}_{2,k} \right) \right \rangle
\end{flalign*}
By Lemma~\ref{lemma4} we get:
\begin{flalign*}
\left \langle \delta_4, \dv{\Phi}{\mathrm{W}_{3,j,k}} \right \rangle &=
\left \langle \delta_4 \odot \varphi'_3 \left( \mathrm{h}_3 \right) , \vec{\upvarepsilon}_j \mathrm{x}_{2,k} \right \rangle =&\\
&= \left( \vec{\upvarepsilon}_j \mathrm{x}_{2,k} \right)^{\mathrm{T}} \left( \delta_4 \odot \varphi'_3 \left( \mathrm{h}_3 \right) \right) =&\\
&= \mathrm{x}_{2,k} \vec{\upvarepsilon}_j^{\mathrm{T}} \left( \delta_4 \odot \varphi'_3 \left( \mathrm{h}_3 \right) \right) =&\\
&= \mathrm{x}_{2,k} \left( \delta_4 \odot \varphi'_3 \left( \mathrm{h}_3 \right) \right)_j =&\\
&= \left( \delta_4 \odot \varphi'_3 \left( \mathrm{h}_3 \right) \right)_j \mathrm{x}_{2,k}
\end{flalign*}
Hence:
\begin{flalign*}
& \left \langle \delta_4, \dv{\Phi}{\mathrm{W}_3} \right \rangle =
\left( \delta_4 \odot \varphi'_3 \left( \mathrm{h}_3 \right) \right) \mathrm{x}_2^{\mathrm{T}} &
\end{flalign*}

\paragraph{Computing $\left \langle \delta_4, \dv{\Phi}{\mathrm{b}_3} \right \rangle$}
\begin{flalign*}
\left \langle \delta_4, \dv{\Phi}{\mathrm{b}_{3,j}} \right \rangle &=
\left \langle \delta_4, \sum_{i=1}^{4} \dv{\Phi}{\mathrm{x}_{3,i}} \cdot \dv{\mathrm{x}_{3,i}}{\mathrm{b}_{3,j}} \right \rangle =&\\
&= \sum_{i=1}^{4} \left \langle \delta_4, \dv{\Phi}{\mathrm{x}_{3,i}} \cdot \dv{\mathrm{x}_{3,i}}{\mathrm{b}_{3,j}} \right \rangle =&\\
&= \sum_{i=1}^{4} \left \langle \delta_4, \dv{\Phi}{\mathrm{x}_{3,i}} \right \rangle \cdot \dv{\mathrm{x}_{3,i}}{\mathrm{b}_{3,j}} =&\\
&= \sum_{i=1}^{4} \delta_{4,i} \cdot \dv{\mathrm{x}_{3,i}}{\mathrm{b}_{3,j}} =&\\
&= \sum_{i=1}^{4} \delta_{4,i} \cdot \dv{}{\mathrm{b}_{3,j}} { \left( \varphi_3 \left( \mathrm{h}_3 \right) \right) }_i =&\\
&= \sum_{i=1}^{4} \delta_{4,i} \cdot \left( \varphi'_3 \left( \mathrm{h}_3 \right) \dv{\mathrm{h}_3}{\mathrm{b}_{3,j}} \right)_i =&\\
&= \sum_{i=1}^{4} \delta_{4,i} \cdot \left( \varphi'_3 \left( \mathrm{h}_3 \right) \dv{}{\mathrm{b}_{3,j}} \left( \mathrm{W}_3 \mathrm{x}_2 + \mathrm{b}_3 \right) \right)_i =&\\
&= \sum_{i=1}^{4} \delta_{4,i} \cdot \left( \varphi'_3 \left( \mathrm{h}_3 \right) \dv{\mathrm{b}_3}{\mathrm{b}_{3,j}} \right)_i =&\\
&= \sum_{i=1}^{4} \delta_{4,i} \cdot \left( \varphi'_3 \left( \mathrm{h}_3 \right) \vec{\upvarepsilon}_j \right)_i =&\\
&= \sum_{i=1}^{4} \delta_{4,i} \cdot \varphi'_{3,i} \left( \mathrm{h}_3 \right) \cdot \delta_{i{j}} =&\\
&= \delta_{4,j} \cdot \varphi'_{3,j} \left( \mathrm{h}_3 \right)
\end{flalign*}
Hence:
\begin{flalign*}
& \left \langle \delta_4, \dv{\Phi}{\mathrm{b}_3} \right \rangle =
\delta_4 \cdot \varphi'_3 \left( \mathrm{h}_3 \right) &
\end{flalign*}

\paragraph{Computing $\left \langle \delta_4, \dv{\Phi}{\mathrm{W}_2} \right \rangle$}
\begin{flalign*}
\left \langle \delta_4, \dv{\Phi}{\mathrm{W}_{2,j,k}} \right \rangle &=
\left \langle \delta_4, \sum_{i=1}^{64} \dv{\Phi}{\mathrm{x}_{2,i}} \cdot \dv{\mathrm{x}_{2,i}}{\mathrm{W}_{2,j,k}} \right \rangle =&\\
&= \sum_{i=1}^{64} \left \langle \delta_4, \dv{\Phi}{\mathrm{x}_{2,i}} \cdot \dv{\mathrm{x}_{2,i}}{\mathrm{W}_{2,j,k}} \right \rangle =&\\
&= \sum_{i=1}^{64} \left \langle \delta_4, \dv{\Phi}{\mathrm{x}_{2,i}} \right \rangle \cdot \dv{\mathrm{x}_{2,i}}{\mathrm{W}_{2,j,k}} =&\\
&= \sum_{i=1}^{64} \delta_{3,i} \cdot \dv{\mathrm{x}_{2,i}}{\mathrm{W}_{2,j,k}} =&\\
&= \sum_{i=1}^{64} \delta_{3,i} \cdot \dv{}{\mathrm{W}_{2,j,k}} { \left( \varphi_2 \left( \mathrm{h}_2 \right) \right) }_i =&\\
&= \sum_{i=1}^{64} \delta_{3,i} \cdot \left( \varphi'_2 \left( \mathrm{h}_2 \right) \odot \dv{\mathrm{h}_2}{\mathrm{W}_{2,j,k}} \right)_i =&\\
&= \sum_{i=1}^{64} \delta_{3,i} \cdot \left( \varphi'_2 \left( \mathrm{h}_2 \right) \odot \dv{}{\mathrm{W}_{2,j,k}} \left( \mathrm{W}_2 \mathrm{x}_1 + \mathrm{b}_2 \right) \right)_i =&\\
&= \sum_{i=1}^{64} \delta_{3,i} \cdot \left( \varphi'_2 \left( \mathrm{h}_2 \right) \odot \left( \dv{\mathrm{W}_2}{\mathrm{W}_{2,j,k}} \mathrm{x}_1 \right) \right)_i =&\\
&= \sum_{i=1}^{64} \delta_{3,i} \cdot \left( \varphi'_2 \left( \mathrm{h}_2 \right) \odot \left( \vec{\upvarepsilon}_j \vec{\upvarepsilon}_k^{\mathrm{T}} \mathrm{x}_1 \right) \right)_i =&\\
&= \left \langle \delta_3 , \varphi'_2 \left( \mathrm{h}_2 \right) \odot \left( \vec{\upvarepsilon}_j \vec{\upvarepsilon}_k^{\mathrm{T}} \mathrm{x}_1 \right) \right \rangle =&\\
&= \left \langle \delta_3 , \varphi'_2 \left( \mathrm{h}_2 \right) \odot \left( \vec{\upvarepsilon}_j \mathrm{x}_{1,k} \right) \right \rangle
\end{flalign*}
By Lemma~\ref{lemma4} we get:
\begin{flalign*}
\left \langle \delta_4, \dv{\Phi}{\mathrm{W}_{2,j,k}} \right \rangle &=
\left \langle \delta_3 \odot \varphi'_2 \left( \mathrm{h}_2 \right) , \vec{\upvarepsilon}_j \mathrm{x}_{1,k} \right \rangle =&\\
&= \left( \vec{\upvarepsilon}_j \mathrm{x}_{1,k} \right)^{\mathrm{T}} \left( \delta_3 \odot \varphi'_2 \left( \mathrm{h}_2 \right) \right) =&\\
&= \mathrm{x}_{1,k} \vec{\upvarepsilon}_j^{\mathrm{T}} \left( \delta_3 \odot \varphi'_2 \left( \mathrm{h}_2 \right) \right) =&\\
&= \mathrm{x}_{1,k} \left( \delta_3 \odot \varphi'_2 \left( \mathrm{h}_2 \right) \right)_j =&\\
&= \left( \delta_3 \odot \varphi'_2 \left( \mathrm{h}_2 \right) \right)_j \mathrm{x}_{1,k}
\end{flalign*}
Hence:
\begin{flalign*}
& \left \langle \delta_4, \dv{\Phi}{\mathrm{W}_2} \right \rangle =
\left( \delta_3 \odot \varphi'_2 \left( \mathrm{h}_2 \right) \right) \mathrm{x}_1^{\mathrm{T}} &
\end{flalign*}

\paragraph{Computing $\left \langle \delta_4, \dv{\Phi}{\mathrm{b}_2} \right \rangle$}
\begin{flalign*}
\left \langle \delta_4, \dv{\Phi}{\mathrm{b}_{2,j}} \right \rangle &=
\left \langle \delta_4, \sum_{i=1}^{64} \dv{\Phi}{\mathrm{x}_{2,i}} \cdot \dv{\mathrm{x}_{2,i}}{\mathrm{b}_{2,j}} \right \rangle =&\\
&= \sum_{i=1}^{64} \left \langle \delta_4, \dv{\Phi}{\mathrm{x}_{2,i}} \cdot \dv{\mathrm{x}_{2,i}}{\mathrm{b}_{2,j}} \right \rangle =&\\
&= \sum_{i=1}^{64} \left \langle \delta_4, \dv{\Phi}{\mathrm{x}_{2,i}} \right \rangle \cdot \dv{\mathrm{x}_{2,i}}{\mathrm{b}_{2,j}} =&\\
&= \sum_{i=1}^{64} \delta_{3,i} \cdot \dv{\mathrm{x}_{2,i}}{\mathrm{b}_{2,j}} =&\\
&= \sum_{i=1}^{64} \delta_{3,i} \cdot \dv{}{\mathrm{b}_{2,j}} { \left( \varphi_2 \left( \mathrm{h}_2 \right) \right) }_i =&\\
&= \sum_{i=1}^{64} \delta_{3,i} \cdot \left( \varphi'_2 \left( \mathrm{h}_2 \right) \dv{\mathrm{h}_2}{\mathrm{b}_{2,j}} \right)_i =&\\
&= \sum_{i=1}^{64} \delta_{3,i} \cdot \left( \varphi'_2 \left( \mathrm{h}_2 \right) \dv{}{\mathrm{b}_{2,j}} \left( \mathrm{W}_2 \mathrm{x}_1 + \mathrm{b}_2 \right) \right)_i =&\\
&= \sum_{i=1}^{64} \delta_{3,i} \cdot \left( \varphi'_2 \left( \mathrm{h}_2 \right) \dv{\mathrm{b}_2}{\mathrm{b}_{2,j}} \right)_i =&\\
&= \sum_{i=1}^{64} \delta_{3,i} \cdot \left( \varphi'_2 \left( \mathrm{h}_2 \right) \vec{\upvarepsilon}_j \right)_i =&\\
&= \sum_{i=1}^{64} \delta_{3,i} \cdot \varphi'_{2,i} \left( \mathrm{h}_2 \right) \cdot \delta_{i{j}} =&\\
&= \delta_{3,j} \cdot \varphi'_{2,j} \left( \mathrm{h}_2 \right)
\end{flalign*}
Hence:
\begin{flalign*}
& \left \langle \delta_4, \dv{\Phi}{\mathrm{b}_2} \right \rangle =
\delta_3 \cdot \varphi'_2 \left( \mathrm{h}_2 \right) &
\end{flalign*}

\paragraph{Computing $\left \langle \delta_4, \dv{\Phi}{\mathrm{W}_1} \right \rangle$}
\begin{flalign*}
\left \langle \delta_4, \dv{\Phi}{\mathrm{W}_{1,j,k}} \right \rangle &=
\left \langle \delta_4, \sum_{i=1}^{64} \dv{\Phi}{\mathrm{x}_{1,i}} \cdot \dv{\mathrm{x}_{1,i}}{\mathrm{W}_{1,j,k}} \right \rangle =&\\
&= \sum_{i=1}^{64} \left \langle \delta_4, \dv{\Phi}{\mathrm{x}_{1,i}} \cdot \dv{\mathrm{x}_{1,i}}{\mathrm{W}_{1,j,k}} \right \rangle =&\\
&= \sum_{i=1}^{64} \left \langle \delta_4, \dv{\Phi}{\mathrm{x}_{1,i}} \right \rangle \cdot \dv{\mathrm{x}_{1,i}}{\mathrm{W}_{1,j,k}} =&\\
&= \sum_{i=1}^{64} \delta_{2,i} \cdot \dv{\mathrm{x}_{1,i}}{\mathrm{W}_{1,j,k}} =&\\
&= \sum_{i=1}^{64} \delta_{2,i} \cdot \dv{}{\mathrm{W}_{1,j,k}} { \left( \varphi_1 \left( \mathrm{h}_1 \right) \right) }_i =&\\
&= \sum_{i=1}^{64} \delta_{2,i} \cdot \left( \varphi'_1 \left( \mathrm{h}_1 \right) \odot \dv{\mathrm{h}_1}{\mathrm{W}_{1,j,k}} \right)_i =&\\
&= \sum_{i=1}^{64} \delta_{2,i} \cdot \left( \varphi'_1 \left( \mathrm{h}_1 \right) \odot \dv{}{\mathrm{W}_{1,j,k}} \left( \mathrm{W}_1 \mathrm{x}_0 + \mathrm{b}_1 \right) \right)_i =&\\
&= \sum_{i=1}^{64} \delta_{2,i} \cdot \left( \varphi'_1 \left( \mathrm{h}_1 \right) \odot \left( \dv{\mathrm{W}_1}{\mathrm{W}_{1,j,k}} \mathrm{x}_0 \right) \right)_i =&\\
&= \sum_{i=1}^{64} \delta_{2,i} \cdot \left( \varphi'_1 \left( \mathrm{h}_1 \right) \odot \left( \vec{\upvarepsilon}_j \vec{\upvarepsilon}_k^{\mathrm{T}} \mathrm{x}_0 \right) \right)_i =&\\
&= \left \langle \delta_2 , \varphi'_1 \left( \mathrm{h}_1 \right) \odot \left( \vec{\upvarepsilon}_j \vec{\upvarepsilon}_k^{\mathrm{T}} \mathrm{x}_0 \right) \right \rangle =&\\
&= \left \langle \delta_2 , \varphi'_1 \left( \mathrm{h}_1 \right) \odot \left( \vec{\upvarepsilon}_j \mathrm{x}_{0,k} \right) \right \rangle
\end{flalign*}
By Lemma~\ref{lemma4} we get:
\begin{flalign*}
\left \langle \delta_4, \dv{\Phi}{\mathrm{W}_{1,j,k}} \right \rangle &=
\left \langle \delta_2 \odot \varphi'_1 \left( \mathrm{h}_1 \right) , \vec{\upvarepsilon}_j \mathrm{x}_{0,k} \right \rangle =&\\
&= \left( \vec{\upvarepsilon}_j \mathrm{x}_{0,k} \right)^{\mathrm{T}} \left( \delta_2 \odot \varphi'_1 \left( \mathrm{h}_1 \right) \right) =&\\
&= \mathrm{x}_{0,k} \vec{\upvarepsilon}_j^{\mathrm{T}} \left( \delta_2 \odot \varphi'_1 \left( \mathrm{h}_1 \right) \right) =&\\
&= \mathrm{x}_{0,k} \left( \delta_2 \odot \varphi'_1 \left( \mathrm{h}_1 \right) \right)_j =&\\
&= \left( \delta_2 \odot \varphi'_1 \left( \mathrm{h}_1 \right) \right)_j \mathrm{x}_{0,k}
\end{flalign*}
Hence:
\begin{flalign*}
& \left \langle \delta_4, \dv{\Phi}{\mathrm{W}_1} \right \rangle =
\left( \delta_2 \odot \varphi'_1 \left( \mathrm{h}_1 \right) \right) \mathrm{x}_0^{\mathrm{T}} &
\end{flalign*}

\paragraph{Computing $\left \langle \delta_4, \dv{\Phi}{\mathrm{b}_1} \right \rangle$}
\begin{flalign*}
\left \langle \delta_4, \dv{\Phi}{\mathrm{b}_{1,j}} \right \rangle &=
\left \langle \delta_4, \sum_{i=1}^{64} \dv{\Phi}{\mathrm{x}_{1,i}} \cdot \dv{\mathrm{x}_{1,i}}{\mathrm{b}_{1,j}} \right \rangle =&\\
&= \sum_{i=1}^{64} \left \langle \delta_4, \dv{\Phi}{\mathrm{x}_{1,i}} \cdot \dv{\mathrm{x}_{1,i}}{\mathrm{b}_{1,j}} \right \rangle =&\\
&= \sum_{i=1}^{64} \left \langle \delta_4, \dv{\Phi}{\mathrm{x}_{1,i}} \right \rangle \cdot \dv{\mathrm{x}_{1,i}}{\mathrm{b}_{1,j}} =&\\
&= \sum_{i=1}^{64} \delta_{2,i} \cdot \dv{\mathrm{x}_{1,i}}{\mathrm{b}_{1,j}} =&\\
&= \sum_{i=1}^{64} \delta_{2,i} \cdot \dv{}{\mathrm{b}_{1,j}} { \left( \varphi_1 \left( \mathrm{h}_1 \right) \right) }_i =&\\
&= \sum_{i=1}^{64} \delta_{2,i} \cdot \left( \varphi'_1 \left( \mathrm{h}_1 \right) \dv{\mathrm{h}_1}{\mathrm{b}_{1,j}} \right)_i =&\\
&= \sum_{i=1}^{64} \delta_{2,i} \cdot \left( \varphi'_1 \left( \mathrm{h}_1 \right) \dv{}{\mathrm{b}_{1,j}} \left( \mathrm{W}_1 \mathrm{x}_0 + \mathrm{b}_1 \right) \right)_i =&\\
&= \sum_{i=1}^{64} \delta_{2,i} \cdot \left( \varphi'_1 \left( \mathrm{h}_1 \right) \dv{\mathrm{b}_1}{\mathrm{b}_{1,j}} \right)_i =&\\
&= \sum_{i=1}^{64} \delta_{2,i} \cdot \left( \varphi'_1 \left( \mathrm{h}_1 \right) \vec{\upvarepsilon}_j \right)_i =&\\
&= \sum_{i=1}^{64} \delta_{2,i} \cdot \varphi'_{1,i} \left( \mathrm{h}_1 \right) \cdot \delta_{i{j}} =&\\
&= \delta_{2,j} \cdot \varphi'_{1,j} \left( \mathrm{h}_1 \right)
\end{flalign*}
Hence:
\begin{flalign*}
& \left \langle \delta_4, \dv{\Phi}{\mathrm{b}_1} \right \rangle =
\delta_2 \cdot \varphi'_1 \left( \mathrm{h}_1 \right) &
\end{flalign*}

\paragraph{Computing $\left \langle \delta_4, \dv{\Phi}{\mathrm{z}} \right \rangle$}
Since
$$
\mathrm{z} = \begin{bmatrix}
x_{0,8+24+1} \\
\vdots \\
x_{0,8+24+96}
\end{bmatrix} = \begin{bmatrix}
x_{0,32+1} \\
\vdots \\
x_{0,32+96}
\end{bmatrix}
$$
we have:
$$
\left \langle \delta_4, \dv{\Phi}{\mathrm{z}_i} \right \rangle =
\left \langle \delta_4, \dv{\Phi}{\mathrm{x}_{0,32+i}} \right \rangle =
\delta_{1,32+i}
$$
Consequently:
$$
\left \langle \delta_4, \dv{\Phi}{\mathrm{z}} \right \rangle = \begin{bmatrix}
\delta_{1,32+1} \\
\vdots \\
\delta_{1,32+96}
\end{bmatrix}
$$

\paragraph{Computing $\left \langle \delta_4, \dv{\Phi}{\mathcal{F}} \right \rangle$}
\noindent
\newline
\newline
Since $\mathcal{F} = \left \lbrace \mathrm{f}_1, \ldots, \mathrm{f}_{6577} \right \rbrace$, let us fix $i \in \left \lbrace 1, \ldots, 6577 \right \rbrace$ and $j \in \left \lbrace 1, 2 \right \rbrace$. Then:
\begin{flalign*}
\left \langle \delta_4, \dv{\Phi}{\mathrm{f}_{i,j}} \right \rangle &=
\left \langle \delta_4, \sum_{k=1}^{4} \dv{\Phi}{\mathrm{x}_{0,2 \left( k - 1 \right) + j}} \cdot \dv{\mathrm{x}_{0,2 \left( k - 1 \right) + j}}{\mathrm{f}_{i,j}} \right \rangle =&\\
&= \sum_{k=1}^{4} \left \langle \delta_4, \dv{\Phi}{\mathrm{x}_{0,2 \left( k - 1 \right) + j}} \cdot \dv{\mathrm{x}_{0,2 \left( k - 1 \right) + j}}{\mathrm{f}_{i,j}} \right \rangle =&\\
&= \sum_{k=1}^{4} \left \langle \delta_4, \dv{\Phi}{\mathrm{x}_{0,2 \left( k - 1 \right) + j}} \right \rangle \cdot \dv{\mathrm{x}_{0,2 \left( k - 1 \right) + j}}{\mathrm{f}_{i,j}} =&\\
&= \sum_{k=1}^{4} \delta_{1,2 \left( k - 1 \right) + j} \cdot \dv{\mathrm{x}_{0,2 \left( k - 1 \right) + j}}{\mathrm{f}_{i,j}} =&\\
&= \sum_{k=1}^{4} \delta_{1,2 \left( k - 1 \right) + j} \cdot \dv{\hat{\mathrm{f}}_j^{\left( k \right)}}{\mathrm{f}_{i,j}}
\end{flalign*}
It remains, therefore, to show how to compute: $\dv{\hat{\mathrm{f}}_j^{\left( k \right)}}{\mathrm{f}_{i,j}}$. As can be easily verified:
\begin{flalign*}
\dv{\hat{\mathrm{f}}_j^{\left( k \right)}}{\mathrm{f}_{i,j}} &=
\left( 1 - u_{\text{frac}}^{\left( k \right)} \right) \cdot \left( 1 - v_{\text{frac}}^{\left( k \right)} \right) \cdot \delta_{\frac{\text{ind}_{00}^{\left( k \right)}}{2} + 1, i} +&\\
&+ \; u_{\text{frac}}^{\left( k \right)} \cdot \left( 1 - v_{\text{frac}}^{\left( k \right)} \right) \cdot \delta_{\frac{\text{ind}_{01}^{\left( k \right)}}{2} + 1, i} +&\\
&+ \; \left( 1 - u_{\text{frac}}^{\left( k \right)} \right) \cdot v_{\text{frac}}^{\left( k \right)} \cdot \delta_{\frac{\text{ind}_{10}^{\left( k \right)}}{2} + 1, i} +&\\
&+ \; u_{\text{frac}}^{\left( k \right)} \cdot v_{\text{frac}}^{\left( k \right)} \cdot \delta_{\frac{\text{ind}_{11}^{\left( k \right)}}{2} + 1, i}
\end{flalign*}
where temporarily, due to a notation conflict, $\delta_{ij}$ denotes the Kronecker delta.
\newline
\newline
\noindent
To complete the derivation, we provide the closed-form formulas for the derivatives of the $\operatorname{hardswish}$, $\operatorname{ReLU}$ and sigmoid activation function $\sigma$.

\paragraph{Computing: $\dv{}{x} \left( \operatorname{hardswish} \left( x \right) \right)$, $\dv{}{x} \left( \operatorname{ReLU} \left( x \right) \right)$ and $\dv{}{x} \left( \sigma \left( x \right) \right)$}

\begin{flalign*}
\dv{}{x} \left( \operatorname{hardswish} \left( x \right) \right) &=
\dv{}{x} \left( 
\begin{cases}
0 & , \; x \le -3 \\
x & , \; x \ge 3 \\
\frac{x \left( x + 3 \right)}{6} & , \; \text{otherwise} \\
\end{cases}
\right) =&\\
&= \dv{}{x} \left( 
x \cdot \operatorname{clamp} \left( x + 3, 0, 6 \right)
\right) =&\\
&= \left( 
1 \cdot \operatorname{clamp} \left( x + 3, 0, 6 \right)
\right) + \left( x \cdot \dv{}{x} \left( \operatorname{clamp} \left( x + 3, 0, 6 \right) \right) \right) =&\\
&= \left( 
\operatorname{clamp} \left( x + 3, 0, 6 \right)
\right) + \left( x \cdot \mathds{1}_{ \left(x + 3\right) \in \left[ 0, 6 \right] } \right) =&\\
&= \left( 
\operatorname{clamp} \left( x + 3, 0, 6 \right)
\right) + \left( x \cdot \mathds{1}_{ x \in \left[ -3, 3 \right] } \right) =&\\
\end{flalign*}
\begin{flalign*}
\dv{}{x} \left( \operatorname{ReLU} \left( x \right) \right) &=
\dv{}{x} \left( \max \left \lbrace x, 0 \right \rbrace \right) =&\\
&= \mathds{1}_{x \ge 0}
\end{flalign*}
\begin{flalign*}
\dv{}{x} \left( \sigma \left( x \right) \right) &=
\dv{}{x} \left( \frac{1}{1 + e^{-x}} \right) =&\\
&= -\frac{1}{\left( 1 + e^{-x} \right)^2} \cdot \left( -e^{-x} \right) =&\\
&= \frac{e^{-x}}{\left( 1 + e^{-x} \right)^2} =&\\
&= \frac{e^{-x} + 0}{\left( 1 + e^{-x} \right)^2} =&\\
&= \frac{e^{-x} + \left( 1 - 1 \right)}{\left( 1 + e^{-x} \right)^2} =&\\
&= \frac{\left( 1 + e^{-x} \right) - 1}{\left( 1 + e^{-x} \right)^2} =&\\
&= \frac{1}{1 + e^{-x}} - \frac{1}{\left( 1 + e^{-x} \right)^2} =&\\
&= \frac{1}{1 + e^{-x}} \left( 1 - \frac{1}{1 + e^{-x}} \right) =&\\
&= \sigma \left( x \right) \left( 1 - \sigma \left( x \right) \right)
\end{flalign*}
\noindent
We will now show how to compute the inner products: $\left \langle \delta_4, \dv{\Phi}{u} \right \rangle$ and $\left \langle \delta_4, \dv{\Phi}{v} \right \rangle$. These quantities will be used to deduce the derivatives of $\dv{\mathcal{L}}{\upmu_i^{\left( r \right)}}$, $\dv{\mathcal{L}}{\tilde{\mathrm{s}}_i^{\left( r \right)}}$ and $\dv{\mathcal{L}}{\mathrm{q}_i^{\left( r \right)}}$, as they comprise one of their key components.

\paragraph{Computing $\left \langle \delta_4, \dv{\Phi}{u} \right \rangle$ and $\left \langle \delta_4, \dv{\Phi}{v} \right \rangle$}
\begin{flalign*}
\left \langle \delta_4, \dv{\Phi}{u} \right \rangle &=
\left \langle \delta_4, \sum_{i=1}^{8} \dv{\Phi}{\mathrm{x}_{0,i}} \cdot \dv{\mathrm{x}_{0,i}}{u} \right \rangle =&\\
&= \left \langle \delta_4, \sum_{i=1}^{4} \sum_{j=1}^{2} \dv{\Phi}{\mathrm{x}_{0,2 \left( i - 1 \right) + j}} \cdot \dv{\mathrm{x}_{0,2 \left( i - 1 \right) + j}}{u} \right \rangle =&\\
&= \sum_{i=1}^{4} \sum_{j=1}^{2} \left \langle \delta_4, \dv{\Phi}{\mathrm{x}_{0,2 \left( i - 1 \right) + j}} \cdot \dv{\mathrm{x}_{0,2 \left( i - 1 \right) + j}}{u} \right \rangle =&\\
&= \sum_{i=1}^{4} \sum_{j=1}^{2} \left \langle \delta_4, \dv{\Phi}{\mathrm{x}_{0,2 \left( i - 1 \right) + j}} \right \rangle \cdot \dv{\mathrm{x}_{0,2 \left( i - 1 \right) + j}}{u} =&\\
&= \sum_{i=1}^{4} \sum_{j=1}^{2} \delta_{1,2 \left( i - 1 \right) + j} \cdot \dv{\mathrm{x}_{0,2 \left( i - 1 \right) + j}}{u} =&\\
&= \sum_{i=1}^{4} \sum_{j=1}^{2} \delta_{1,2 \left( i - 1 \right) + j} \cdot \dv{\hat{\mathrm{f}}_j^{\left( i \right)}}{u}
\end{flalign*}
\begin{flalign*}
\left \langle \delta_4, \dv{\Phi}{v} \right \rangle &=
\left \langle \delta_4, \sum_{i=1}^{8} \dv{\Phi}{\mathrm{x}_{0,i}} \cdot \dv{\mathrm{x}_{0,i}}{v} \right \rangle =&\\
&= \left \langle \delta_4, \sum_{i=1}^{4} \sum_{j=1}^{2} \dv{\Phi}{\mathrm{x}_{0,2 \left( i - 1 \right) + j}} \cdot \dv{\mathrm{x}_{0,2 \left( i - 1 \right) + j}}{v} \right \rangle =&\\
&= \sum_{i=1}^{4} \sum_{j=1}^{2} \left \langle \delta_4, \dv{\Phi}{\mathrm{x}_{0,2 \left( i - 1 \right) + j}} \cdot \dv{\mathrm{x}_{0,2 \left( i - 1 \right) + j}}{v} \right \rangle =&\\
&= \sum_{i=1}^{4} \sum_{j=1}^{2} \left \langle \delta_4, \dv{\Phi}{\mathrm{x}_{0,2 \left( i - 1 \right) + j}} \right \rangle \cdot \dv{\mathrm{x}_{0,2 \left( i - 1 \right) + j}}{v} =&\\
&= \sum_{i=1}^{4} \sum_{j=1}^{2} \delta_{1,2 \left( i - 1 \right) + j} \cdot \dv{\mathrm{x}_{0,2 \left( i - 1 \right) + j}}{v} =&\\
&= \sum_{i=1}^{4} \sum_{j=1}^{2} \delta_{1,2 \left( i - 1 \right) + j} \cdot \dv{\hat{\mathrm{f}}_j^{\left( i \right)}}{v}
\end{flalign*}
It remains, therefore, to show how to compute: $\dv{\hat{\mathrm{f}}^{\left( i \right)}}{u}$ and $\dv{\hat{\mathrm{f}}^{\left( i \right)}}{v}$. As can be easily verified:
\begin{flalign*}
\dv{\hat{\mathrm{f}}^{\left( i \right)}}{u} &= 
\dv{\hat{\mathrm{f}}^{\left( i \right)}}{u_{\text{frac}}^{\left( i \right)}} \cdot \dv{u_{\text{frac}}^{\left( i \right)}}{u} =&\\
&= \left( \left( 1 - v_{\text{frac}}^{\left( i \right)} \right) \left( \mathrm{f}_{01}^{\left( i \right)} - \mathrm{f}_{00}^{\left( i \right)} \right) + v_{\text{frac}}^{\left( i \right)} \left( \mathrm{f}_{11}^{\left( i \right)} - \mathrm{f}_{10}^{\left( i \right)} \right) \right) \cdot \dv{u_{\text{frac}}^{\left( i \right)}}{u} =&\\
&= \left( \left( 1 - v_{\text{frac}}^{\left( i \right)} \right) \left( \mathrm{f}_{01}^{\left( i \right)} - \mathrm{f}_{00}^{\left( i \right)} \right) + v_{\text{frac}}^{\left( i \right)} \left( \mathrm{f}_{11}^{\left( i \right)} - \mathrm{f}_{10}^{\left( i \right)} \right) \right) \cdot \frac{N^{\left( i \right)} - 1}{2{R_{\ast}}}
\end{flalign*}
\begin{flalign*}
\dv{\hat{\mathrm{f}}^{\left( i \right)}}{v} &= 
\dv{\hat{\mathrm{f}}^{\left( i \right)}}{v_{\text{frac}}^{\left( i \right)}} \cdot \dv{v_{\text{frac}}^{\left( i \right)}}{v} =&\\
&= \left( \left( 1 - u_{\text{frac}}^{\left( i \right)} \right) \left( \mathrm{f}_{10}^{\left( i \right)} - \mathrm{f}_{00}^{\left( i \right)} \right) + u_{\text{frac}}^{\left( i \right)} \left( \mathrm{f}_{11}^{\left( i \right)} - \mathrm{f}_{01}^{\left( i \right)} \right) \right) \cdot \dv{v_{\text{frac}}^{\left( i \right)}}{v} =&\\
&= \left( \left( 1 - u_{\text{frac}}^{\left( i \right)} \right) \left( \mathrm{f}_{10}^{\left( i \right)} - \mathrm{f}_{00}^{\left( i \right)} \right) + u_{\text{frac}}^{\left( i \right)} \left( \mathrm{f}_{11}^{\left( i \right)} - \mathrm{f}_{01}^{\left( i \right)} \right) \right) \cdot \frac{N^{\left( i \right)} - 1}{2{R_{\ast}}}
\end{flalign*}
where the straight-through estimator was employed in the final equalities of both derivative formulas, neglecting the discontinuities associated with the floor function.

\subsection{Computing leaf gradients}
Since the vector: ${\mathrm{P}_i'}^{\left( r \right)}-\upmu_i^{\left( r \right)}$ connecting the mean $\upmu_i^{\left( r \right)}$ of the $i$-th Gaussian intersected by a ray $\mathrm{r}$ to the intersection point ${\mathrm{P}_i'}^{\left( r \right)}$ is orthogonal to the Gaussians' normal vector $\mathrm{\vec{N}_i^{\left( r \right)}}$, we have:
$$
\left \langle \mathrm{\vec{N}_i^{\left( r \right)}}, {\mathrm{P}_i'}^{\left( r \right)}-\upmu_i^{\left( r \right)} \right \rangle = 0
$$
Since:
\begin{flalign*}
\left \langle \mathrm{\vec{N}_i^{\left( r \right)}}, {\mathrm{P}_i'}^{\left( r \right)}-\upmu_i^{\left( r \right)} \right \rangle &=
\left \langle \mathrm{\vec{N}_i^{\left( r \right)}}, {\mathrm{P}_i'}^{\left( r \right)} \right \rangle - \left \langle \mathrm{\vec{N}_i^{\left( r \right)}}, \upmu_i^{\left( r \right)} \right \rangle =&\\
&= \left \langle \mathrm{\vec{N}_i^{\left( r \right)}}, \mathrm{o}^{\left( r \right)}+t_i^{\left( r \right)} \mathrm{\vec{d}}^{\left( r \right)} \right \rangle - \left \langle \mathrm{\vec{N}_i^{\left( r \right)}}, \upmu_i^{\left( r \right)} \right \rangle =&\\
&= \left \langle \mathrm{\vec{N}_i^{\left( r \right)}}, \mathrm{o}^{\left( r \right)} \right \rangle + \left \langle \mathrm{\vec{N}_i^{\left( r \right)}}, t_i^{\left( r \right)} \mathrm{\vec{d}}^{\left( r \right)} \right \rangle - \left \langle \mathrm{\vec{N}_i^{\left( r \right)}}, \upmu_i^{\left( r \right)} \right \rangle =&\\
&= \left \langle \mathrm{\vec{N}_i^{\left( r \right)}}, \mathrm{o}^{\left( r \right)} \right \rangle + {\left \langle \mathrm{\vec{N}_i^{\left( r \right)}}, \mathrm{\vec{d}}^{\left( r \right)} \right \rangle} \cdot t_i^{\left( r \right)} - \left \langle \mathrm{\vec{N}_i^{\left( r \right)}}, \upmu_i^{\left( r \right)} \right \rangle =&\\
&= \left \langle \mathrm{\vec{N}_i^{\left( r \right)}}, \mathrm{o}^{\left( r \right)} - \upmu_i^{\left( r \right)} \right \rangle + {\left \langle \mathrm{\vec{N}_i^{\left( r \right)}}, \mathrm{\vec{d}}^{\left( r \right)} \right \rangle} \cdot t_i^{\left( r \right)}
\end{flalign*}
for some $t_i^{\left( r \right)} \in \mathbb{R}_{+} \cup \left \lbrace 0 \right \rbrace$, it must hold that:
$$
{\left \langle \mathrm{\vec{N}_i^{\left( r \right)}}, \mathrm{\vec{d}}^{\left( r \right)} \right \rangle} \cdot t_i^{\left( r \right)} = -{\left \langle \mathrm{\vec{N}_i^{\left( r \right)}}, \mathrm{o}^{\left( r \right)} - \upmu_i^{\left( r \right)} \right \rangle} = \left \langle \mathrm{\vec{N}_i^{\left( r \right)}}, \upmu_i^{\left( r \right)} - \mathrm{o}^{\left( r \right)} \right \rangle
$$
Hence:
$$
t_i^{\left( r \right)} = \frac{\left \langle \mathrm{\vec{N}_i^{\left( r \right)}}, \upmu_i^{\left( r \right)} - \mathrm{o}^{\left( r \right)} \right \rangle}{ \left \langle \mathrm{\vec{N}_i^{\left( r \right)}}, \mathrm{\vec{d}}^{\left( r \right)} \right \rangle }
$$
and the hit point ${\mathrm{P}_i'}^{\left( r \right)}$ is given by the following formula:
$$
{\mathrm{P}_i'}^{\left( r \right)} = \mathrm{o}^{\left( r \right)}+t_i^{\left( r \right)} \mathrm{\vec{d}}^{\left( r \right)}
$$
The vectors: $\mathrm{\vec{U}}_i^{\left( r \right)}$, $\mathrm{\vec{V}}_i^{\left( r \right)}$, and $\mathrm{\vec{N}}_i^{\left( r \right)}$ form an orthogonal coordinate frame, therefore the intersection point local frame coordinates: $u_i^{\left( r \right)}$ and $v_i^{\left( r \right)}$ can be determined by:
$$
u_i^{\left( r \right)} = \frac{ \left \langle \mathrm{\vec{U}_i^{\left( r \right)}}, {\mathrm{P}_i'}^{\left( r \right)} - \upmu_i^{\left( r \right)} \right \rangle }{ \mathrm{s}_{i,1}^{\left( r \right)} }
$$
$$
v_i^{\left( r \right)} = \frac{ \left \langle \mathrm{\vec{V}_i^{\left( r \right)}}, {\mathrm{P}_i'}^{\left( r \right)} - \upmu_i^{\left( r \right)} \right \rangle }{ \mathrm{s}_{i,2}^{\left( r \right)} }
$$
Thus:
\begin{flalign*}
& \dv{t_i^{\left( r \right)}}{\upmu_{i,j}^{\left( r \right)}} =
\frac{\left \langle \mathrm{\vec{N}_i^{\left( r \right)}}, \pdv{}{\upmu_{i,j}^{\left( r \right)}} \left( \upmu_i^{\left( r \right)} - \mathrm{o}^{\left( r \right)} \right) \right \rangle}{ \left \langle \mathrm{\vec{N}_i^{\left( r \right)}}, \mathrm{\vec{d}}^{\left( r \right)} \right \rangle } =
\frac{\left \langle \mathrm{\vec{N}_i^{\left( r \right)}}, \vec{\upvarepsilon}_j \right \rangle}{ \left \langle \mathrm{\vec{N}_i^{\left( r \right)}}, \mathrm{\vec{d}}^{\left( r \right)} \right \rangle } =
\frac{\mathrm{\vec{N}_{i,j}^{\left( r \right)}}}{ \left \langle \mathrm{\vec{N}_i^{\left( r \right)}}, \mathrm{\vec{d}}^{\left( r \right)} \right \rangle } &
\end{flalign*}
\begin{flalign*}
\dv{u_i^{\left( r \right)}}{\upmu_{i,j}^{\left( r \right)}} &=
\frac{\left \langle \mathrm{\vec{U}_i^{\left( r \right)}}, \pdv{}{\upmu_{i,j}^{\left( r \right)}} \left( {\mathrm{P}_i'}^{\left( r \right)} - \upmu_i^{\left( r \right)} \right) \right \rangle}{ \mathrm{s}_{i,1}^{\left( r \right)} } =&\\
&= \frac{\left \langle \mathrm{\vec{U}_i^{\left( r \right)}}, \pdv{}{\upmu_{i,j}^{\left( r \right)}} \left( {\mathrm{P}_i'}^{\left( r \right)} - \upmu_i^{\left( r \right)} \right) \right \rangle}{ \mathrm{s}_{i,1}^{\left( r \right)} } =&\\
&= \frac{\left \langle \mathrm{\vec{U}_i^{\left( r \right)}}, {\pdv{}{\upmu_{i,j}^{\left( r \right)}}}{\mathrm{P}_i'}^{\left( r \right)} - \pdv{}{\upmu_{i,j}} \upmu_i^{\left( r \right)} \right \rangle}{ \mathrm{s}_{i,1}^{\left( r \right)} } =&\\
&= \frac{\left \langle \mathrm{\vec{U}_i^{\left( r \right)}}, {\pdv{}{\upmu_{i,j}^{\left( r \right)}}} \left( \mathrm{o}^{\left( r \right)}+t_i^{\left( r \right)} \mathrm{\vec{d}}^{\left( r \right)} \right) -  \vec{\upvarepsilon}_j \right \rangle}{ \mathrm{s}_{i,1}^{\left( r \right)} } =&\\
&= \frac{\left \langle \mathrm{\vec{U}_i^{\left( r \right)}}, {\pdv{}{\upmu_{i,j}^{\left( r \right)}}} \left( t_i^{\left( r \right)} \right) \cdot \mathrm{\vec{d}}^{\left( r \right)} - \vec{\upvarepsilon}_j \right \rangle}{ \mathrm{s}_{i,1}^{\left( r \right)} } =&\\
&= \frac{\left \langle \mathrm{\vec{U}_i^{\left( r \right)}}, {\pdv{}{\upmu_{i,j}^{\left( r \right)}}} \left( t_i^{\left( r \right)} \right) \cdot \mathrm{\vec{d}}^{\left( r \right)} \right \rangle}{ \mathrm{s}_{i,1}^{\left( r \right)} } - \frac{\left \langle \mathrm{\vec{U}_i^{\left( r \right)}},  \vec{\upvarepsilon}_j \right \rangle}{ \mathrm{s}_{i,1}^{\left( r \right)} } =&\\
&= \frac{\left \langle \mathrm{\vec{U}_i^{\left( r \right)}}, \mathrm{\vec{d}}^{\left( r \right)} \right \rangle}{ \mathrm{s}_{i,1}^{\left( r \right)} } \cdot {\pdv{}{\upmu_{i,j}^{\left( r \right)}}} t_i^{\left( r \right)} - \frac{ \mathrm{\vec{U}_{i,j}^{\left( r \right)}}}{ \mathrm{s}_{i,1}^{\left( r \right)} } =&\\
&= \frac{\left \langle \mathrm{\vec{U}_i^{\left( r \right)}}, \mathrm{\vec{d}}^{\left( r \right)} \right \rangle \cdot {\pdv{}{\upmu_{i,j}^{\left( r \right)}}} t_i^{\left( r \right)} - \mathrm{\vec{U}_{i,j}^{\left( r \right)}}}{ \mathrm{s}_{i,1}^{\left( r \right)} }
\end{flalign*}
\begin{flalign*}
\dv{v_i^{\left( r \right)}}{\upmu_{i,j}^{\left( r \right)}} &=
\frac{\left \langle \vec{\mathrm{V}}_i^{\left( r \right)}, \pdv{}{\upmu_{i,j}^{\left( r \right)}} \left( {\mathrm{P}_i'}^{\left( r \right)} - \upmu_i^{\left( r \right)} \right) \right \rangle}{ \mathrm{s}_{i,2}^{\left( r \right)} } &\\
&= \frac{\left \langle \vec{\mathrm{V}}_i^{\left( r \right)}, \pdv{}{\upmu_{i,j}^{\left( r \right)}} \left( {\mathrm{P}_i'}^{\left( r \right)} - \upmu_i^{\left( r \right)} \right) \right \rangle}{ \mathrm{s}_{i,2}^{\left( r \right)} } &\\
&= \frac{\left \langle \vec{\mathrm{V}}_i^{\left( r \right)}, {\pdv{}{\upmu_{i,j}^{\left( r \right)}}}{\mathrm{P}_i'}^{\left( r \right)} - \pdv{}{\upmu_{i,j}^{\left( r \right)}} \upmu_i^{\left( r \right)} \right \rangle}{ \mathrm{s}_{i,2}^{\left( r \right)} } &\\
&= \frac{\left \langle \vec{\mathrm{V}}_i^{\left( r \right)}, {\pdv{}{\upmu_{i,j}^{\left( r \right)}}} \left( \mathrm{o}^{\left( r \right)}+t_i^{\left( r \right)} \mathrm{\vec{d}}^{\left( r \right)} \right) -  \vec{\upvarepsilon}_j \right \rangle}{ \mathrm{s}_{i,2}^{\left( r \right)} } &\\
&= \frac{\left \langle \vec{\mathrm{V}}_i^{\left( r \right)}, {\pdv{}{\upmu_{i,j}^{\left( r \right)}}} \left( t_i^{\left( r \right)} \right) \cdot \mathrm{\vec{d}}^{\left( r \right)} -  \vec{\upvarepsilon}_j \right \rangle}{ \mathrm{s}_{i,2}^{\left( r \right)} } &\\
&= \frac{\left \langle \vec{\mathrm{V}}_i^{\left( r \right)}, {\pdv{}{\upmu_{i,j}^{\left( r \right)}}} \left( t_i^{\left( r \right)} \right) \cdot \mathrm{\vec{d}}^{\left( r \right)} \right \rangle}{ \mathrm{s}_{i,2}^{\left( r \right)} } - \frac{\left \langle \vec{\mathrm{V}}_i^{\left( r \right)},  \vec{\upvarepsilon}_j \right \rangle}{ \mathrm{s}_{i,2}^{\left( r \right)} } &\\
&= \frac{\left \langle \vec{\mathrm{V}}_i^{\left( r \right)}, \mathrm{\vec{d}}^{\left( r \right)} \right \rangle}{ \mathrm{s}_{i,2}^{\left( r \right)} } \cdot {\pdv{}{\upmu_{i,j}^{\left( r \right)}}} t_i^{\left( r \right)} - \frac{ \mathrm{\vec{V}}_{i,j}^{\left( r \right)}}{ \mathrm{s}_{i,2}^{\left( r \right)} } &\\
&= \frac{\left \langle \vec{\mathrm{V}}_i^{\left( r \right)}, \mathrm{\vec{d}}^{\left( r \right)} \right \rangle \cdot {\pdv{}{\upmu_{i,j}^{\left( r \right)}}} t_i^{\left( r \right)} - \mathrm{\vec{V}}_{i,j}^{\left( r \right)}}{ \mathrm{s}_{i,2}^{\left( r \right)} }
\end{flalign*}
\begin{flalign*}
& \dv{t_i^{\left( r \right)}}{\tilde{\mathrm{s}}_{i,j}^{\left( r \right)}} = 0 &
\end{flalign*}
\begin{flalign*}
\dv{u_i^{\left( r \right)}}{\tilde{\mathrm{s}}_{i,j}^{\left( r \right)}} &=
\dv{}{\tilde{\mathrm{s}}_{i,j}^{\left( r \right)}} \frac{ \left \langle \mathrm{\vec{U}_i^{\left( r \right)}}, {\mathrm{P}_i'}^{\left( r \right)} - \upmu_i^{\left( r \right)} \right \rangle }{ \exp( \tilde{\mathrm{s}}_{i,1}^{\left( r \right)}) } =&\\
&= -\frac{ \left \langle \mathrm{\vec{U}_i^{\left( r \right)}}, {\mathrm{P}_i'}^{\left( r \right)} - \upmu_i^{\left( r \right)} \right \rangle }{ \exp^2( \tilde{\mathrm{s}}_{i,1}^{\left( r \right)}) } \cdot \dv{}{\tilde{\mathrm{s}}_{i,j}^{\left( r \right)}} \left( \exp( \tilde{\mathrm{s}}_{i,1}^{\left( r \right)}) \right) \cdot \delta_{j{1}} =&\\
&= -\frac{ \left \langle \mathrm{\vec{U}_i^{\left( r \right)}}, {\mathrm{P}_i'}^{\left( r \right)} - \upmu_i^{\left( r \right)} \right \rangle }{ \exp^2( \tilde{\mathrm{s}}_{i,1}^{\left( r \right)}) } \cdot \exp( \tilde{\mathrm{s}}_{i,1}^{\left( r \right)}) \cdot \delta_{j{1}} =&\\
&= -\frac{ \left \langle \mathrm{\vec{U}_i^{\left( r \right)}}, {\mathrm{P}_i'}^{\left( r \right)} - \upmu_i^{\left( r \right)} \right \rangle }{ \exp( \tilde{\mathrm{s}}_{i,1}^{\left( r \right)}) } \cdot \delta_{j{1}} =&\\
&= -\frac{ \left \langle \mathrm{\vec{U}_i^{\left( r \right)}}, {\mathrm{P}_i'}^{\left( r \right)} - \upmu_i^{\left( r \right)} \right \rangle }{ \exp( \tilde{\mathrm{s}}_{i,1}^{\left( r \right)}) } \cdot \delta_{j{1}} =&\\
&= -u_i^{\left( r \right)} \cdot \delta_{j{1}}
\end{flalign*}
\begin{flalign*}
\dv{v_i^{\left( r \right)}}{\tilde{\mathrm{s}}_{i,j}^{\left( r \right)}} &=
\dv{}{\tilde{\mathrm{s}}_{i,j}^{\left( r \right)}} \frac{ \left \langle \mathrm{\vec{V}_i^{\left( r \right)}}, {\mathrm{P}_i'}^{\left( r \right)} - \upmu_i^{\left( r \right)} \right \rangle }{ \exp( \tilde{\mathrm{s}}_{i,2}^{\left( r \right)}) } =&\\
&= -\frac{ \left \langle \mathrm{\vec{V}_i^{\left( r \right)}}, {\mathrm{P}_i'}^{\left( r \right)} - \upmu_i^{\left( r \right)} \right \rangle }{ \exp^2( \tilde{\mathrm{s}}_{i,2}^{\left( r \right)}) } \cdot \dv{}{\tilde{\mathrm{s}}_{i,j}^{\left( r \right)}} \left( \exp( \tilde{\mathrm{s}}_{i,2}^{\left( r \right)}) \right) \cdot \delta_{j{2}} =&\\
&= -\frac{ \left \langle \mathrm{\vec{V}_i^{\left( r \right)}}, {\mathrm{P}_i'}^{\left( r \right)} - \upmu_i^{\left( r \right)} \right \rangle }{ \exp^2( \tilde{\mathrm{s}}_{i,2}^{\left( r \right)}) } \cdot \exp( \tilde{\mathrm{s}}_{i,2}^{\left( r \right)}) \cdot \delta_{j{2}} =&\\
&= -\frac{ \left \langle \mathrm{\vec{V}_i^{\left( r \right)}}, {\mathrm{P}_i'}^{\left( r \right)} - \upmu_i^{\left( r \right)} \right \rangle }{ \exp( \tilde{\mathrm{s}}_{i,2}^{\left( r \right)}) } \cdot \delta_{j{2}} =&\\
&= -\frac{ \left \langle \mathrm{\vec{V}_i^{\left( r \right)}}, {\mathrm{P}_i'}^{\left( r \right)} - \upmu_i^{\left( r \right)} \right \rangle }{ \exp( \tilde{\mathrm{s}}_{i,2}^{\left( r \right)}) } \cdot \delta_{j{2}} =&\\
&= -v_i^{\left( r \right)} \cdot \delta_{j{2}}
\end{flalign*}
\begin{flalign*}
\dv{t_i^{\left( r \right)}}{\mathrm{q}_{i,j}^{\left( r \right)}} &=
\frac{{\pdv{}{\mathrm{q}_{i,j}^{\left( r \right)}} \left( \left \langle \mathrm{\vec{N}_i^{\left( r \right)}}, \upmu_i^{\left( r \right)} - \mathrm{o}^{\left( r \right)} \right \rangle \right)} \cdot {\left \langle \mathrm{\vec{N}_i^{\left( r \right)}}, \mathrm{\vec{d}}^{\left( r \right)} \right \rangle} - {\left \langle \mathrm{\vec{N}_i^{\left( r \right)}}, \upmu_i^{\left( r \right)} - \mathrm{o}^{\left( r \right)} \right \rangle} \cdot {\pdv{}{\mathrm{q}_{i,j}^{\left( r \right)}} \left( \left \langle \mathrm{\vec{N}_i^{\left( r \right)}}, \mathrm{\vec{d}}^{\left( r \right)} \right \rangle \right)}}{{\left \langle \mathrm{\vec{N}_i^{\left( r \right)}}, \mathrm{\vec{d}}^{\left( r \right)} \right \rangle}^2} =&\\
&= \frac{{ \left \langle \pdv{}{\mathrm{q}_{i,j}^{\left( r \right)}} \mathrm{\vec{N}_i^{\left( r \right)}}, \upmu_i^{\left( r \right)} - \mathrm{o}^{\left( r \right)} \right \rangle } \cdot {\left \langle \mathrm{\vec{N}_i^{\left( r \right)}}, \mathrm{\vec{d}}^{\left( r \right)} \right \rangle} - {\left \langle \mathrm{\vec{N}_i^{\left( r \right)}}, \upmu_i^{\left( r \right)} - \mathrm{o}^{\left( r \right)} \right \rangle} \cdot { \left \langle \pdv{}{\mathrm{q}_{i,j}^{\left( r \right)}} \mathrm{\vec{N}_i^{\left( r \right)}}, \mathrm{\vec{d}}^{\left( r \right)} \right \rangle }}{{\left \langle \mathrm{\vec{N}_i^{\left( r \right)}}, \mathrm{\vec{d}}^{\left( r \right)} \right \rangle}^2}
\end{flalign*}
\begin{flalign*}
\dv{u_i^{\left( r \right)}}{\mathrm{q}_{i,j}^{\left( r \right)}} &=
\frac{ \left \langle \mathrm{ \pdv{}{\mathrm{q}_{i,j}^{\left( r \right)}} \vec{U}_i^{\left( r \right)}}, {\mathrm{P}_i'}^{\left( r \right)} - \upmu_i^{\left( r \right)} \right \rangle + \left \langle \mathrm{ \vec{U}_i^{\left( r \right)}}, \pdv{}{\mathrm{q}_{i,j}^{\left( r \right)}} \left( {\mathrm{P}_i'}^{\left( r \right)} - \upmu_i^{\left( r \right)} \right) \right \rangle }{ \mathrm{s}_{i,1}^{\left( r \right)} } =&\\
&= \frac{ \left \langle \mathrm{ \pdv{}{\mathrm{q}_{i,j}^{\left( r \right)}} \vec{U}_i^{\left( r \right)}}, {\mathrm{P}_i'}^{\left( r \right)} - \upmu_i^{\left( r \right)} \right \rangle + \left \langle \mathrm{ \vec{U}_i^{\left( r \right)}}, {\pdv{}{\mathrm{q}_{i,j}^{\left( r \right)}}} {\mathrm{P}_i'}^{\left( r \right)} \right \rangle }{ \mathrm{s}_{i,1}^{\left( r \right)} } =&\\
&= \frac{ \left \langle \mathrm{ \pdv{}{\mathrm{q}_{i,j}^{\left( r \right)}} \vec{U}_i^{\left( r \right)}}, {\mathrm{P}_i'}^{\left( r \right)} - \upmu_i^{\left( r \right)} \right \rangle + \left \langle \mathrm{ \vec{U}_i^{\left( r \right)}}, {\pdv{}{\mathrm{q}_{i,j}^{\left( r \right)}}} \left( \mathrm{o}^{\left( r \right)}+t_i^{\left( r \right)} \mathrm{\vec{d}}^{\left( r \right)} \right) \right \rangle }{ \mathrm{s}_{i,1}^{\left( r \right)} } =&\\
&= \frac{ \left \langle \mathrm{ \pdv{}{\mathrm{q}_{i,j}^{\left( r \right)}} \vec{U}_i^{\left( r \right)}}, {\mathrm{P}_i'}^{\left( r \right)} - \upmu_i^{\left( r \right)} \right \rangle + \left \langle \mathrm{ \vec{U}_i^{\left( r \right)}}, {\pdv{}{\mathrm{q}_{i,j}^{\left( r \right)}}} \left( t_i^{\left( r \right)} \mathrm{\vec{d}}^{\left( r \right)} \right) \right \rangle }{ \mathrm{s}_{i,1}^{\left( r \right)} } =&\\
&= \frac{ \left \langle \mathrm{ \pdv{}{\mathrm{q}_{i,j}^{\left( r \right)}} \vec{U}_i^{\left( r \right)}}, {\mathrm{P}_i'}^{\left( r \right)} - \upmu_i^{\left( r \right)} \right \rangle + \left \langle \mathrm{ \vec{U}_i^{\left( r \right)}}, {\pdv{}{\mathrm{q}_{i,j}^{\left( r \right)}}} \left( t_i^{\left( r \right)} \right) \cdot \mathrm{\vec{d}}^{\left( r \right)} \right \rangle }{ \mathrm{s}_{i,1}^{\left( r \right)} } =&\\
&= \frac{ \left \langle \mathrm{ \pdv{}{\mathrm{q}_{i,j}^{\left( r \right)}} \vec{U}_i^{\left( r \right)}}, {\mathrm{P}_i'}^{\left( r \right)} - \upmu_i^{\left( r \right)} \right \rangle + { \left \langle \mathrm{ \vec{U}_i^{\left( r \right)}}, \mathrm{\vec{d}}^{\left( r \right)} \right \rangle } \cdot {\pdv{}{\mathrm{q}_{i,j}^{\left( r \right)}}} t_i^{\left( r \right)} }{ \mathrm{s}_{i,1}^{\left( r \right)} } =&\\
\end{flalign*}
\begin{flalign*}
\dv{v_i^{\left( r \right)}}{\mathrm{q}_{i,j}^{\left( r \right)}} &=
\frac{ \left \langle \mathrm{ \pdv{}{\mathrm{q}_{i,j}^{\left( r \right)}} \vec{V}_i^{\left( r \right)}}, {\mathrm{P}_i'}^{\left( r \right)} - \upmu_i^{\left( r \right)} \right \rangle + \left \langle \mathrm{ \vec{V}_i^{\left( r \right)}}, \pdv{}{\mathrm{q}_{i,j}^{\left( r \right)}} \left( {\mathrm{P}_i'}^{\left( r \right)} - \upmu_i^{\left( r \right)} \right) \right \rangle }{ \mathrm{s}_{i,2}^{\left( r \right)} } =&\\
&= \frac{ \left \langle \mathrm{ \pdv{}{\mathrm{q}_{i,j}^{\left( r \right)}} \vec{V}_i^{\left( r \right)}}, {\mathrm{P}_i'}^{\left( r \right)} - \upmu_i^{\left( r \right)} \right \rangle + \left \langle \mathrm{ \vec{V}_i^{\left( r \right)}}, {\pdv{}{\mathrm{q}_{i,j}^{\left( r \right)}}} {\mathrm{P}_i'}^{\left( r \right)} \right \rangle }{ \mathrm{s}_{i,2}^{\left( r \right)} } =&\\
&= \frac{ \left \langle \mathrm{ \pdv{}{\mathrm{q}_{i,j}^{\left( r \right)}} \vec{V}_i^{\left( r \right)}}, {\mathrm{P}_i'}^{\left( r \right)} - \upmu_i^{\left( r \right)} \right \rangle + \left \langle \mathrm{ \vec{V}_i^{\left( r \right)}}, {\pdv{}{\mathrm{q}_{i,j}^{\left( r \right)}}} \left( \mathrm{o}^{\left( r \right)}+t_i^{\left( r \right)} \mathrm{\vec{d}}^{\left( r \right)} \right) \right \rangle }{ \mathrm{s}_{i,2}^{\left( r \right)} } =&\\
&= \frac{ \left \langle \mathrm{ \pdv{}{\mathrm{q}_{i,j}^{\left( r \right)}} \vec{V}_i^{\left( r \right)}}, {\mathrm{P}_i'}^{\left( r \right)} - \upmu_i^{\left( r \right)} \right \rangle + \left \langle \mathrm{ \vec{V}_i^{\left( r \right)}}, {\pdv{}{\mathrm{q}_{i,j}^{\left( r \right)}}} \left( t_i^{\left( r \right)} \mathrm{\vec{d}}^{\left( r \right)} \right) \right \rangle }{ \mathrm{s}_{i,2}^{\left( r \right)} } =&\\
&= \frac{ \left \langle \mathrm{ \pdv{}{\mathrm{q}_{i,j}^{\left( r \right)}} \vec{V}_i^{\left( r \right)}}, {\mathrm{P}_i'}^{\left( r \right)} - \upmu_i^{\left( r \right)} \right \rangle + \left \langle \mathrm{ \vec{V}_i^{\left( r \right)}}, {\pdv{}{\mathrm{q}_{i,j}^{\left( r \right)}}} \left( t_i^{\left( r \right)} \right) \cdot \mathrm{\vec{d}}^{\left( r \right)} \right \rangle }{ \mathrm{s}_{i,2}^{\left( r \right)} } =&\\
&= \frac{ \left \langle \mathrm{ \pdv{}{\mathrm{q}_{i,j}^{\left( r \right)}} \vec{V}_i^{\left( r \right)}}, {\mathrm{P}_i'}^{\left( r \right)} - \upmu_i^{\left( r \right)} \right \rangle + { \left \langle \mathrm{ \vec{V}_i^{\left( r \right)}}, \mathrm{\vec{d}}^{\left( r \right)} \right \rangle } \cdot {\pdv{}{\mathrm{q}_{i,j}^{\left( r \right)}}} t_i^{\left( r \right)} }{ \mathrm{s}_{i,2}^{\left( r \right)} } =&\\
\end{flalign*}
Recall that:
$$
\mathrm{\vec{U}}_i^{\left( r \right)} = {
\begin{bmatrix}
\mathrm{R}_{i,1,1}^{\left( r \right)}, \mathrm{R}_{i,2,1}^{\left( r \right)}, \mathrm{R}_{i,3,1}^{\left( r \right)}
\end{bmatrix}
}^\mathrm{T}
$$
$$
\mathrm{\vec{V}_i^{\left( r \right)}} = {
\begin{bmatrix}
\mathrm{R}_{i,1,2}^{\left( r \right)}, \mathrm{R}_{i,2,2}^{\left( r \right)}, \mathrm{R}_{i,3,2}^{\left( r \right)}
\end{bmatrix}
}^\mathrm{T}
$$
$$
\mathrm{\vec{N}_i^{\left( r \right)}} = {
\begin{bmatrix}
\mathrm{R}_{i,1,3}^{\left( r \right)}, \mathrm{R}_{i,2,3}^{\left( r \right)}, \mathrm{R}_{i,3,3}^{\left( r \right)}
\end{bmatrix}
}^\mathrm{T}
$$
where:
$$
\mathrm{R}_i^{\left( r \right)} = {\begin{bmatrix}
1 - cc_i^{\left( r \right)} - dd_i^{\left( r \right)} & bc_i^{\left( r \right)} - ad_i^{\left( r \right)} & bd_i^{\left( r \right)} + ac_i^{\left( r \right)} \\
bc_i^{\left( r \right)} + ad_i^{\left( r \right)} & 1 - bb_i^{\left( r \right)} - dd_i^{\left( r \right)} & cd_i^{\left( r \right)} - ab_i^{\left( r \right)} \\
bd_i^{\left( r \right)} - ac_i^{\left( r \right)} & cd_i^{\left( r \right)} + ab_i^{\left( r \right)} & 1 - bb_i^{\left( r \right)} - cc_i^{\left( r \right)}
\end{bmatrix}
}
$$
for:
$$
s_i^{\left( r \right)} = \frac{2}{\left( \mathrm{q}_{i,1}^{\left( r \right)} \right)^2 + \left( \mathrm{q}_{i,2}^{\left( r \right)} \right)^2 + \left( \mathrm{q}_{i,3}^{\left( r \right)} \right)^2 + \left( \mathrm{q}_{i,4}^{\left( r \right)} \right)^2}
$$
$$
\begin{matrix}
bs_i^{\left( r \right)} = \mathrm{q}_{i,2}^{\left( r \right)} \cdot s_i^{\left( r \right)}, & cs_i^{\left( r \right)} = \mathrm{q}_{i,3}^{\left( r \right)} \cdot s_i^{\left( r \right)}, & ds_i^{\left( r \right)} = \mathrm{q}_{i,4}^{\left( r \right)} \cdot s_i^{\left( r \right)} \\
ab_i^{\left( r \right)} = \mathrm{q}_{i,1}^{\left( r \right)} \cdot bs_i^{\left( r \right)}, & ac_i^{\left( r \right)} = \mathrm{q}_{i,1}^{\left( r \right)} \cdot cs_i^{\left( r \right)}, & ad_i^{\left( r \right)} = \mathrm{q}_{i,1}^{\left( r \right)} \cdot ds_i^{\left( r \right)} \\
bb_i^{\left( r \right)} = \mathrm{q}_{i,2}^{\left( r \right)} \cdot bs_i^{\left( r \right)}, & bc_i^{\left( r \right)} = \mathrm{q}_{i,2}^{\left( r \right)} \cdot cs_i^{\left( r \right)}, & bd_i^{\left( r \right)} = \mathrm{q}_{i,2}^{\left( r \right)} \cdot ds_i^{\left( r \right)} \\
cc_i^{\left( r \right)} = \mathrm{q}_{i,3}^{\left( r \right)} \cdot cs_i^{\left( r \right)}, & cd_i^{\left( r \right)} = \mathrm{q}_{i,3}^{\left( r \right)} \cdot ds_i^{\left( r \right)}, & dd_i^{\left( r \right)} = \mathrm{q}_{4,1}^{\left( r \right)} \cdot ds_i^{\left( r \right)}
\end{matrix}
$$
Let us define $aa \coloneq \left( \mathrm{q}_{i,1}^{\left( r \right)} \right)^2 \cdot s_i^{\left( r \right)}$, $a_i^{\left( r \right)} \coloneq \mathrm{q}_{i,1}^{\left( r \right)}$, $b_i^{\left( r \right)} \coloneq \mathrm{q}_{i,2}^{\left( r \right)}$, $c_i^{\left( r \right)} \coloneq \mathrm{q}_{i,3}^{\left( r \right)}$, and $d_i^{\left( r \right)} \coloneq \mathrm{q}_{i,4}^{\left( r \right)}$. Then, after straightforward yet tedious calculations, we finally obtain:
$$
\scriptsize
\dv{\mathrm{R}_i^{\left( r \right)}}{\mathrm{q}_{i,1}^{\left( r \right)}} = \begin{bmatrix}
s_i^{\left( r \right)} a_i^{\left( r \right)} \left( \left( 1 - aa_i^{\left( r \right)} \right) + \left( 1 - bb_i^{\left( r \right)} \right) \right) &
s_i^{\left( r \right)} \left( -d_i^{\left( r \right)} \left( 1 - aa_i^{\left( r \right)} \right) - ab_i^{\left( r \right)} \cdot c_i^{\left( r \right)} \right) &
s_i^{\left( r \right)} \left( c_i^{\left( r \right)} \left( 1 - aa_i^{\left( r \right)} \right) - ab_i^{\left( r \right)} \cdot d_i^{\left( r \right)} \right) \\
s_i^{\left( r \right)} \left( d_i^{\left( r \right)} \left( 1 - aa_i^{\left( r \right)} \right) - ab_i^{\left( r \right)} \cdot c_i^{\left( r \right)} \right) &
s_i^{\left( r \right)} a_i^{\left( r \right)} \left( \left( 1 - aa_i^{\left( r \right)} \right) + \left( 1 - cc_i^{\left( r \right)} \right) \right) &
s_i^{\left( r \right)} \left( -b_i^{\left( r \right)} \left( 1 - aa_i^{\left( r \right)} \right) - a_i^{\left( r \right)} \cdot cd_i^{\left( r \right)} \right) \\
s_i^{\left( r \right)} \left( -c_i^{\left( r \right)} \left( 1 - aa_i^{\left( r \right)} \right) - ab_i^{\left( r \right)} \cdot d_i^{\left( r \right)} \right) &
s_i^{\left( r \right)} \left( b_i^{\left( r \right)} \left( 1 - aa_i^{\left( r \right)} \right) - a_i^{\left( r \right)} \cdot cd_i^{\left( r \right)} \right) &
s_i^{\left( r \right)} a_i^{\left( r \right)} \left( \left( 1 - aa_i^{\left( r \right)} \right) + \left( 1 - dd_i^{\left( r \right)} \right) \right)
\end{bmatrix}
$$
$$
\scriptsize
\dv{\mathrm{R}_i^{\left( r \right)}}{\mathrm{q}_{i,2}^{\left( r \right)}} = \begin{bmatrix}
s_i^{\left( r \right)} b_i^{\left( r \right)} \left( \left( 1 - bb_i^{\left( r \right)} \right) + \left( 1 - aa_i^{\left( r \right)} \right) \right) &
s_i^{\left( r \right)} \left( c_i^{\left( r \right)} \left( 1 - bb_i^{\left( r \right)} \right) - \left( -ab_i^{\left( r \right)} \cdot d_i^{\left( r \right)} \right) \right) &
s_i^{\left( r \right)} \left( d_i^{\left( r \right)} \left( 1 - bb_i^{\left( r \right)} \right) - ab_i^{\left( r \right)} \cdot c_i^{\left( r \right)} \right) \\
s_i^{\left( r \right)} \left( c_i^{\left( r \right)} \left( 1 - bb_i^{\left( r \right)} \right) - ab_i^{\left( r \right)} \cdot d_i^{\left( r \right)} \right) &
-s_i^{\left( r \right)} b_i^{\left( r \right)} \left( \left( 1 - bb_i^{\left( r \right)} \right) + \left( 1 - dd_i^{\left( r \right)} \right) \right) &
s_i^{\left( r \right)} \left( -a_i^{\left( r \right)} \left( 1 - bb_i^{\left( r \right)} \right) - b_i^{\left( r \right)} \cdot cd_i^{\left( r \right)} \right) \\
s_i^{\left( r \right)} \left( d_i^{\left( r \right)} \left( 1 - bb_i^{\left( r \right)} \right) - \left( -ab_i^{\left( r \right)} \cdot c_i^{\left( r \right)} \right) \right) &
s_i^{\left( r \right)} \left( a_i^{\left( r \right)} \left( 1 - bb_i^{\left( r \right)} \right) - b_i^{\left( r \right)} \cdot cd_i^{\left( r \right)} \right) &
-s_i^{\left( r \right)} b_i^{\left( r \right)} \left( \left( 1 - bb_i^{\left( r \right)} \right) + \left( 1 - cc_i^{\left( r \right)} \right) \right)
\end{bmatrix}
$$
$$
\scriptsize
\dv{\mathrm{R}_i^{\left( r \right)}}{\mathrm{q}_{i,3}^{\left( r \right)}} = \begin{bmatrix}
-s_i^{\left( r \right)} c_i^{\left( r \right)} \left( \left( 1 - cc_i^{\left( r \right)} \right) + \left( 1 - dd_i^{\left( r \right)} \right) \right) &
s_i^{\left( r \right)} \left( b_i^{\left( r \right)} \left( 1 - cc_i^{\left( r \right)} \right) - \left( -a_i^{\left( r \right)} \cdot cd_i^{\left( r \right)} \right) \right) &
s_i^{\left( r \right)} \left( a_i^{\left( r \right)} \left( 1 - cc_i^{\left( r \right)} \right) - b_i^{\left( r \right)} \cdot cd_i^{\left( r \right)} \right) \\
s_i^{\left( r \right)} \left( b_i^{\left( r \right)} \left( 1 - cc_i^{\left( r \right)} \right) - a_i^{\left( r \right)} \cdot cd_i^{\left( r \right)} \right) &
s_i^{\left( r \right)} c_i^{\left( r \right)} \left( \left( 1 - cc_i^{\left( r \right)} \right) + \left( 1 - aa_i^{\left( r \right)} \right) \right) &
s_i^{\left( r \right)} \left( d_i^{\left( r \right)} \left( 1 - cc_i^{\left( r \right)} \right) - \left( -ab_i^{\left( r \right)} \cdot c_i^{\left( r \right)} \right) \right) \\
s_i^{\left( r \right)} \left( -a_i^{\left( r \right)} \left( 1 - cc_i^{\left( r \right)} \right) - b_i^{\left( r \right)} \cdot cd_i^{\left( r \right)} \right) &
s_i^{\left( r \right)} \left( d_i^{\left( r \right)} \left( 1 - cc_i^{\left( r \right)} \right) - ab_i^{\left( r \right)} \cdot c_i^{\left( r \right)} \right) &
-s_i^{\left( r \right)} c_i^{\left( r \right)} \left( \left( 1 - cc_i^{\left( r \right)} \right) + \left( 1 - bb_i^{\left( r \right)} \right) \right)
\end{bmatrix}
$$
$$
\scriptsize
\dv{\mathrm{R}_i^{\left( r \right)}}{\mathrm{q}_{i,4}^{\left( r \right)}} = \begin{bmatrix}
-s_i^{\left( r \right)} d_i^{\left( r \right)} \left( \left( 1 - dd_i^{\left( r \right)} \right) + \left( 1 - cc_i^{\left( r \right)} \right) \right) &
s_i^{\left( r \right)} \left( -a_i^{\left( r \right)} \left( 1 - dd_i^{\left( r \right)} \right) - b_i^{\left( r \right)} \cdot cd_i^{\left( r \right)} \right) &
s_i^{\left( r \right)} \left( b_i^{\left( r \right)} \left( 1 - dd_i^{\left( r \right)} \right) - a_i^{\left( r \right)} \cdot cd_i^{\left( r \right)} \right) \\
s_i^{\left( r \right)} \left( a_i^{\left( r \right)} \left( 1 - dd_i^{\left( r \right)} \right) - b_i^{\left( r \right)} \cdot cd_i^{\left( r \right)} \right) &
-s_i^{\left( r \right)} d_i^{\left( r \right)} \left( \left( 1 - dd_i^{\left( r \right)} \right) + \left( 1 - bb_i^{\left( r \right)} \right) \right) &
s_i^{\left( r \right)} \left( c_i^{\left( r \right)} \left( 1 - dd_i^{\left( r \right)} \right) - \left( -ab_i^{\left( r \right)} \cdot d_i^{\left( r \right)} \right) \right) \\
s_i^{\left( r \right)} \left( b_i^{\left( r \right)} \left( 1 - dd_i^{\left( r \right)} \right) - \left( -a_i^{\left( r \right)} \cdot cd_i^{\left( r \right)} \right) \right) &
s_i^{\left( r \right)} \left( c_i^{\left( r \right)} \left( 1 - dd_i^{\left( r \right)} \right) - ab_i^{\left( r \right)} \cdot d_i^{\left( r \right)} \right) &
s_i^{\left( r \right)} d_i^{\left( r \right)} \left( \left( 1 - dd_i^{\left( r \right)} \right) + \left( 1 - aa_i^{\left( r \right)} \right) \right)
\end{bmatrix}
$$

\end{document}